\documentclass[twoside]{article}

\usepackage[preprint]{aistats2026}

\usepackage{amsfonts,amsmath,amssymb}
\usepackage{xcolor}
\usepackage{hyperref}
\usepackage[noabbrev]{cleveref}
\usepackage{aliascnt}

\usepackage{xargs}
\usepackage[textwidth=2cm,textsize=tiny]{todonotes}

\usepackage[authoryear]{natbib}
\usepackage{algorithm}
\usepackage[noend]{algpseudocode}

\usepackage{booktabs}

\usepackage{amsthm}

\usepackage{graphicx}
\usepackage{subcaption}

\newtheorem{theorem}{Theorem}[section]

\newaliascnt{lemma}{theorem}      
\newtheorem{lemma}[lemma]{Lemma}  
\aliascntresetthe{lemma}          
\crefname{lemma}{lemma}{lemmas}
\Crefname{lemma}{Lemma}{Lemmas}

\newaliascnt{proposition}{theorem}
\newtheorem{proposition}[proposition]{Proposition}
\aliascntresetthe{proposition}
\crefname{proposition}{proposition}{propositions}
\Crefname{proposition}{Proposition}{Propositions}

\newaliascnt{corollary}{theorem}
\newtheorem{corollary}[corollary]{Corollary}
\aliascntresetthe{corollary}
\crefname{corollary}{corollary}{corollaries}
\Crefname{corollary}{Corollary}{Corollaries}

\newaliascnt{definition}{theorem}

\aliascntresetthe{definition}
\crefname{definition}{definition}{definitions}
\Crefname{definition}{Definition}{Definitions}

\newaliascnt{remark}{theorem}
\newtheorem{remark}[remark]{Remark}
\aliascntresetthe{remark}
\crefname{remark}{remark}{remarks}
\Crefname{remark}{Remark}{Remarks}

\newaliascnt{example}{theorem}

\aliascntresetthe{example}
\crefname{example}{example}{examples}
\Crefname{example}{Example}{Examples}

\newtheorem{assumption}{\textbf{A}\hspace{-2pt}}
\crefname{assumption}{A\hspace{-2pt}}{A\hspace{-2pt}}

\definecolor{amaranth}{rgb}{0.8, 0, 0.34}
\definecolor{refblue}{rgb}{0, 0, 0.8}
\hypersetup{
  colorlinks=true,
  linkcolor=refblue,
  citecolor=refblue,
  urlcolor=refblue
}

\usepackage{enumitem}

\usepackage{xfrac}
\newcommand{\ie}{\textit{i.e.}}

\def\paramw{\theta}
\def\Paramw{\Theta}
\newcommand{\locparamw}[1]{\paramw_{#1}}
\def\paramstar{\paramw^{\star}}
\def\Paramstar{\Paramw^{\star}}
\newcommand{\locparam}[2]{\paramw_{#2}^{(#1)}}
\newcommand{\globParam}[1]{\Paramw_{#1}}
\newcommand{\globParama}[1]{{\Paramw_{#1}^{(1)}}}
\newcommand{\globParamb}[1]{{\Paramw_{#1}^{(2)}}}
\newcommand{\globParamp}[1]{{\Paramw_{#1}'}}

\newcommand{\locparamstar}[1]{\paramstar_{#1}}
\newcommand{\locParamstar}{\Paramstar_{(\text{loc})}}

\def\Paramlim{\Paramw_{\mathrm{det}}}
\def\cParamlim{\overline{\Paramw}_{\mathrm{det}}}
\def\dParamlim{\widetilde{\Paramw}_{\mathrm{det}}}

\newcommand{\paramlimloc}[1]{\paramw_{#1}^{\mathrm{det}}}
\def\cparamlim{\bar{\paramw}^{\mathrm{det}}}
\newcommand{\dparamlimloc}[1]{\widetilde{\paramw}_{#1}^{\mathrm{det}}}

\def\Paramsto{\Paramw_{\mathrm{sto}}}
\newcommand{\paramlimsto}[1]{\paramw_{#1}^{\mathrm{sto}}}

\newcommand{\globParamRR}[1]{\Paramw_{#1}^{\mathrm{RR}, \step}}
\def\ParamlimRR{\Paramw_{\mathrm{det}}^{\mathrm{RR}, \step}}

\DeclareMathOperator{\diag}{diag}
\DeclareMathOperator{\Span}{Span}

\def\Id{\mathrm{Id}}
\def\commmat{W}
\def\tenscommmat{\boldsymbol{W}}
\def\tensId{\boldsymbol{I}}

\newcommand{\Avec}{\boldsymbol{A}}
\newcommand{\barAvec}{\boldsymbol{\bar{A}}}
\def\hetMat{\boldsymbol{\mathcal{H}}}
\def\melMat{\boldsymbol{\mathcal{G}}}
\def\melMatA{\boldsymbol{\mathcal{B}}}
\newcommand{\nbarA}[1]{A_{#1}}
\newcommand{\barA}{\bar{A}}

\def\projconsensus{\boldsymbol{\mathcal{P}}}
\def\projdisagreement{\boldsymbol{\mathcal{Q}}}

\def\rset{\mathbb{R}}
\def\nset{\mathbb{N}}

\def\nagent{m}
\def\lip{L}
\def\strcvx{\mu}
\def\oneVec{\boldsymbol{1}}

\def\fw{f}
\newcommand{\nfw}[1]{\fw_{#1}}
\newcommand{\nf}[2]{\nfw{#1}(#2)}
\newcommand{\ngf}[2]{\nabla \nf{#1}{#2}}
\newcommand{\nhf}[2]{\nabla^2\nf{#1}{#2}}

\newcommand{\hessf}[1]{\nabla^2 \fw(#1)}
\newcommand{\terf}[1]{\nabla^3 \fw(#1)}
\newcommand{\terfk}[1]{\nabla^3\nfw{#1}}
\newcommand{\quaterfk}[1]{\nabla^4\nfw{#1}}

\def\boundThird{K_3}
\def\boundFourth{K_4}

\def\Fw{F}
\newcommand{\F}[1]{\Fw(#1)}
\newcommand{\gF}[1]{\nabla \F{#1}}
\newcommand{\hF}[1]{\nabla^2 \F{#1}}
\newcommand{\terF}[1]{\nabla^3 \F{#1}}

\def\paramnoise{\varepsilon}
\newcommand{\globnoise}[2]{\paramnoise_{#1}(#2)}
\newcommand{\funnoise}[1]{\paramnoise_{#1}}
\newcommand{\locnoise}[3]{\paramnoise_{#1}^{(#2)}(#3)}
\newcommand{\locnoisefun}[2]{\paramnoise_{#1}^{(#2)}}
\newcommand{\gfnoisy}[2]{\nabla \Fw({#1}; {#2})}

\newcommand{\calF}{\mathcal{F}}

\def\CovarianceTens{\boldsymbol{\mathcal{C}}}
\def\Covariance{\mathcal{C}}

\def\CovBoundA{{\sigma_\star^2}}
\def\CovBoundB{\beta}

\def\paramker{R}
\def\kernelfun{\paramker_\step}
\newcommand{\kernel}[2]{\paramker_\step({#1},{#2})}

\def\parammeas{\pi}
\newcommand{\meas}[1]{\parammeas_{#1}}
\def\wasserstein{\mathbf{W}_2}

\newcommand{\eqdef}{:=}

\def\step{\gamma}

\def\PE{\mathbb{E}}
\newcommand{\CPE}[2]{\PE \left[ #1 ~\middle|~ #2 \right]}
\def\PP{\mathbb{P}}
\newcommand{\CPP}[2]{\PP \left[ #1 ~\middle|~ #2 \right]}

\newcommand{\norm}[1]{ \lVert #1 \rVert }
\newcommand{\bnorm}[1]{ \Big\lVert #1 \Big\rVert }

\newcommandx{\paul}[1]{\todo[color=purple!20]{PM: #1}}
\newcommandx{\pauli}[1]{\todo[inline,color=purple!20]{PM: #1}}

\newcommandx{\lucas}[1]{\todo[color=green!20]{LV: #1}}
\newcommandx{\lucasi}[1]{\todo[inline,color=green!20]{LV: #1}}

\newcommandx{\aymeric}[1]{\todo[color=blue!20]{AD: #1}}
\newcommandx{\aymerici}[1]{\todo[inline,color=blue!20]{AD: #1}}

\newcommandx{\rema}[1]{\todo[color=yellow!20]{Rem: #1}}
\newcommandx{\remai}[1]{\todo[inline,color=yellow!20]{Rem: #1}}

\renewcommand{\paul}[1]{}
\renewcommand{\pauli}[1]{}
\renewcommand{\lucas}[1]{}
\renewcommand{\lucasi}[1]{}
\renewcommand{\aymeric}[1]{}
\renewcommand{\aymerici}[1]{}

\def\eqsp{\enspace}

\def\covw{\boldsymbol{\Sigma}}
\newcommand{\covstationary}{\covw_{\theta}}
\newcommand{\covconsensus}{\covw_{\projconsensus\theta}}
\newcommand{\covdisagreement}{\covw_{\projdisagreement\theta}}
\newcommand{\covdisagreementconsensus}{\covw_{\projdisagreement\projconsensus}}
\newcommand{\covavgnoise}{\bar{\covw}_{{\varepsilon}}}
\newcommand{\covnoise}{\covw_{{\varepsilon}}}

\def\spectralgap{\lambda_2}
\def\spectralgapabs{\rho}

\newcommand{\bglobParam}[1]{\bar{\paramw}_{#1}}

\def\bparamlim{\bar{\paramw}_{det}}
\newcommand{\bglobnoise}[2]{\bar{\varepsilon}_{#1}(#2)}
\newcommand{\bhF}[1]{\bar{A}(#1)}

\def\hgty{\zeta_\star}
\def\ophess{\boldsymbol{J}}

\begin{document}

\twocolumn[

\aistatstitle{Tight Analysis of Decentralized SGD: A Markov Chain Perspective}

\aistatsauthor{Lucas Versini \And Paul Mangold \And Aymeric Dieuleveut}

\aistatsaddress{
CMAP, CNRS, École polytechnique, Institut Polytechnique de Paris
} ]

\begin{abstract}
We propose a novel analysis of the Decentralized Stochastic Gradient Descent (DSGD) algorithm with constant step size, interpreting the iterates of the algorithm as a Markov chain.
We show that DSGD converges to a stationary distribution, with its bias, to first order, decomposable into two components: one due to decentralization (growing with the graph's spectral gap and clients' heterogeneity) and one due to stochasticity.
Remarkably, the variance of \emph{local} parameters is, at the first-order, inversely proportional to the number of clients, regardless of the network topology and even when clients' iterates are not averaged at the end.
{As a consequence of our analysis, we obtain non-asymptotic convergence bounds for clients' local iterates, confirming that DSGD has linear speed-up in the number of clients, and that the network topology only impacts higher-order terms.}
\end{abstract}

\section{INTRODUCTION}
\label{sec:introduction}
Decentralized optimization is a key paradigm for large-scale machine learning. For such problems, it quickly becomes necessary to keep data distributed across multiple clients, sharing the computation among all participants. While many methods rely on a central server to coordinate communication \citep{kairouz2021advances}, decentralized approaches offer a strong alternative for handling even larger and more heterogeneous clients. A fundamental method in this setting is Decentralized Stochastic Gradient Descent (DSGD), where clients perform local stochastic updates and exchange information with their neighbors \citep{nedic2009distributed,duchi2011dual,lian2017can}.
This algorithm alleviates the strong requirements of centralized and federated learning, which either necessitate collecting all data in a single location or relying on a central server to manage all communications.

Decentralized learning thus appears as a more flexible alternative to the widely studied federated learning \citep{mcmahan2017communication}, and many variants of DSGD have been proposed an analyzed in a wide range of settings 
\citep{colla2021automated, koloskova2020unified, yuan2016convergence, richards2020graph, bars2023improved, sayed2014diffusion, 4749425, sundhar2010distributed}. 
Nonetheless, some of DSGD's theoretical properties still remain misunderstood.
Specifically, this algorithm and its variants have been studied extensively in the deterministic setting, with corresponding upper-bounds on the error and the convergence rate \citep{vogels2022beyond, 4749425, sundhar2010distributed, yuan2016convergence, koloskova2020unified}, but its behavior under stochastic noise is less precisely understood.

In this work, we present a refined analysis of DSGD, with a particular focus on the stochastic nature of local updates. 
We propose a novel interpretation of DSGD as a Markov chain, which we show to be geometrically ergodic and therefore convergent to a stationary distribution. By analyzing this stationary distribution, we derive exact first-order expansions in the step size for the bias and variance of DSGD. Two key insights stem from these expressions: a bias arises when the communication graph is not fully connected and clients are heterogeneous, while the variance depends on the graph only through higher-order terms. Finally, we provide a non-asymptotic convergence rate, enabling a precise characterization of DSGD's convergence.

We stress that the goal of our work is to \emph{precisely} study the properties of DSGD, one of the most fundamental algorithm of decentralized learning.
Here, we aim to develop an analytical framework for the analysis of DSGD, that (i) gives very precise insights on the bias of the algorithm through first-order expansions, rather than upper bounds, (ii) could readily be extended to more complex decentralized algorithms, and (iii) could guide the design of future decentralized algorithms.
Our contributions can be summarized as follows:
\begin{itemize}[topsep=1pt,itemsep=1pt,leftmargin=*]
    \item In \Cref{sec:deterministic}, we first study the \textit{deterministic} counterpart of DSGD, simply denoted as deterministic GD (DGD),  and derive an explicit first-order expansion in the step size of its bias. This shows that the limit point reached by the algorithm moves further away from the optimal point with both network's spectral gap and clients' heterogeneity. In particular, it vanishes when either of the two is zero.
    
    \item In \Cref{sec:stochastic}, interpreting DSGD as a Markov chain, we show that DSGD's iterates  converge geometrically fast (in Wasserstein distance) towards a stationary distribution. We provide first-order expansions of DSGD's bias and variance at stationarity, disentangling two sources of bias: one due to decentralization and heterogeneity combined, and one due to stochasticity. Interestingly, DSGD's variance decreases with the number of clients, and depends on the network topology only in higher-order terms.
    
    \item
    We then derive sharper non-asymptotic  rates for DSGD. Specifically, we (i) show that clients iterates converge with linear speed-up \emph{without averaging the parameters across clients}, and (ii) tightly control the impact  of decentralization and heterogeneity, as well as  that of the topology on the variance.
    
    \item In \Cref{sec:RR}, we propose a novel decentralized algorithm, based on Richardson-Romberg extrapolation,  eliminating  first order bias.
    We study its sample complexity, proving that it reduces communication without any knowledge of the network topology.
    
    \item Finally, we provide experiments in \Cref{sec:experiments}.
\end{itemize}

\paragraph{Notations.}
For $f:\rset^d\to\rset$ a function $i$-times differentiable ($i\ge1$), we denote $\nabla^i f$ its $i$-th tensor derivative.
For a vector $\norm{\cdot}$ is its Euclidean norm, and for a matrix, it is its $\ell_2$ operator norm. $\Id$ is the identity matrix, and $\oneVec_m$ the vector filled with $1$ in dimension $m$. 
For a symmetric matrix $A$, we denote $A$'s second largest eigenvalue $\lambda_2(A)$, and $\lambda_{\mathrm{min}}(A)$ its smallest eigenvalue.
We write $\wasserstein(\rho_1, \rho_2)$ the second-order Wasserstein distance between two probability measures $\rho_1$ and $\rho_2$.
For two matrices $A, B$, $A \otimes B$ is the Kronecker product of $A$ and $B$, \ie, the operator $A \otimes B: M \mapsto B M A^\top$. We also define $A^{\otimes k}$ as the $k$-th power of the tensor $A$.

\section{PROBLEM SETTING}
\label{sec:setting}
\paragraph{Decentralized Learning.}
We consider the distributed optimization problem
\begin{equation}\label{eq:minimization-problem}
    \paramstar \in \underset{\paramw\in\rset^d}{\arg\min} \, f(\paramw) = \textstyle \frac{1}{\nagent} \sum_{k = 1}^\nagent \nf{k}{\paramw}
    \eqsp,
\end{equation}
where for $k \in \{ 1, \dots, \nagent \}$, $\nfw{k}: \rset^d \to \rset$ is the local objective of agent $k$.
Each agent starts from an initial vector $\locparam{k}{0}\in\rset^d$, and updates it with its neighbors
\[
    \textstyle\locparam{k}{t+1}
    =
    \sum\limits_{\ell=1}^{\nagent} \commmat_{k\ell} \big( \locparam{\ell}{t} - \step \big( \ngf{\ell}{\locparam{\ell}{t}} + \locnoise{t+1}{\ell}{\locparam{\ell}{t}} \big) \big)
    \eqsp,
\]
where $\commmat\in\rset^{\nagent\times\nagent}$ is a fixed communication matrix, assumed to be symmetric and doubly stochastic, and
$\locnoise{t+1}{\ell}{\locparam{\ell}{t}}$ is a random noise term.

Note the clients first update their parameters, before averaging the parameters of their neighbors, which is usually referred to as "Adapt-then-Combine".

\algrenewcommand\algorithmicindent{1em} 
\begin{algorithm}
    \caption{DSGD}\label{alg:DSGD}
    \begin{algorithmic}[1]
        \State \textbf{Input:} $\forall k, \ \locparam{k}{0} = \paramw_{0} \in \rset^d$, number of iterations $T$, step size $\step$, matrix $\commmat$.
        \For{$t = 0, \ldots, T-1$}
            \For{each client $k = 1, \ldots, \nagent$}
                \State
                $\!\!\!\!\locparam{k}{t+1}
                \!\leftarrow\!\!
                \sum_{\ell=1}^{\nagent}\! \commmat_{k\ell} \big( \locparam{\ell}{t}\! - \step \big( \ngf{\ell}{\locparam{\ell}{t}} + \locnoise{t+1}{\ell}{\locparam{\ell}{t}} \big) \!\big)
                $
            \EndFor
        \EndFor
        \State\Return $\locparam{1}{T}, \dots, \locparam{\nagent}{T}$.
    \end{algorithmic}
\end{algorithm}
It is convenient to rewrite the updates in global form. We define the stacked parameter vector $\globParam{t}$, the global gradient $\gF{\globParam{t}}$ and the global noise  $\globnoise{t+1}{\globParam{t}}$ in $\rset^{\nagent d}$ as
\begin{align*}
    \globParam{t}
    & \eqdef
    ((\locparam{1}{t})^\top \quad  \cdots \quad (\locparam{\nagent}{t})^\top )^\top
    \eqsp, \\
    \gF{\globParam{t}}
    & \eqdef ( (\ngf{1}{\locparam{1}{t}})^\top \quad \cdots \quad (\ngf{\nagent}{\locparam{\nagent}{t}})^\top )^\top
    \eqsp, \\
    \globnoise{t+1}{\globParam{t}}
    & \eqdef
    ( (\locnoise{t+1}{1}{\locparam{1}{t}})^\top \quad \cdots \quad (\locnoise{t+1}{\nagent}{\locparam{\nagent}{t}})^\top )^\top
    \eqsp.
\end{align*}
With these notations, the algorithm recursion writes
\begin{align}\label{eq:DSGD}
   \!\!\! \globParam{t + 1}
    = \tenscommmat \left( \globParam{t}
    - \step \big( \gF{\globParam{t}} + \globnoise{t+1}{\globParam{t}} \big) \right)
    \eqsp,
    \tag{DSGD}
\end{align}
where $\tenscommmat \eqdef \commmat\otimes \Id$. 
Lastly, we also set $\Paramstar = \oneVec_m \otimes \paramstar \in \rset^{\nagent d}$, defined by repeating $m$ times $\paramstar$.

\paragraph{Assumptions.}
Throughout this paper, we work under the following regularity assumptions.
\begin{assumption}[Regularity]\label{assum:functions}
    For all $k \in \{ 1, \dots, \nagent \}$:
    \begin{itemize}[leftmargin=15pt,topsep=-2pt,itemsep=0pt]
        \item[(a)] $\nfw{k}$ is $\strcvx$-strongly convex.
        \item[(b)] $\nfw{k}$ is four times differentiable and $\lip$-smooth.
        \item[(c)]  
        $\nfw{k}$\!\!'s third derivative is bounded: there exists $\boundThird > 0$ such that for $\paramw, u \in\rset^d$,
        $\norm{\terfk{k}(\paramw) u^{\otimes 2}} \le \boundThird \norm{u}^2$.
    \end{itemize}
\end{assumption}
This assumption is standard in the analysis of stochastic and decentralized algorithms \citep{dieuleveut2020bridging,mangold2025refined,mangold2025scaffold,vogels2022beyond}. The third derivative's boundedness is useful to control higher-order terms, as can be seen in the proofs of the results presented in \Cref{sec:stochastic}.

Finally, to measure heterogeneity, we also define
\begin{align}
\label{eq:def-hgty}
\hgty^2 = \norm{\gF{\Paramstar}}^2
\eqsp.
\end{align} 
\begin{assumption}
\label{assum:commmat}
        The matrix $\commmat$ is symmetric and stochastic, with $\lambda_2(\commmat) < 1$; we let $\Lambda = 2 \norm{(\Id - \commmat)^\dagger \commmat}_2$.
    We define $\spectralgapabs = \max(|\lambda_2(\commmat)|,|\lambda_{\min}(\commmat)|)$.
\end{assumption}
This assumption is classical in decentralized learning, and ensures that the underlying undirected communication graph is connected. In particular, information from each client eventually propagates to all other clients.
For instance, given a weighted, connected, and undirected graph with Laplacian matrix $L$, the matrix $\commmat = \Id - t L$ satisfies \Cref{assum:commmat} for a small enough $t > 0$.

\section{DETERMINISTIC ALGORITHM}
\label{sec:deterministic}
In this section, we start with the analysis of deterministic Decentralized GD (DGD), ie.~
\begin{align}
    \globParam{t+1}
    & =
    \tenscommmat \big( \globParam{t} - \step \gF{\globParam{t}} \big)
    \eqsp.
    \tag{DGD}
    \label{eq:DGD}
\end{align}
One of the first notable analysis of DGD was provided by~\cite{4749425}, and many rates were later established on the sub-optimality gap $f(\bar\paramw_T) -  \min f$, where $\bar\paramw_T$ is the average of clients iterates at time $T$ \citep{colla2021automated, yuan2016convergence,le2023refined}.  In particular, it is known that the clients' models do not all converge to the optimal point, motivating the introduction of \textit{gradient tracking}~\citep{shi2015extra,nedic2017achieving}. Most works focus on upper bounding the error between $\bar\paramw_T$ and $\theta_*$. 
Yet, computing the parameter $\bar\paramw_T$ requires a global averaging step over the entire graph, which may be costly.

Furthermore, our goal is to understand the limit behavior of the limit as function of $\gamma$. We thus need to establish the convergence to a limit and characterize it precisely. We  show that DGD converges, and derive an explicit expansion of its bias (\ie, the difference between its limit point and the problem's global solution).
We first recall the convergence of DGD to a fixed point. While this first result was  established by \citep{vogels2022beyond,larsson2025unified},  our following  results are more precise, with an \emph{exact first-order expansion of the limit point}, rather than  upper bounds, allowing to obtain extrapolation results.

\begin{lemma}\label{lemma:DGD-A}
    Assume \Cref{assum:functions}, \Cref{assum:commmat}, and $\step \le 1/\lip$. Then the sequence $(\globParam{t})_t$ generated by \eqref{eq:DGD} converges to a vector $\Paramlim$, independent of $\globParam{0}$, which satisfies
    \begin{align}\label{eq:fixed-point}
        (\tensId - \tenscommmat)\Paramlim
        = - \step \tenscommmat \gF{\Paramlim}
        \eqsp.
    \end{align}
    Moreover, for all $t \ge 0$,
    \begin{align}
    \label{eq:conv-rate-dgd-to-stationary}
        \norm{\globParam{t} - \Paramlim}
        & \le
        (1 - \step\strcvx)^t \norm{\globParam{0} - \Paramlim}
        \eqsp.
    \end{align}
\end{lemma}
We give a proof based on a contraction argument of this lemma  in \Cref{sec:det-dgd}. 
This  establishes that the sequence $(\globParam{t})_t$ converges to a point $\Paramlim$, which depends on both the step size $\step$ and the communication matrix $\commmat$.
We next use \eqref{eq:fixed-point} to derive an explicit expansion of $\Paramlim$.
Note that the matrix $\Id - \commmat$ has different behavior on $\Span(\oneVec)$, where it is zero, and on its orthogonal space $\Span(\oneVec)^{\perp}$, where its largest eigenvalue is strictly smaller than $1$.
We thus decompose the vectors into two orthogonal components: a consensus part, where all the clients have the same value, and a disagreement part, characterizing individual specificity. We define the consensus and disagreement projection operators, which project a vector into the relevant subspace
\begin{align}
    \projconsensus &\eqdef \textstyle \tfrac{1}{\nagent} \oneVec \oneVec^\top \otimes \Id
    \eqsp,
    &
    \projdisagreement &\eqdef \textstyle \tensId - \projconsensus
    \eqsp.
\end{align}
The operator $\projconsensus$ projects a vector on the consensus space $\Span(\oneVec)$ by averaging its block-components, while $\projdisagreement$ computes the difference between a vector and its consensus.
We can now define the consensus and disagreement parts of the clients' iterates
\begin{align}
    \cParamlim &\eqdef \projconsensus \Paramlim
    \eqsp,
    &
    \dParamlim &\eqdef \projdisagreement \Paramlim
    \eqsp,
\end{align}
so that $\Paramlim = \cParamlim + \dParamlim$.
With these notations, the identity $\eqref{eq:fixed-point}$ can be reformulated as follows.
\begin{lemma}\label{lem:stationary-def-general}
    Under \Cref{assum:functions}, \Cref{assum:commmat}, and for $0 < \step \le 1/\lip$:
    \begin{gather}
    \label{eq:optim-condition-det}
        \textstyle \frac{1}{\nagent} \sum_{k=1}^\nagent \ngf{k}{\paramlimloc{k}} = 0
        \eqsp, \\
    \label{eq:relation-disagreement-det}
        \dParamlim = - \step (\tensId - \tenscommmat)^\dagger \tenscommmat \gF{\Paramlim}
        \eqsp.
    \end{gather}
\end{lemma}
The identity \eqref{eq:optim-condition-det} is a kind of optimality condition; if $\Paramlim$ is already at consensus, it means that DGD converged to the solution of the global problem.
The second equation \eqref{eq:relation-disagreement-det} defines the disagreement part. Based on the identity $(1 - x)^{-1} = \sum_{k=0}^\infty x^k$, \eqref{eq:relation-disagreement-det} can be interpreted as the part of the gradients that reaches each client after any $k \ge 0$ rounds of propagation through the network, making iterates depart from consensus.

Now, we derive expansions of the error based on \Cref{lem:stationary-def-general}: first for quadratics, where we give \emph{exact} expressions, and then for general functions.

\paragraph{Quadratic functions.}
When the functions $\nfw{k}$ are quadratic, we can derive an exact expression of $\Paramlim$. 
\begin{assumption}\label{assum:quadratic}
    For $k \in \{1, \dots, \nagent\}$, there exist $\nbarA{k} \in \rset^{d \times d}$ positive definite and $\locparamstar{k} \in \rset^d$ such that
    $\nf{k}{\paramw} = \frac{1}{2} \norm{\nbarA{k}^{1/2}(\paramw - \locparamstar{k})}^2$. We let
    \begin{gather*}
    \textstyle
        \barA \eqdef \frac{1}{\nagent} \sum_{k=1}^\nagent \nbarA{k}
        \eqsp,
        \qquad \quad
    \Avec \eqdef \diag(\nbarA{1}, \dots, \nbarA{\nagent})
        \eqsp, \\
    \textstyle
         \locParamstar \eqdef (\locparamstar{1}{}^\top, \dots, \locparamstar{\nagent}{}^\top)^\top
        \eqsp, \quad
        \barAvec \eqdef \diag(\barA, \dots, \barA)
        \eqsp.
    \end{gather*}
\end{assumption}
In this setting, the fixed-point equation \eqref{eq:fixed-point} reduces to a linear system $(\Id - \commmat) \Paramlim = - \step \commmat \Avec(\Paramlim - \locParamstar)$.
Under \Cref{assum:quadratic}, we define the matrices
\begin{gather}
\label{eq:def-het-mel-mat}
    \hetMat \eqdef \projconsensus \barAvec^{-1} \Avec
    \eqsp,
    \qquad
    \melMat \eqdef (\tensId - \tenscommmat)^\dagger \tenscommmat
    \eqsp,
    \\
\label{eq:def-mel-matA}
    \melMatA \eqdef (\tensId + \step \melMat \Avec)^{-1} \melMat \Avec
\end{gather}
The matrix $\hetMat$ measures heterogeneity among clients, and relates local and global solutions.
Indeed, we have $\paramstar = \hetMat(\locparamstar{1}{}^\top, \dots, \locparamstar{\nagent}{}^\top)^\top$.
If all matrices $\nbarA{k}$ are equal, then $\paramstar$ is simply the average of the vectors $\paramw_k$.
The matrix $\melMat$, which is zero on the consensus subspace $\Span(\oneVec)$, acts as $\sum_{t = 1}^{+\infty} \tenscommmat^t$ on the disagreement subspace $\Span(\oneVec)^\perp$.
When applied to a vector, $\melMat$ sums all the disagreements propagated in the graph from this vector. 
The matrix $\melMatA$ mixes these two matrices, accounting both for heterogeneity and network topology.
Using these matrices, we can give an exact expression, and first-order expansions, of the error.

\begin{proposition}\label{lemma:deter-quad}
    Assume \Cref{assum:commmat}-\ref{assum:quadratic} and $\step < 2 / ((1+\lip/\strcvx) \lip\Lambda)$.
    Then we have the following expressions
    \begin{equation*}
            \setlength{\arraycolsep}{1pt}
\begin{array}{lll}
        \cParamlim 
        =&\  \Paramstar -&  \step \hetMat (\tensId - \step \melMatA \hetMat)^{-1} \melMatA (\locParamstar - \Paramstar)
        \eqsp, \\
        \dParamlim
        =&\  & \step (\tensId - \step \melMatA \hetMat)^{-1} \melMatA (\locParamstar - \Paramstar)
        \eqsp, \\
        \Paramlim
        =&\  \Paramstar + & \step (\tensId - \hetMat)(\tensId - \step \melMatA \hetMat)^{-1} \melMatA (\locParamstar - \Paramstar)
        \eqsp,
    \end{array}
    \end{equation*}
    Moreover, the following expansions hold:
    \begin{equation*}
        \setlength{\arraycolsep}{1pt}
\begin{array}{llll}
        \cParamlim
        =&\  \Paramstar
        - &  \step \hetMat \melMat \Avec (\locParamstar - \Paramstar)
        & + \mathcal{O}(\step^2)
        \eqsp, \\
        \dParamlim
        =&\  & \step \melMat \Avec (\locParamstar - \Paramstar)
        & + \mathcal{O}(\step^2)
        \eqsp, \\
        \Paramlim
        =&\  \Paramstar
       + &  \step (\tensId - \hetMat) \melMat \Avec (\locParamstar - \Paramstar)
        & + \mathcal{O}(\step^2)
        \eqsp.
    \end{array}
    \end{equation*}
\end{proposition}
We prove this lemma in \Cref{sec:det-dgd}.
Remark that the condition on $\step$ is only needed to guarantee the existence of all inverse matrices that appear in the expressions. 
For larger values of $\step$ for which all inverse matrices exist, the above formulas still hold.
The expressions from \Cref{lemma:deter-quad} allow to make multiple observations on the bias of DGD in this setting.

\emph{Impact of the step size.}
The expansion shows that $\Paramlim - \Paramstar = \mathcal{O}(\step)$, and thus $\smash{\lim_{\step\to0} \Paramlim = \Paramstar}$. Hence, decreasing the step size reduces the bias of DGD, which eventually vanishes.
Note that the bias, as well a the disagreement, scales linearly with $\step$ for both the consensus and disagreement components.
Of course, as seen in \Cref{lemma:DGD-A}, reducing the bias by decreasing $\step$  also slows down the convergence of the algorithm.

\emph{Impact of data heterogeneity.}
In the homogeneous case where all local minimizers $\locparamstar{k}$ coincide, we have $\locParamstar = \Paramstar$, and therefore $\Paramlim = \Paramstar$. This shows that heterogeneity across clients is the only source of bias: the greater the differences between the vectors $\locparamstar{k}$, the larger the deviation of $\Paramlim$ from $\Paramstar$.

\emph{Role of the communication matrix.}
The network topology appears through $\melMat = (\tensId - \tenscommmat)^\dagger \tenscommmat$. Recall that $\melMat$ captures how disagreement between clients is smoothed out by iterated communication by summing all the disagreements obtained on the gradient $\gF{\Paramstar}$ at each time step $t \ge 1$.
As an example, when the network is fully connected ($\commmat = \frac{1}{\nagent} \oneVec\oneVec^\top$), $\melMat$ is equal to $0$, and $\Paramlim = \Paramstar$; for sparse graphs, $\melMat$ is non-zero, leading to a larger bias.
Generally, under \Cref{assum:commmat}, the spectral norm of $\melMat$ is equal to $\Lambda = 2 \norm{(\Id - \commmat)^\dagger \commmat}_2$,
which decreases when the eigenvalues of $\commmat$, other than the one at 1, are close to $0$, i.e., when the graph is more connected.

\paragraph{General functions.}
We now study the case of strongly convex and smooth functions.
To derive our refined analysis, which finely quantifies the prominent terms in the error, we need a crude bound on DGD's convergence, which we will use to bound higher-order terms.
First, we control the bias $\Paramlim - \Paramstar$, using a bound recently established by \citet{larsson2025unified}, showing that it is of order $\step$.
\begin{lemma}[{\cite[Formula 47]{larsson2025unified}}]\label{lemma:upper-bound-error}
    Assume \Cref{assum:functions} and \Cref{assum:commmat}.
    Then for $\step < \min\left(\frac{1}{\Lambda\lip}, \frac{1}{\lip} \right)$,
    \[
        \norm{\Paramlim - \Paramstar}
        \le
        \tfrac{\step \lip \Lambda}{\strcvx} \hgty
        \eqsp.
    \]
    In particular, as $\step\to0$, $\norm{\Paramlim - \Paramstar} = \mathcal{O}(\step)$.
    \end{lemma}
This directly gives a convergence rate for DGD.
\begin{corollary}\label{coro:conv-dgd}
    Assume \Cref{assum:functions}-\ref{assum:commmat}.
    Then, for any $t\ge0$. we have the convergence rate
    \begin{align*}
        \norm{\globParam{t} - \Paramstar}
        & \le
        \textstyle
        (1 - \step\strcvx)^t \norm{\globParam{0} - \Paramstar}
        +
         \frac{2\step \lip \Lambda}{\strcvx}  \hgty
        \eqsp.
    \end{align*}
    For $\epsilon > 0$, setting $\step = \min(\tfrac{1}{\lip}, \tfrac{1}{\Lambda L}, \tfrac{\strcvx \epsilon}{\lip \Lambda \hgty})$, we have $\norm{ \globParam{t} - \Paramstar } \le \epsilon$ for $T \gtrsim \max(\tfrac{\lip}{\strcvx}, \tfrac{\Lambda \lip}{ \strcvx}, \tfrac{\lip \Lambda \hgty}{\strcvx^2 \epsilon} ) \log(\tfrac{\norm{ \globParam{0}- \Paramstar}}{\epsilon})$.
\end{corollary}
Notice that to achieve good precision, one needs to require $\step$ to be smaller when the graph is less connected.
Using this bound, we adapt \Cref{lemma:deter-quad} to the general strongly convex and smooth setting.
To this end, and under \Cref{assum:functions}-\ref{assum:commmat}, we define the average Hessian $\hessf{\paramstar} \eqdef \frac{1}{\nagent} \sum_{k=1}^\nagent \nhf{k}{\paramstar}$, as well as the the block-diagonal matrices
    \begin{align*}
        \Avec & \eqdef \diag\left( \nhf{1}{\paramstar}, \dots, \nhf{\nagent}{\paramstar} \right)
        \eqsp, \\
        \barAvec & \eqdef \diag(\hessf{\paramstar}, \dots, \hessf{\paramstar})
        \eqsp,
    \end{align*}
    and formally define $\hetMat$ and $\melMat$ as in \eqref{eq:def-het-mel-mat}-\eqref{eq:def-mel-matA} with theses matrices $\Avec$ and $\bar\Avec$.

\begin{proposition}[First-order bias expansion]\label{prop:expansion-general}
    Assume \Cref{assum:functions}-\ref{assum:commmat}.
    Then we have
    \begin{equation*}
    \setlength{\arraycolsep}{1pt}
    \begin{array}{llll}
        \cParamlim
        =& \  \Paramstar 
        & + \step  \hetMat \melMat \gF{\Paramstar}
        &  + \mathcal{O}(\step^2)
        \eqsp, 
        \\
        \dParamlim 
        =& \ &  - \step  \melMat \gF{\Paramstar}
       &   + \mathcal{O}(\step^2)
        \eqsp, 
        \\
        \Paramlim 
        =& \  \Paramstar 
        & - \step  (\tensId - \hetMat) \melMat \gF{\Paramstar}
        &  + \mathcal{O}(\step^2)
        \eqsp.
        \end{array}
    \end{equation*}
    For $\step \le \min\left(\frac{1}{\Lambda\lip}, \frac{1}{\lip} \right)$, the residual term is bounded by
    \begin{align*}
        \textstyle
        \step^2 \frac{\lip^2}{2\strcvx^2} \Lambda^2
        \Big(
        \tfrac{\boundThird}{\strcvx\sqrt{\nagent}} \hgty^2
        + \lip \cdot \hgty
        \Big)
        \eqsp.
    \end{align*}
    \end{proposition}
The proof, given in \Cref{sec:det-dgd}, linearizes the local gradients using Taylor expansion around $\paramstar$, and uses \Cref{lemma:upper-bound-error} to control higher-order terms.
\Cref{prop:expansion-general} is the exact analogue of \Cref{lemma:deter-quad} for non-quadratic functions. The bias, to the first order, is the same as in the quadratic setting, the same discussion thus applies.
In short, the same three factors as in the quadratic setting (step size, heterogeneity and network connectivity) determine the bias of deterministic decentralized gradient descent.
The main difference is that, contrarily to the quadratic setting, no closed-form expressions are available, but only first-order expansions.

\section{STOCHASTIC ALGORITHM}
\label{sec:stochastic}
We now consider DSGD with stochastic gradients.
Prior work \citep{7840591, lan2020communication, lian2017can,bars2023improved,neglia2020decentralized} provided convergence rates on $\PE\norm{\nabla f(\Paramw_t)}^2$ or $\PE\left[f(\Paramw_t)\right] - \min f$, where $\Paramw_t$ is the vector of all local models at iteration $t$.
Other papers focus directly on the sequence of iterates $(\globParam{t})_{t\ge0}$ \citep[Theorem 9]{vogels2022beyond} and prove that the distance between $\globParam{t}$ and $\Paramstar$ can be upper-bounded by a sum of three terms: one which depends on the initialization and vanishes as the time step $t$ tends to $+\infty$, one which depends on the variance, and one which depends on heterogeneity.

Our analysis departs from these results by analyzing the iterates of DSGD under a constant step size \textit{as a Markov chain.}
In the single-client setting, \citet{doi:10.1137/0324039,681779,pmlr-v84-chee18a} showed that, with a constant step size, SGD does not converge to the minimizer, but rather oscillates around it. Viewing SGD with a constant step size as a Markov chain was used to characterizes SGD's stationary distribution, with convergence rates and asymptotic bias expansions \citep{dieuleveut2020bridging}.
This approach has recently been extended to federated learning algorithms such as Scaffold and FedAvg \citep{mangold2025scaffold,mangold2025refined}, giving expansions of the bias and variance, with associated convergence rates.

In this paper, we thus show the convergence of the Markov chain generated by DSGD, and derive an explicit expansion of its bias and variance.

\begin{itemize}[leftmargin=*,noitemsep,topsep=0pt]
    \item 
We highlight the fact that DSGD has linear speed-up \emph{regardless of the topology}, and even \emph{without averaging the local iterates}, up to higher-order terms.
\item We show that stochasticity introduces an additional bias in the limit, corresponding to SGD's bias.
\item We provide non-asymptotic convergence bounds for clients' local iterates.
\end{itemize}

\paragraph{DSGD as a Markov Chain.}
We introduce the following three assumptions on the noise model.
\begin{assumption}[Zero-mean noise]\label{assum:noise-filtration}
    There exists a filtration $(\calF_t)_{t \in \nset}$ such that for all $t \in \nset$ and $\Paramw \in \rset^{\nagent d}$,
    $\globnoise{t+1}{\Paramw}$ is $\calF_t\text{-measurable}$,
    and
    $\CPE{ \globnoise{t+1}{\Paramw} }{ \calF_t } = 0$.
    Moreover, the sequence of random functions $(\funnoise{t})_t$ is i.i.d., and the noises $\locnoisefun{t+1}{k}$ of the different clients are independent.
\end{assumption}
\begin{assumption}[$p$](Co-coercivity)\label{assum:noise-cocoercive}
    For every $t \ge 1$, define the stochastic gradient mapping
    $\gfnoisy{\cdot}{\funnoise{t}} \eqdef \gF{\cdot} + \globnoise{t}{\cdot}$.
    Then, almost surely, $\gfnoisy{\cdot}{\funnoise{t}}$ is $\lip$-co-coercive, i.e.
    \begin{align*}
        & \forall \Paramw, \Paramw' \in \rset^{\nagent d}, \quad
        \norm{ \gfnoisy{\Paramw}{\funnoise{t}} - \gfnoisy{\Paramw'}{\funnoise{t}} }^2 \\
        & \qquad \qquad \qquad \leq \lip \left\langle \gfnoisy{\Paramw}{\funnoise{t}} - \gfnoisy{\Paramw'}{\funnoise{t}}, \Paramw - \Paramw' \right\rangle.
    \end{align*}
    In addition, the noise has a finite $p$-th moment at $\Paramstar$:
    $\forall t \ge 1, \quad
    \PE \left[ \norm{ \globnoise{t}{\Paramstar} }^p \right]^{1/p} \le \tau_p$.
\end{assumption}
\begin{assumption}[Smoothness of the noise covariance]\label{assum:C}
    Define the covariance operator
    $\CovarianceTens(\Paramw) \eqdef \PE \left[ \globnoise{1}{\Paramw}^{\otimes 2} \right]$.
    Assume that $\CovarianceTens$ is twice continuously differentiable.
    Furthermore, we define
    $\textstyle \Covariance(\paramw) \eqdef \frac{1}{\nagent} \sum_{k = 1}^\nagent \PE[ \big( \paramnoise_1^{(k)}(\paramw) \big)^{\otimes 2} ]$.
\end{assumption}
These assumptions are standard when analyzing SGD iterates as a Markov chain \citep{dieuleveut2020bridging,mangold2025scaffold,mangold2025refined}. They assume regularity of the noisy gradients, as well as bounded moments for the noise, which both hold in natural contexts such as least-squares regression or logistic regression under mild assumptions.
Under \Cref{assum:noise-filtration}, the sequence of iterates $(\globParam{t})_{t \ge 0}$ generated by \eqref{eq:DSGD} defines a time-homogeneous Markov chain on $\rset^{\nagent d}$ with a Markov kernel defined as $\kernel{\Paramw}{A}
    \eqdef \CPP{ \globParam{t+1} \in A }{ \globParam{t} = \Paramw }$ for any Borel set $A \subseteq \rset^{\nagent d}$.

\paragraph{Convergence to a Stationary Distribution.}
In the stochastic setting, \Cref{lemma:DGD-A} obviously no longer holds. However, we show below that the Markov chain $(\globParam{t})_{t \ge 0}$ converges in Wasserstein distance.
\begin{proposition}[Convergence of DSGD]\label{prop:DSGD}
    Assume
    \Cref{assum:functions},
    \Cref{assum:commmat},
    \Cref{assum:noise-filtration},
    and \Cref{assum:noise-cocoercive}(2).
    If $\step < 2/\lip$, then the Markov chain $(\globParam{t})_{t \ge 0}$ admits a unique stationary distribution $\meas{\step} \in \mathcal{P}_2(\rset^{\nagent d})$, i.e., with finite second moment.
    Moreover, for any initial distribution $\rho_0$ and any $t \ge 0$,
    \begin{align*}
        \wasserstein^2\big( \rho_0 \kernelfun^t, \meas{\step} \big)
        \le
        \left( 1 - 2 \step \strcvx \left( 1 - \lip \step / 2 \right) \right)^t
        \wasserstein^2\big( \rho_0, \meas{\step} \big).
    \end{align*}
\end{proposition}
We prove this result in \Cref{sec:proof-prop-markov-cv}.
\Cref{prop:DSGD} shows that DSGD converges to a stationary regime, independently of the initialization.
The convergence holds in the Wasserstein metric, with a geometric rate $1 - 2 \step \strcvx \left( 1 - \lip \step / 2 \right) \le 1 - \step\strcvx < 1$. As expected, smaller step sizes yield slower convergence, but also reduce the deterministic bias of the algorithm (see \Cref{prop:expansion-general}).

\paragraph{Variance of DSGD.}
We now analyze the stationary distribution $\meas{\step}$, by characterizing its mean and variance (with a slight abuse of language),
\begin{align*}
    \Paramsto
    \eqdef & \
    \textstyle\int_{\rset^{\nagent d}} \Paramw \, d\meas{\step}(\Paramw)
    \eqsp, \\
    \covstationary
    \eqdef & \
    \textstyle\int_{\rset^{\nagent d}} (\Paramw - \Paramlim)^{\otimes 2} \, d\meas{\step}(\Paramw)
    \eqsp.
\end{align*}
The matrix $\covstationary$ is of size $(\nagent d) \times (\nagent d)$. For $k, \ell \in \{1, \dots, \nagent\}$ we denote $[\covstationary]_{k,\ell}$ its $d\times d$-size block of indices $k$ and $\ell$.
We first give an expansion of the variance.
\begin{proposition}\label{prop:DSGD-variance-general}
    Assume
    \Cref{assum:functions},
    \Cref{assum:commmat},
    \Cref{assum:noise-filtration},
    \Cref{assum:noise-cocoercive}(4), \Cref{assum:C}.
    Then for any $k, \ell \in \{ 1, \dots, \nagent \}$, it holds that
    \begin{align*}
        [\covstationary]_{k,\ell}
        & =
        \frac{\step}{\nagent}
        \ophess
        \Covariance(\paramstar)
        + \mathcal{O}(\step^{3/2}) \eqsp,
    \end{align*}
    where we defined $\ophess = (\tensId \otimes \barA + \barA \otimes \tensId)^{-1}$.
\end{proposition}
To establish this proposition, we linearize the gradients in DSGD's updates, expand the square, and bound the remainder terms; we give the detailed proof in \Cref{sec:bias-var-sgd-proof}.
There are two key take-aways from this proposition: (i) at the first-order in the step size, DSGD's variance decreases with the number of clients, resulting in a linear speed-up, and (ii) this is the case \emph{regardless of the topology}.
In other words, for two different graph topologies, the variance of the stationary distribution is essentially the same, provided that the step size $\step$ is small enough.
Interestingly, this first-order expression exactly matches with classical federated learning (see Theorem 4 of \citet{mangold2025refined}).
The dominant contribution is determined by the noise and the local objectives, while the influence of the network structure only appears in higher-order terms.
To highlight this phenomenon, we derive the following bound on the limit variance.
\begin{lemma}
\label{lem:bound-variance-dsgd}
Assume \Cref{assum:functions}, \Cref{assum:commmat}, \Cref{assum:noise-filtration}, \Cref{assum:noise-cocoercive}(2).
Let $\step \le \min(1/\lip, 1/(\Lambda \lip))$, then the variance of DSGD at stationarity is bounded by
\begin{align*}
    & \frac{1}{\nagent} \sum_{k=1}^\nagent \norm{ [\covstationary]_{k,k}}
    \lesssim
    \frac{\step\tau_2^2 + \step^{3/2} \boundThird B^{3/2} }{\strcvx\nagent} 
    + 
    \frac{\step^2 \spectralgapabs^2}{1 - \spectralgapabs^2} \tau_2^2  C
    \eqsp,
\end{align*}
where $C = ( \lip B + \boundThird B^{3/2} \step^{1/2} + \tfrac{1}{2} \step^2 \boundThird^2 B^2 + \tau_2^2) / \tau_2^2$
and
$B =
\big( \tau_4^2 + \step^2 \frac{\lip^4}{\strcvx^2} \Lambda^2 \hgty^2 \big) /
\strcvx$. We use $  \lesssim$ to indicate an absolute numerical constant dependence.
\end{lemma}
To prove this bound, we leverage the decomposition we used to prove \Cref{prop:DSGD-variance-general}, and decompose it between a \emph{consensus}, which dominates and \emph{disagreement} part, which gives higher-order terms. 
We give a detailed proof in \Cref{sec:non-asymptotic-bounds-sup}.
The key takeaway of this lemma is that as long as
\begin{align*}
m = \mathcal{O}(1 / (\strcvx \step \Lambda^2))
\eqsp,
\end{align*}
the leading term is the term in $\step$, which does not depend on the topology. More precisely, as long as the number of clients is of order $O(1/\Lambda^2)$, the impact of the graph's topology on the variance is negligible.

\paragraph{Bias of DSGD.}
First, we remark that stochasticity does not affect the bias for \textit{quadratic} functions.
\begin{proposition}\label{prop:DSGD-bias-quadratic}
    Assume
    \Cref{assum:commmat},
    \Cref{assum:quadratic},
    \Cref{assum:noise-filtration},
    \Cref{assum:noise-cocoercive}(4), \Cref{assum:C}.
    Then stochasticity does not introduce any additional bias, and it holds that $\Paramsto = \Paramlim$.
\end{proposition}
The proof follows from the linearity of the gradient, see \Cref{sec:bias-var-sgd-proof}.
For general smooth and strongly convex functions, however, stochasticity induces an additional first-order bias in the stationary mean.
\begin{proposition}\label{prop:DSGD-bias-general}
    Suppose
    \Cref{assum:functions},
    \Cref{assum:commmat},
    \Cref{assum:noise-filtration},
    \Cref{assum:noise-cocoercive}(4), \Cref{assum:C}.
    Then for any $i \in \{ 1, \dots, \nagent \}$,
    \begin{align*}
        \!\!\paramlimsto{i} \!- \paramlimloc{i} 
        \!=\!
        - \frac{\step}{2\nagent} \hessf{\paramstar}^{-\!1} \terf{\paramstar}
        \ophess
        \Covariance(\paramstar) 
        + \mathcal{O}(\step^{3/2})\,,
    \end{align*}
    where $\ophess$ is defined in \Cref{prop:DSGD-variance-general}.
\end{proposition}
\Cref{prop:DSGD-bias-general} shows that for smooth and strongly convex functions, the mean of the stationary distribution does not coincide with $\Paramlim$: an additional bias of order $\mathcal{O}(\step)$ appears.
We emphasize two key points: (i) the stationary mean $\Paramsto$ still depends on the network topology, since $\Paramlim$ does; (ii) yet, the \emph{additional stochastic bias} $\Paramsto - \Paramlim$ and the stationary variance are \emph{independent of the topology at first order in $\step$}.

\paragraph{Non-Asymptotic Bounds}
We now state one of our main results, which gives a non-asymptotic convergence rate for DSGD based on bounds on the asymptotic variance and geometric ergodicity of DSGD.
\begin{theorem}\label{prop:non-asymptotic-bounds}
    Assume
    \Cref{assum:functions},
    \Cref{assum:commmat},
    \Cref{assum:noise-filtration},
    \Cref{assum:noise-cocoercive}(4), \Cref{assum:C}.
    Taking $\step \le \min( \tfrac{1}{\Lambda\lip}, \tfrac{1}{10\lip} )$, it holds that
    \begin{align*} 
         \PE[ \norm{ \globParam{t} - \Paramstar }^2 ]
        & \lesssim
        (1 - \step \strcvx)^t \psi_0
        + \step \tfrac{\tau_2^2}{\mu\nagent}
        +
        \step^{3/2} \tfrac{\boundThird B^{3/2}}{\strcvx\nagent} 
        \\
        &  \quad 
        + \step^2 \Big(
        \tfrac{\lip^2 \Lambda^2}{\strcvx^2} \hgty^2
        +
        \tfrac{\lip B\spectralgapabs^2}{1 - \spectralgapabs^2}
        + \tfrac{\spectralgapabs^2}{1 - \spectralgapabs^2} \tau_2^2
        \Big)
        \\
        &\quad 
        + 
        \step^{5/2} \tfrac{\spectralgapabs^2}{1 - \spectralgapabs^2}
        \big(
        \boundThird B^{3/2} + \step^{3/2} \boundThird^2 B^2
        \big)
        \eqsp,
    \end{align*}
    with $\psi_0 = \norm{\globParam{0} \!-\! \Paramstar}^2
        +
        \frac{\step^2 \lip^2 \Lambda^2}{\strcvx^2} \hgty^2 
        +
        \frac{\step }{\strcvx} ( \tau_2^2 + \step^2 \frac{4\lip^4}{\strcvx^2} \Lambda^2 \hgty^2 )$ and
    $B=
        \frac{1}{ \strcvx }
            \big( \tau_4^2 + \step^2 \frac{\lip^4}{\strcvx^2} \Lambda^2 \hgty^2 \big)$.
\end{theorem}
We prove this Theorem in \Cref{sec:app-non-asymptotic-analysis}. 
We stress that the proof scheme is \emph{fundamentally different} from existing convergence proofs of DSGD, as it does not rely on establishing a bound on the convergence of consensus/disagreement part, but rather studies convergence of the algorithm, together with its properties in its stationary regime.
We can now give a sample complexity for this algorithm.
\begin{corollary}
\label{coro:sample-complexity-dsgd}
Let $\epsilon > 0$ be small enough.
Under the assumptions of \Cref{prop:non-asymptotic-bounds}, set $\step = \min(\frac{1}{\lip}, \frac{\strcvx}{\lip^2 \Lambda \hgty}, \tfrac{\strcvx \nagent \epsilon^2 }{\tau_2^2},
\frac{\strcvx \epsilon}{\lip \Lambda \hgty}, \epsilon (\frac{1-\rho^2}{\lip B \rho^2})^{1/2}, \epsilon (\frac{1-\rho^2}{\tau_2^2 \rho^2})^{1/2})$,
the DSGD algorithm achieves $\PE[ \norm{ \globParam{t} - \Paramstar}^2 ] \le \epsilon^2$ after 
\begin{align*}
\textstyle
T \!\gtrsim\!  
\max(\frac{\lip}{\strcvx}, \frac{\lip^2 \Lambda \hgty}{\strcvx^2}, \!\tfrac{\tau_2^2}{\strcvx^2 \nagent \epsilon^2 }, \!\frac{\lip \Lambda \hgty}{\strcvx^2 \epsilon}, \!
\frac{\tau_2}{\strcvx\epsilon} (\tfrac{\rho^2}{1-\rho^2})^{1/2})
\log(\frac{\psi_0}{\epsilon})
\,.
\end{align*}
\end{corollary}
This corollary establishes the sample complexity of DSGD. It shows that, when high precision is desired, DSGD has linear speed-up until $\nagent \lesssim \min(\frac{\tau^2}{\lip \Lambda \hgty \epsilon}, \frac{((1-\rho^2)/\rho^2)^{1/2}}{\strcvx \epsilon})$.
In other words, provided the network topology is not too disconnected and the problem not too heterogeneous, DSGD achieves linear speed-up in the number of clients.

\begin{figure*}[t]
    \centering
    \begin{minipage}{0.85\linewidth}
        \centering
        \includegraphics[width=\linewidth]{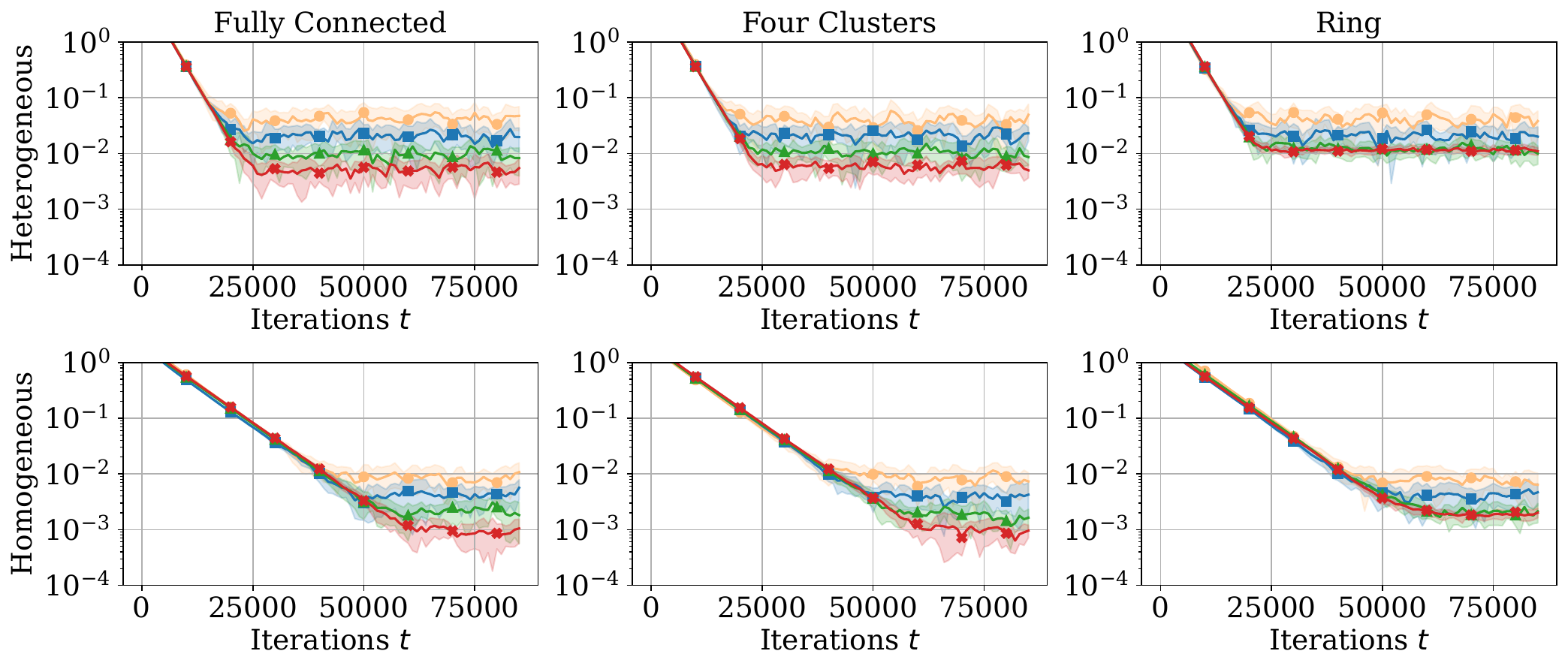}
    \end{minipage}%
    \hfill
    \begin{minipage}{0.1\linewidth}
        \centering
        \hspace{-3em}\includegraphics[width=1.3\linewidth]{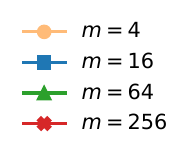}
    \end{minipage}
    
    \vspace{-0.7em}
    \caption{DSGD for heterogeneous (top row) and homogeneous (bottom row) clients, for various numbers of clients $\nagent$ and communication graphs.
    Graphs are fully connected (left), four clusters sparsely connected (middle), and ring (right).
    Colored areas indicate variations (± standard deviation) obtained from $20$ independent runs.}
    \label{fig:DSGD}
    
\vspace{-0.3em}
\end{figure*}

\section{RICHARDSON-ROMBERG FOR DECENTRALIZED LEARNING}
\label{sec:RR}
Building on our analysis of DSGD's bias, we extend the Richardson-Romberg extrapolation technique to the decentralized world. The main idea of this technique, initially developed in numerical analysis, is to cancel leading-order error terms by combining iterates obtained with different step sizes.
It has since found applications in diverse areas, like data science \citep{bach2021effectiveness}, and in the analysis of SGD algorithms, both in the single-client setting \citep{dieuleveut2020bridging} and in federated learning \citep{ mangold2025refined}.

To construct the decentralized Richardson-Romberg estimator, we set a step size $\step \le 1/\lip$, and run the algorithm two times, giving two sequences of iterates: one with step size $\step$ and one with step size $\step / 2$, denoted respectively by $(\globParam{t}^\step)_{t\ge0}$ and $(\globParam{t}^{\step/2})_{t\ge0}$.
We then define the \emph{Richardson-Romberg extrapolated iterate} as
\begin{align}
\label{eq:rr-dsgd}
    \globParamRR{t}
\eqdef 2 \globParam{t}^{\step/2} - \globParam{t}^\step
\eqsp.
\end{align}
We stress that building this estimator only requires running the algorithm twice. Contrary to gradient tracking methods \citep{di2016next,koloskova2021improved}, it does not require clients to maintain a slack variable in memory.
This makes this algorithm fit for running on devices with low-resources, which are typical in decentralized settings.

Leveraging our expression of DSGD's bias, we show that the estimator built in \eqref{eq:rr-dsgd} has reduced bias.
\begin{proposition}\label{prop:RR-det}
    Assume \Cref{assum:functions}-\ref{assum:commmat}.
    If $\step \le \min\left(\frac{1}{\Lambda\lip}, \frac{1}{\lip} \right)$,
    \begin{align*}
        \norm{
        \ParamlimRR
        -
        \Paramstar
        } 
        & \le
        \textstyle
        \step^2 \tfrac{\lip^2}{\strcvx^2} \Lambda^2
        \big(
        \tfrac{\boundThird}{\strcvx\sqrt{\nagent}} \hgty^2
        + \lip \hgty
        \big)
        \eqsp.
    \end{align*}
\end{proposition}
This shows that Richardson-Romberg extrapolation successfully eliminates the deterministic first-order bias term $\step \cdot (\tensId - \hetMat) \melMat \gF{\Paramstar}$, which depends on both the topology of the communication graph and the heterogeneity.
Remarkably, this does not require explicit knowledge of the graph structure or of the heterogeneity.
We then obtain the following convergence rate.
   \begin{corollary}
    \label{coro:conv-rr-dgd}
    Assume \Cref{assum:functions}-\ref{assum:commmat}.
    Then, for any $t\ge0$, in the deterministic setting, we have the convergence rate
    \begin{align*}
        \norm{\globParamRR{t} - \Paramstar}
        & \le
        3 (1 - \step\strcvx / 2)^t \norm{\globParam{0} - \Paramstar} \\
        & \quad
        + 2(1 - \step\strcvx / 2)^t \frac{\step \lip \Lambda}{\strcvx} \hgty 
        \\
        & \quad
        + \step^2 \frac{5\lip^2}{8\strcvx^2} \Lambda^2
        \big(
        \frac{\boundThird}{\strcvx\sqrt{\nagent}} \hgty^2
        + \lip \hgty
        \big)
        \eqsp.
    \end{align*}
For $\epsilon > 0$, let $\step = \min(\tfrac{1}{\lip}, \tfrac{1}{\Lambda L}, \tfrac{\strcvx \epsilon^{1/2}}{\lip \Lambda \hgty})$ and assume $\boundThird/m^{1/2}$ is small, then we have $\norm{ \globParamRR{t} - \Paramstar } \le \epsilon$ after $T \gtrsim \max(\tfrac{\lip}{\strcvx}, \tfrac{\Lambda \lip}{ \strcvx}, \tfrac{\lip \Lambda \hgty}{\strcvx^2 \epsilon^{1/2}} ) \log(\tfrac{\norm{ \globParam{0}- \Paramstar}}{\epsilon})$ iterations.
\end{corollary}
This shows that the speed of convergence is not affected, but that the limit bias is quadratically reduced, allowing for a quadratic improvement in the sample complexity when heterogeneity is large or the network not very connected.
Finally, we stress that similar results would hold in the stochastic setting, allowing to speed-up convergence when the variance is dominated by the decentralization error.
Thus, this allows to handle decentralized learning with much less connected graphs, de facto reducing communications.

In the stochastic setting, combining \Cref{coro:conv-rr-dgd} with \Cref{lem:bound-variance-dsgd} gives the following result.
\begin{corollary}
\label{coro:sample-complexity-dsgd-RR}
Let $\epsilon > 0$ be small enough.
Under the assumptions of \Cref{prop:non-asymptotic-bounds}, set $\step = \min(\frac{1}{\lip}, \frac{\strcvx}{\lip^2 \Lambda \hgty}, \tfrac{\strcvx \nagent \epsilon^2 }{\tau_2^2},
\frac{\strcvx \epsilon^{1/2}}{\lip \Lambda \hgty}, \!\epsilon (\frac{1-\rho^2}{\lip B \rho^2})^{1/2}, \!\epsilon (\frac{1-\rho^2}{\tau_2^2 \rho^2})^{\sfrac{1}{2}})$,\!
the RR-DSGD algorithm achieves $\PE[ \norm{ \globParamRR{t} - \Paramstar}^2 ] \le \epsilon^2$ for 
\begin{align*}
\textstyle
T \!\gtrsim\!  
\max(\frac{\lip}{\strcvx}, \!\frac{\lip^2 \Lambda \hgty}{\strcvx^2}, \tfrac{\tau_2^2}{\strcvx^2 \nagent \epsilon^2 }, \frac{\lip \Lambda \hgty}{\strcvx^2 \epsilon^{1/2}},
\frac{\tau_2}{\strcvx\epsilon} (\tfrac{\rho^2}{1-\rho^2})^{1/2})
\log(\frac{\psi_0}{\epsilon})
\,.
\end{align*}
\end{corollary}
This shows that the Richardson-Romberg extrapolation procedure has the same advantages in the stochastic regime and in the deterministic one.
Notably, it can be particularly beneficial in settings where noise is dominated by the bias due to decentralization.

\section{EXPERIMENTS}
\label{sec:experiments}
This section illustrates our theoretical results through numerical experiments. We consider the optimization problem \eqref{eq:minimization-problem}, where for client $k \in \{1, \dots, \nagent\}$, the local objective is given by$\nf{k}{\paramw} = \frac{1}{n} \sum_{i = 1}^n \log(1 + \exp(\langle \paramw, \paramw_{k, i} \rangle)) + \frac{\lambda}{2} \norm{\paramw}^2$.
For each client $k$, the vectors $\paramw_{k, i}$ are sampled independently from a distribution specific to that client, which enables us to model heterogeneity across the network.
All experiments are conducted in dimension $d = 2$, and the communication graph is chosen to be either
(i) fully connected, with $\commmat = \frac{1}{\nagent} \oneVec\oneVec^\top$),
(ii) a ring topology, or
(iii) a clustered topology with four well-connected clusters that are only sparsely connected to each other.

\textbf{Deterministic DGD.}
In the deterministic setting, \Cref{lemma:DGD-A} and \Cref{prop:expansion-general} show that using a smaller step size $\step$ reduces the bias, but slows down convergence.
Moreover, \Cref{prop:RR-det} shows that that Richardson-Romberg extrapolation improves the bias order from $\mathcal{O}(\step)$ for classical DGD to $\mathcal{O}(\step^2)$.
This is indeed what we observe in \Cref{fig:RR}, where $\nagent = 12, \step = 10^{-3}$, with the four-clusters graph.

\begin{figure}[t]
    \centering
    \begin{minipage}{0.5\linewidth}
        \centering
        \includegraphics[width=\linewidth]{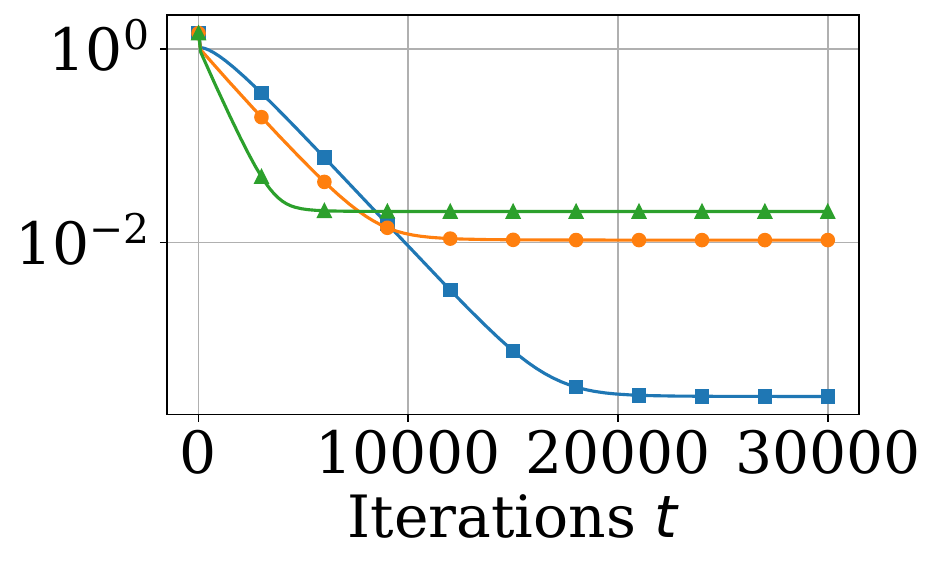}
    \end{minipage}%
    \begin{minipage}{0.5\linewidth}
        \centering
        \includegraphics[width=\linewidth]{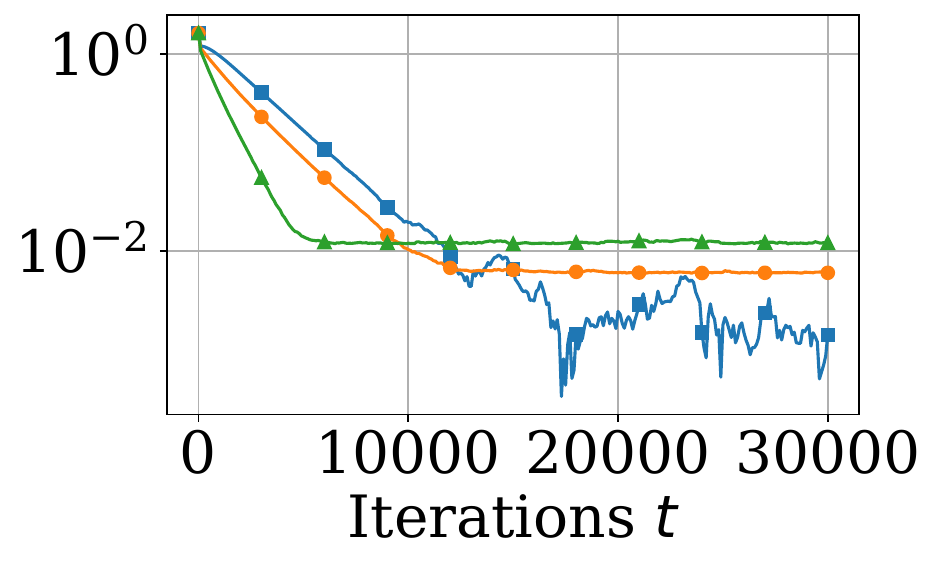}
    \end{minipage}
    \vspace{-0.3em}
    
    \begin{minipage}{\linewidth}
        \centering
        \includegraphics[width=0.8\linewidth]{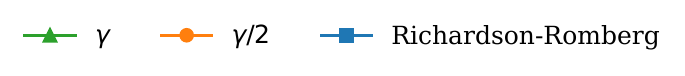}
    \end{minipage}
    \caption{$\frac{1}{m} \sum_{i=1}^m\left\|\theta_i^t-\theta^*\right\|$ for deterministic (left) and stochastic (right) DGD with step size $\step$, step size $\step/2$ and RR extrapolation.}
    \label{fig:RR}
    \vspace{-1em}
\end{figure}

\textbf{Stochastic DGD.}
In the stochastic setting, \Cref{prop:DSGD-variance-general} and \Cref{prop:DSGD-bias-general} establish that both the variance and the bias scale as $1/\nagent$ at the first-order.
We illustrated this numerically by plotting the evolution of the distance to the deterministic limit $\Paramlim$ as a function of the iteraiton number, for different $\nagent$ and graphs.
As observed in \Cref{fig:DSGD} ($\step = 10^{-3}$), increasing the number of clients yields smaller variance and bias.
However, for poorly connected graphs like the ring topology, the improvement stops when the bias begins: since the second-largest eigenvalue of $\commmat$ is very close to $1$, higher-order error that depend on the graph prevent further speed-up.

\textbf{Decentralized Richardson-Romberg.}
We now illustrate the performance of the decentralized Richardson-Romberg extrapolation method.
In both deterministic and stochastic settings, it effectively reduces DSGD's bias by one order of magnitude.
In \Cref{fig:RR} (left), we run the deterministic variant of the algorithm, confirming a reduction of the bias from $10^{-2}$ to $10^{-4}$. The stochastic counterpart of the algorithm, reported in \Cref{fig:RR} (right), also reduces the bias: there, bias strongly hinders DSGD's convergence, while DSGD with Richardson-Romberg extrapolation reaches a stationary regime, where variance dominates.

\section{CONCLUSION}
\label{sec:conclusion}
In this paper, we analyzed the deterministic and stochastic Decentralized Gradient Descent algorithm, focusing on the characterization of their limiting points and the impact of heterogeneity, network topology, and stochastic noise.
In the deterministic setting, we showed that DGD's bias depends on the step size, network topology, and heterogeneity, with a bias of order $\mathcal{O}(\step)$ that can be reduced to $\mathcal{O}(\step^2)$ with Richardson–Romberg extrapolation.
In the stochastic setting, the iterates admit a unique stationary distribution with an additional $\mathcal{O}(\step)$ bias, and have a linear speed-up in the number of clients even without formally averaging the iterates.

Overall, our results reveal a clear contrast: deterministic bias is shaped by network effects, while stochastic bias is dominated by noise. This deepens our understanding of the impact of decentralization and stochasticity on optimization, and hints towards novel methods like, e.g., Richardson–Romberg extrapolation. 
We proved that this method reduces DSGD's sample complexity,  notably in small noise regimes. We hope that our analysis of DSGD will serve as a basis for the development of decentralized algorithms tailored to handle stochasticity.

\section*{ACKNOWLEDGEMENTS} 
The work of Aymeric Dieuleveut and Lucas Versini is supported by French State aid managed by the Agence Nationale de la Recherche (ANR) under France 2030 program with the reference ANR-23-PEIA-005 (REDEEM project). This work is supported by Hi! PARIS and ANR/France 2030 program (ANR-23-IACL-0005), in particular Hi!Paris FLAG chair.

\bibliography{references}
\clearpage


\include{empty}

\appendix
\thispagestyle{empty}

\onecolumn
\aistatstitle{Supplementary Materials}


In this appendix, we first provide additional related work.
\begin{itemize}
    \item \Cref{app:deterministic} contains the proofs of the deterministic DGD results stated in \Cref{sec:deterministic}.
    
    \item In \Cref{appendix:DSGD}, we give the proofs of the results on DSGD stated in \Cref{sec:stochastic}. We also state and prove additional results:
    \begin{itemize}
        \item \Cref{lemma:control-R2} and \Cref{lemma:control-R4} give contraction results between consecutive iterates of DSGD, as well as upper-bounds of moments of order 2 and 4 under the stationary distribution.
        \item Based on these first two lemmas, \Cref{lemma:first-general-variance} establishes a first expansion of the stationary distribution's variance.
        \item From this expression of the variance, \Cref{lemma:first-general-bias} deduces an expression of the stationary distribution's bias.
        \item \Cref{lemma:correlation} and \Cref{lemma:correlation-2} give estimates of the covariance of the noise under the stationary distribution, based on \Cref{lemma:control-R2}, \Cref{lemma:control-R4} and \Cref{lemma:first-general-bias}.
        \item Based on these previous lemmas, \Cref{corollary:variance-expansion} and \Cref{lemma:simplify-variance} then simplify the expression of the variance, which yields both \Cref{prop:DSGD-variance-general} and \Cref{prop:DSGD-bias-general}.
    \end{itemize}

    \item
    \Cref{sec:app-non-asymptotic-analysis} contains the proofs of the non-asymptotic results stated in \Cref{sec:stochastic}.

    \item
    In \Cref{sec:app-RR}, we give the proofs of the results on Richardson-Romberg extrapolation.

    \item
    \Cref{app:matrices} contains two general lemmas on matrices, used in the proofs of the previous results.
\end{itemize}

\section*{Additional Related Work}

\paragraph{Decentralized Optimization.}
The term decentralized SGD was introduced by 
\citet{lian2017can}.
Nonetheless, similar algorithms had been studied in \emph{distributed} optimization, with the first analysis by 
\citet{nedic2009distributed}, and algorithms like 
dual averaging \citep{duchi2011dual}, 
Extra \citep{shi2015extra}, and other algorithms mentioned below. 
Many analyses \citep{hendrikx2021optimal} 
and variants of decentralized SGD have then been studied, with variance reduction \citep{tang2018d}, 
compression \citep{koloskova2019decentralized}, 
for neural networks \citep{assran2019stochastic}, 
with changing topology \citep{koloskova2020unified,le2023refined}, 
biased gradients \citep{jiang2025on}, 
asynchronous updates \citep{even2024asynchronous}
or projections \citep{choi2025convergence}. 
Other works have studied the generalization of decentralized SGD
\citep{le2024improved,ye2025generalization}.
All these approaches still suffer from bias due to decentralization and heterogeneity, and some methods have been proposed to mitigate this bias
\citep{zhang2019decentralized,pu2021distributed,koloskova2021improved}.
Other algorithms have also been proposed, like DIGing
\citep{nedic2017achieving}, 
NIDS \citep{li2019decentralized}, 
together with refined analyses of these methods \citep{jakovetic2018unification,xu2020unified,li2020revisiting}
A long line of work have been dedicated to the study of dual algorithms 
\citep{scaman2017optimal,uribe2020dual,kovalev2020optimal} 


\paragraph{Distributed and Parallel Stochastic Gradient Descent.}

The methods from decentralized optimization find their roots in distributed optimization \citep{tsitsiklis2003distributed,nedic2009distributed,boyd2011distributed}.
In this context, the goal is essentially to accelerate the learning by leveraging the computational power of multiple machines.
A key difficulty in this setting is to handle asynchronous updates, which cause delay \citep{zinkevich2009slow}. 
This was studied to learn conditional entropy models
\citep{mcdonald2009efficient}, 
distributed neural networks \citep{mcdonald2010distributed,dean2012large}, 
and for parallel SGD in general
\citep{zinkevich2010parallelized,agarwal2014reliable,agarwal2011distributed}.
and more generally in the context of asynchronous SGD without locks
\citep{recht2011hogwild,duchi2013estimation,leblond2017asaga,de2015taming,j2015variance,mania2017perturbed}. 
Note that in the more specific setting of federated learning (with a central server) similar ideas of linear speed-up in the number of clients have been observed \citep{yang2021achieving, qu2021federated,mangold2025refined,mangold2025scaffold}.

\paragraph{Fixed-Point Analyses of Decentralized Learning.}

Closely related to our work, \citet{vogels2022beyond} studied the convergence of decentralized GD to the solution of an equation involving the gradient and the network, providing convergence rates for DSGD to a neighborhood of this point in expectation.
Similarly, \citet{larsson2025unified} proved the convergence to a fixed point in the deterministic setting.
Additionally, \citet{yuan2016convergence} showed convergence of decentralized GD by interpreting it as gradient descent on a modified function.

In the federated learning setting, when a central server is available, FedAvg has also been studied under the fixed-point framework, first in the deterministic setting \citep{malinovskiy2020local,wang2021local}, establishing convergence to a point, which can be related to the global solution \citep{malinovskiy2020local,charles2021convergence,pathak2020fedsplit}, with an explicit characterization of the bias in the quadratic case.
This fixed-point framework was later extended to the stochastic setting, providing new analyses of FedAvg and Scaffold \citep{mangold2025refined,mangold2025scaffold}.

\paragraph{Other Approaches in Federated Learning.}
When a central server can orchestrate the training, decentralized learning boils down to federated learning \citep{mcmahan2017communication}.
The most prominent algorithm in this setting is the FedAvg \citep{mcmahan2017communication}, which has been widely studied.
A first line of work has studied this method in the homogeneous setting, where clients share the same objective function \citep{stich2019local,wang2018cooperative,haddadpour2019convergence,yu2019parallel,wang2018cooperative,li2019communication}. 
Then, many works studied the properties of FedAvg when clients differ, inducing a client-drift phenomenon \citep{karimireddy2020scaffold}. Many heterogeneity measures have been proposed, based on first-order information \citep{yu2019linear,khaled2020tighter,karimireddy2020scaffold,reddi2021adaptive,zindari2023convergence,crawshaw2024federated}, second-order similarity \citep{arjevani2015communication,khaled2020tighter}, relaxed first-order heterogeneity \citep{glasgow2022sharp}, and average drift at the optimum \citep{wang2024Unreasonable,patel2023still}.

In analyses of FedAvg with stochastic gradients, a bias appears as higher order terms \citep{khaled2020tighter,glasgow2022sharp, wang2024Unreasonable}.
This bias essentially stems from the bias of SGD, which has been widely studied in the single-client setting \citep{lan2012optimal, defossez2015averaged, dieuleveut2016nonparametric, chee2017convergence}, and through Markov chain analyses of SGD \citep{dieuleveut2020bridging,pflug1986stochastic}.

\paragraph{Richardson-Romberg.}
Richardson-Romberg extrapolation, which we extend to the decentralized setting for the first time, can be traced back to \citet{richardson1911ix}. This is a classical method in numerical analysis, with many applications, including time-varying autoregressive processes \citep{moulines2005recursive}, data science \citep{bach2021effectiveness}, and many others \citep{stoer1980introduction}.
It was brought to stochastic approximation by \citet{dieuleveut2020bridging}, and later studied by \citet{sheshukova2024nonasymptotic}.
Most closely related to our work, \citet{mangold2025refined} proposed a variant of this method in the context of federated learning, when a central server is available.

\section{DETERMINISTIC DECENTRALIZED GRADIENT DESCENT}\label{app:deterministic}

Hereafter, we recall the proof of the convergence in the deterministic case~\citep{larsson2025unified}. 

\label{sec:det-dgd}
\begin{proof}[Proof of \Cref{lemma:DGD-A}]
    The proof consists in noticing that the operator $\Paramw \mapsto \tenscommmat\big( \Paramw - \step\gF{\Paramw} \big)$ is $(1 - \step\strcvx)$-Lipschitz, which comes from
    \begin{align*}
        \norm{
        \tenscommmat\big( \Paramw - \step\gF{\Paramw} \big) - \tenscommmat\big( \Paramw' - \step\gF{\Paramw'} \big)
        }
        & =
        \norm{
        \tenscommmat\big( \Paramw - \Paramw' - \step\gF{\Paramw} + \step\gF{\Paramw'} \big)
        } \\
        & \le
        \norm{
        \Paramw - \Paramw' - \step(\gF{\Paramw} - \gF{\Paramw'})
        } \\
        & =
        \norm{
        (\tensId - \step A) (\Paramw - \Paramw')
        }
        \eqsp,
    \end{align*}
    where $A \eqdef \displaystyle\int_0^1 \hF{\Paramw' + t(\Paramw - \Paramw')} \, \mathrm{d}t$ satisfies $\strcvx\tensId \preceq A \preceq \lip\tensId$.

    For $\step < 1/\lip$, $\norm{\tensId - \step A}_2 \le 1 - \step\strcvx$, hence
    \begin{align*}
        \norm{
        \tenscommmat\big( \Paramw - \step\gF{\Paramw} \big) - \tenscommmat\big( \Paramw' - \step\gF{\Paramw'} \big)
        }
        & \le
        (1 - \step\strcvx)
        \norm{\Paramw - \Paramw'}
        \eqsp.
    \end{align*}
    We conclude using Banach fixed-point theorem.
\end{proof}
\medskip

\begin{proof}[Proof of \Cref{lem:stationary-def-general}]
    Applying the projection $\projconsensus$ on the consensus space to \eqref{eq:fixed-point}, and using the fact that $\projconsensus \tenscommmat = \projconsensus$ since $\commmat$ is doubly stochastic, we have $\projconsensus \gF{\Paramlim} = 0$, which gives a first equation 
    \begin{align*}
        \frac{1}{\nagent} \sum_{k=1}^\nagent \ngf{k}{\paramlimloc{k}} = 0
        \eqsp.
    \end{align*}
    Next, since $\cParamlim = \tenscommmat \cParamlim$, the fixed-point equation \eqref{eq:fixed-point} implies
    \[
        (\tensId - \tenscommmat) \dParamlim = - \step \tenscommmat \gF{\Paramlim}
        \eqsp,
    \]
    which gives the expression for $\dParamlim$ by applying the Moore-Penrose pseudoinverse of $\tensId - \tenscommmat$.
\end{proof}
\medskip

\begin{proof}[Proof of \Cref{lemma:deter-quad}]
    From \Cref{lem:stationary-def-general}, we know
    \begin{align*}
        \frac{1}{\nagent} \sum_{k=1}^\nagent \ngf{k}{\paramlimloc{k}}
        & =
        \frac{1}{\nagent} \sum_{k=1}^\nagent \nbarA{k} (\paramlimloc{k} - \locparamstar{k}) = 0
        \eqsp.
    \end{align*}
    Using the decomposition $\paramlimloc{k} = \cparamlim + \dparamlimloc{k}$, and the expression of the gradient for quadratic functions, we obtain
    \begin{align*}
        \barA \cparamlim 
        = 
        \frac{1}{\nagent} \sum_{k=1}^\nagent \nbarA{k} \locparamstar{k}
        - \frac{1}{\nagent} \sum_{k=1}^\nagent \nbarA{k} \dparamlimloc{k}
        \eqsp,
    \end{align*}
    which can be rewritten in compact matrix form, using the matrix $\hetMat$ defined in \Cref{lemma:deter-quad} and the fact that $\barA$ is invertible, as
    \begin{align}\label{eq:bar-tilde}
        \cParamlim
        =
        \hetMat \locParamstar
        - \hetMat \dParamlim
        \eqsp.
    \end{align}
    \Cref{lem:stationary-def-general} also provides the following equation:
    \begin{align*}
        \dParamlim 
        &= - \step (\tensId - \tenscommmat)^\dagger \tenscommmat \gF{\Paramlim}
        = - \step \melMat \Avec (\Paramlim - \locParamstar).
    \end{align*}
    Plugging in $\Paramlim = \cParamlim + \dParamlim$, we get:
    \begin{align*}
        (\tensId + \step \melMat \Avec) \dParamlim
        =
        - \step \melMat \Avec (\cParamlim - \locParamstar).
    \end{align*}
    Plugging this in \eqref{eq:bar-tilde}, we obtain:
    \begin{align*}
        \dParamlim 
        &= - \step (\tensId + \step \melMat \Avec)^{-1} \melMat \Avec \left( (\hetMat - \tensId) \locParamstar - \hetMat \dParamlim \right).
    \end{align*}
    We then define
    \[
    \melMatA := (\tensId + \step \melMat \Avec)^{-1} \melMat \Avec,
    \]
    and we get the closed-form
    \begin{align*}
        \dParamlim 
        &= \step (\tensId - \step \melMatA \hetMat)^{-1} \melMatA (\tensId - \hetMat) \locParamstar,
    \end{align*}
    Substituting back into \eqref{eq:bar-tilde} gives the expression for $\cParamlim$, using the fact that $\hetMat\locParamstar = \Paramstar$.
    \medskip
    
    To find the upper-bound on $\step$, we need to make sure that all the above inverse matrices are well-defined.
    The matrix $(\tensId + \step\melMat\Avec)^{-1}$ is well-defined when $\step < 1 / (\lip \norm{\melMat}_2)$.
    Moreover, $\norm{\melMatA}_2
    \le \lip\norm{\melMat}_2 / (1 - \step\lip\norm{\melMat}_2)$,
    and
    $\norm{\hetMat}_2 \le \lip/\strcvx$,
    so $(\tensId - \step\melMatA\hetMat)^{-1}$ is well-defined as soon as 
    $\step < \strcvx (1 - \step\lip\norm{\melMat}_2) / (\lip^2 \norm{\melMat}_2)$, or equivalently
    $\step < \dfrac{1}{(1+\lip/\strcvx) \lip\norm{\melMat}_2}$.
    
    The first-order expansions in $\step$ then follow from $\melMatA = \melMat \Avec + \mathcal{O}(\step)$.
\end{proof}
\medskip


\begin{proof}[Proof of \Cref{coro:conv-dgd}]
    The convergence rate is a direct consequence of \Cref{lemma:DGD-A} and \Cref{lemma:upper-bound-error}:
    \begin{align*}
        \norm{\globParam{t} - \Paramstar}
        & \le
        \norm{\globParam{t} - \Paramlim}
        + \norm{\Paramlim - \Paramstar} \\
        & \le
        (1 - \step\strcvx)^t \norm{\globParam{0} - \Paramlim}
        + \norm{\Paramlim - \Paramstar} \\
        & \le
        (1 - \step\strcvx)^t \norm{\globParam{0} - \Paramstar}
        + \frac{2 \step \lip \Lambda}{\strcvx} \hgty
        \eqsp.
    \end{align*}
    This rate holds if
    $\step
    \le\min(\tfrac{1}{\lip}, \tfrac{1}{\Lambda L})$.
    If moreover
    $\step \le \tfrac{\strcvx \epsilon}{\lip \Lambda \hgty}$, then the above bound becomes
    \begin{align*}
        \norm{\globParam{t} - \Paramstar}
        & \le
        (1 - \step\strcvx)^T \norm{\globParam{0} - \Paramstar}
        + \varepsilon
        \eqsp.
    \end{align*}
    We then need to have
    $(1 - \step\strcvx)^T 
    \lesssim \dfrac{\varepsilon}{\norm{\globParam{0} - \Paramstar}}$,
    or equivalently
    $T
    \gtrsim
    \dfrac{\log(\varepsilon / \norm{\globParam{0} - \Paramstar})}{\log(1 - \step\strcvx)}
    \gtrsim
    \dfrac{\log(\norm{\globParam{0} - \Paramstar} / \varepsilon)}{\step\strcvx}$.
    We conclude using the value on $\step$.
\end{proof}
\medskip

\begin{proof}[Proof of \Cref{prop:expansion-general}]
    We follow the approach used in the quadratic case. Using \Cref{lem:stationary-def-general}, a Taylor expansion, \Cref{lemma:upper-bound-error}, and $\paramlimloc{k} = \cparamlim + \dparamlimloc{k}$, it holds that
    \begin{align}
    \nonumber
        0
        & =
        \frac{1}{\nagent} \sum_{k=1}^\nagent \ngf{k}{\paramlimloc{k}} \\
    \nonumber
        & =
        \frac{1}{\nagent} \sum_{k=1}^\nagent \big(
        \ngf{k}{\paramstar}
        + \nhf{k}{\paramstar}(\paramlimloc{k} -\paramstar)
        + \mathcal{O}(\norm{\paramlimloc{k} -\paramstar}^2) \big) \\
    \label{eq:dev-first-order-det-general-expr-expand-eq}
        & =
        \underbrace{\frac{1}{\nagent} \sum_{k=1}^\nagent \ngf{k}{\paramstar}}_{= 0}
        + \barA \cparamlim
        + \frac{1}{\nagent} \sum_{k=1}^\nagent \nhf{k}{\paramstar} \dparamlimloc{k}
        - \barA \paramstar
        + R_1(\step)
        \eqsp,
    \end{align}
    where $R_1(\step) = \mathcal{O}(\step^2)$, and more precisely, using \Cref{assum:functions}, 
    $\norm{R_1(\step)}
    \le \dfrac{1}{2\nagent} \sum\limits_{k=1}^\nagent \boundThird \norm{\paramlimloc{c} - \paramstar}^2 = \dfrac{\boundThird}{2\nagent} \norm{\Paramlim- \Paramstar}^2$.
    Rearranging \eqref{eq:dev-first-order-det-general-expr-expand-eq} yields
    \begin{align*}
        \cparamlim
        = \paramstar
        - \barA^{-1} \left( \frac{1}{\nagent} \sum_{k=1}^\nagent \nhf{k}{\paramstar} \dparamlimloc{k} \right)
        + \barA^{-1} R_1(\step)
        \eqsp,
    \end{align*}
    which in tensorized form reads:
    \begin{align}\label{eq:bar-tilde-gen}
        \cParamlim
        = \Paramstar - \hetMat \dParamlim
        + \oneVec \otimes \barA^{-1} R_1(\step) \eqsp .
    \end{align}
    Now, using the second part of \Cref{lem:stationary-def-general}, we have
    \begin{align*}
        \dParamlim
        &= -\step (\tensId - \tenscommmat)^\dagger \tenscommmat \gF{\Paramlim} \\
        &= -\step \melMat \left( \gF{\Paramstar} + R_2(\step) \right) \\
        &= -\step \melMat \gF{\Paramstar}
        - \step \melMat R_2(\step)\eqsp,
    \end{align*}
    where
    $R_2(\step)
    = \mathcal{O}(\Paramlim - \Paramstar)
    = \mathcal{O}(\step)$,
    and more precisely, using \Cref{assum:functions},
    $\norm{R_2(\step)}
    \le \lip \norm{\Paramlim - \Paramstar}$.
    
    Substituting back into \eqref{eq:bar-tilde-gen}, we obtain:
    \begin{align*}
        \cParamlim
        &=
        \big( \Paramstar + \step \hetMat \melMat \gF{\Paramstar} \big)
        +
        \big(
        \oneVec \otimes \barA^{-1} R_1(\step)
        +
        \step \hetMat \melMat R_2(\step)
        \big)\eqsp ,
    \end{align*}
    where the residual term is indeed $\mathcal{O}(\step^2)$.
    More precisely, using \Cref{lemma:upper-bound-error},
    \begin{align*}
        \norm{
        \oneVec \otimes \barA^{-1} R_1(\step)
        +
        \step \hetMat \melMat R_2(\step)
        }
        & \le
        \sqrt{\nagent} \norm{\barA^{-1}}_2 \cdot \norm{R_1(\step)}
        + \step \norm{\hetMat}_2 \cdot \norm{\melMat}_2 \cdot \norm{R_2(\step)} \\
        & \le
        \frac{1}{\strcvx\sqrt{\nagent}} \cdot \dfrac{\boundThird}{2} \norm{\Paramlim - \Paramstar}^2
        + \step \frac{\lip}{\strcvx} \cdot \norm{\melMat}_2 \cdot \lip \norm{\Paramlim - \Paramstar} \\
        & \le
        \step^2 \frac{\lip^2}{\strcvx^2} \Lambda
        \Big(
        \dfrac{\boundThird}{2\strcvx\sqrt{\nagent}} \Lambda \hgty^2
        + \lip \norm{\melMat}_2 \cdot \hgty
        \Big)
        \eqsp.
    \end{align*}

\end{proof}

\bigskip

\section{DECENTRALIZED STOCHASTIC GRADIENT DESCENT}
\label{appendix:DSGD}

\subsection{Convergence of DSGD}
\label{sec:proof-prop-markov-cv}

\begin{proof}[Proof of \Cref{prop:DSGD}]
    The proof follows closely the structure of the argument in \cite[Proposition 2]{dieuleveut2020bridging}. The key difference lies in the contraction between consecutive time steps, which we adapt below to our setting:
    \begin{align}
    \nonumber
        \PE\left[ \norm{ \globParama{t+1} - \globParamb{t+1} }^2 \right]
        & =
        \PE\left[ \norm{ \tenscommmat(\globParama{t} - \globParamb{t}) - \step \tenscommmat \big( \gfnoisy{\globParama{t}}{\funnoise{t}} - \gfnoisy{\globParamb{t}}{\funnoise{t}} \big) }^2 \right] \\
    \nonumber
        & \overset{(i)}{\leq}
        \PE\left[ \norm{ (\globParama{t} - \globParamb{t}) - \step \big( \gfnoisy{\globParama{t}}{\funnoise{t}} - \gfnoisy{\globParamb{t}}{\funnoise{t}} \big) }^2 \right] \\
    \nonumber
        & =
        \PE\left[ \norm{ \globParama{t} - \globParamb{t} }^2
        + \step^2 \norm{  \gfnoisy{\globParama{t}}{\funnoise{t}} - \gfnoisy{\globParamb{t}}{\funnoise{t}} }^2 \right] \\
        & \quad - 2\step \PE\left[ \left\langle \globParama{t} - \globParamb{t}, \gfnoisy{\globParama{t}}{\funnoise{t}} - \gfnoisy{\globParamb{t}}{\funnoise{t}}  \right\rangle \right] \\
    \nonumber
        & \overset{(ii)}{\leq}
        \PE\left[ \norm{ \globParama{t} - \globParamb{t} }^2
        - 2\step(1 - \lip\step/2) \left\langle \globParama{t} - \globParamb{t}, \gfnoisy{\globParama{t}}{\funnoise{t}} - \gfnoisy{\globParamb{t}}{\funnoise{t}} \right\rangle \right] \\
    \nonumber
        & \overset{(iii)}{=}
        \PE\left[ \norm{ \globParama{t} - \globParamb{t} }^2
        - 2\step(1 - \lip\step/2) \left\langle \globParama{t} - \globParamb{t}, \gF{\globParama{t}} - \gF{\globParamb{t}} \right\rangle \right] \\
    \label{eq:bound-sto-contract}
        & \overset{(iv)}{\leq}
        \big( 1 - 2\strcvx\step(1 - \lip\step/2) \big) \PE\left[ \norm{ \globParama{t} - \globParamb{t} }^2 \right]
        \eqsp,
    \end{align}
    where:
    \begin{itemize}
        \item $(i)$ uses $\| \commmat \|_2 = 1$, since $\commmat$ is doubly stochastic.
        \item $(ii)$ comes from the co-coercivity $\Cref{assum:noise-cocoercive}(2)$.
        \item $(iii)$ is obtained by taking the expected value conditional on $\calF_t$, using the fact that $\globParama{t}$ and $\globParamb{t}$ are $\calF_t$-measurable, and that $\CPE{\gfnoisy{\globParam{t}}{\funnoise{t}}}{\calF_t} = \gF{\globParam{t}}$, using $\Cref{assum:noise-filtration}$.
        \item $(iv)$ comes from strong convexity $\Cref{assum:functions}$ and from $\step \le 2/\lip$.
    \end{itemize}

    For any probability measures $\rho_1, \rho_2$ with finite second order moment, \citet[Theorem 4.1]{OT} shows the existence of random variables $\globParama{0}$ and $\globParamb{0}$ which are independent of $(\funnoise{t})_{t\geq1}$ and such that $\wasserstein^2(\rho_1, \rho_2) = \PE\left[ \norm{\globParama{0} - \globParamb{0}}^2 \right]$.

    By definition of $\wasserstein$, and using the fact that the distribution of $\globParama{t}$ (resp. $\globParamb{t}$) is $\rho_1 \kernelfun^t$ (resp. $\rho_2 \kernelfun^t$),  we have $\wasserstein^2(\rho_1 \kernelfun^t, \rho_2 \kernelfun^t) \leq \PE\left[ \norm{\globParama{t} - \globParamb{t}}^2 \right]$.
    Then, \eqref{eq:bound-sto-contract} yields
    \begin{align}\label{eq:wasserstein-geometric}
        \wasserstein^2(\rho_1 \kernelfun^t, \rho_2 \kernelfun^t)
        & \leq \big( 1 - 2\strcvx\step(1 - \lip\step/2) \big)^t \, \wasserstein^2 (\rho_1, \rho_2)
        \eqsp,
    \end{align}
    and taking $\rho_2 = \rho_1 \kernelfun$ shows that $\displaystyle
    \sum_{t = 1}^{+\infty} \wasserstein^2 (\rho_1 \kernelfun^t, \rho_1 \kernelfun^{t+1})
    \leq \sum_{t = 1}^{+\infty} \rho^t \wasserstein^2 (\rho_1, \rho_1 \kernelfun)
    < +\infty$.

    Then, since by \citet[Theorem 6.16]{OT} the set of probability measures with finite second moment endowed with $\wasserstein$ is a Polish space, we get that $(\rho_1 \kernelfun^t)_{t\geq1}$ is a Cauchy sequence, which converges to some probability measure with finite second moment $\meas{\step}^{\rho_1}$: $\lim\limits_{t\to+\infty} \wasserstein(\rho_1 \kernelfun^t, \meas{\step}^{\rho_1}) = 0$.

    Morevoer, if we assume that $\lim\limits_{t\to+\infty} \wasserstein(\rho_2 \kernelfun^t, \meas{\step}^{\rho_2}) = 0$ for some $\meas{\step}^{\rho_2}$, then using \eqref{eq:wasserstein-geometric} we get:
    \[
        \wasserstein(\meas{\step}^{\rho_1}, \meas{\step}^{\rho_2})
        \leq
        \wasserstein(\meas{\step}^{\rho_1}, \rho_1 \kernelfun^t)
        + \wasserstein(\rho_1 \kernelfun^t, \rho_2 \kernelfun^t)
        + \wasserstein(\rho_2 \kernelfun^t, \meas{\step}^{\rho_2})
        \underset{t\to+\infty}{\longrightarrow} 0
        \eqsp,
    \]
    so $\meas{\step} := \meas{\step}^{\rho_1} = \meas{\step}^{\rho_2}$ does not depend on $\rho_1$.
    Finally,
    $\wasserstein(\meas{\step} \kernelfun, \meas{\step})
    \leq
    \wasserstein(\meas{\step} \kernelfun, \meas{\step} \kernelfun^t)
    + \wasserstein(\meas{\step} \kernelfun^t, \meas{\step})
    \underset{t\to+\infty}{\longrightarrow} 0$,
    so $\meas{\step} R = \meas{\step}$, and $\meas{\step}$ is indeed a stationary distribution.
    The uniqueness comes from \eqref{eq:wasserstein-geometric}.
\end{proof}
\clearpage

\subsection{Variance and Bias of DSGD}
\label{sec:bias-var-sgd-proof}

\begin{proof}[Proof of \Cref{prop:DSGD-bias-quadratic}]
    Using
    \begin{align*}
        \globParam{1}
        = \tenscommmat \left( \globParam{0}
        - \step \big( \gF{\globParam{0}} + \globnoise{1}{\globParam{0}} \big) \right)
        \eqsp,
    \end{align*}
    integrating over $\globParam{0} \sim \meas{\step}$ and taking the expected value on $\funnoise{1}$, it holds that
    \begin{align*}
        (\tensId - \tenscommmat) \Paramsto
        & =
        -\step \tenscommmat \int_{\rset^{\nagent d}} \gF{\Paramw} \, \meas{\step}(d\Paramw)
        \eqsp.
    \end{align*}
    Since the gradient is linear for quadratic functions, we have
    \begin{align*}
        (\tensId - \tenscommmat) \Paramsto
        & =
        -\step \tenscommmat \gF{\textstyle{\int_{\rset^{\nagent d}}} \Paramw \, \meas{\step}(d\Paramw)} 
       =
        -\step \tenscommmat \gF{\Paramsto}
        \eqsp.
    \end{align*}
    This is precisely the fixed-point equation \eqref{eq:fixed-point}, which has a unique solution: $\Paramlim$. Hence $\Paramsto = \Paramlim$.
%
\end{proof}
\medskip

To establish the expression of $\Paramsto$ and $\displaystyle \int_{\rset^{\nagent d}} (\Paramw - \Paramlim)^{\otimes 2} \, \pi_\step(d\Paramw)$ in the general case, we need a few technical lemmas.
\medskip

Note that in many of the lemmas below, we give explicit upper-bounds on residual terms. These upper-bounds are not necessary to prove \Cref{prop:DSGD-variance-general} or \Cref{prop:DSGD-bias-general}, and could easily be replaced by $\mathcal{O}$. They are only given here for completeness.
\bigskip

\begin{lemma}\label{lemma:control-R2}
    For $\step < 1/\lip$, and under $\Cref{assum:functions}, \Cref{assum:noise-filtration}$, and $\Cref{assum:noise-cocoercive}(2)$:
    \[
        \CPE{ \norm{\globParam{t+1} - \Paramlim}^2}{\calF_t}
        \le
        (1 - 2\step\strcvx(1 - \lip\step)) \norm{\globParam{t} - \Paramlim}^2
        + 2\tau_2(\Paramlim)^2 \step^2
        \eqsp,
    \]
    where
    \[
        \tau_2(\Paramlim)
        =
        \PE \left[ \norm{ \globnoise{1}{\Paramlim} }^2 \right]^{1/2}\eqsp.
    \]
    Moreover
    \[
        \int_{\rset^{\nagent d}} \norm{\Paramw - \Paramlim}^2 \, \meas{\step}(d\Paramw)
        \le
        \frac{\step\tau_2(\Paramlim)^2}{\strcvx(1 - \step\lip)}\eqsp.
    \]
    If $\step \le 1 / (\Lambda\lip)$, then
    \[
        \int_{\rset^{\nagent d}} \norm{\Paramw - \Paramlim}^2 \, \meas{\step}(d\Paramw)
        \le
        \frac{2 \step \left( \tau_2^2 + \step^2 \frac{4\lip^4}{\strcvx^2} \Lambda^2 \hgty^2 \right)}{\strcvx(1 - \step\lip)}\eqsp,
    \]
    where $\tau_2$ is defined in \Cref{assum:noise-cocoercive}.
\end{lemma}
\begin{proof}
By definition of $\globParam{t+1}$ and $\Paramlim$ and expanding the square, we have
    \begin{align*}
        \CPE{ \norm{\globParam{t+1} - \Paramlim}^2}{\calF_t}
        & =
        \CPE{\norm{ \tenscommmat (\globParam{t} - \Paramlim) - \step \tenscommmat (\gfnoisy{\globParam{t}}{\funnoise{t+1}} - \gF{\Paramlim}) }^2}{\calF_t} \\
        & \leq
        \CPE{\norm{ (\globParam{t} - \Paramlim) - \step (\gfnoisy{\globParam{t}}{\funnoise{t+1}} - \gF{\Paramlim}) }^2}{\calF_t} \\
        & =
        \norm{ \globParam{t} - \Paramlim }^2
        + \step^2 \CPE{\norm{ \gfnoisy{\globParam{t}}{\funnoise{t+1}} - \gF{\Paramlim} }^2}{\calF_t} \\
        & \quad - 2\step \CPE{\left\langle \globParam{t} - \Paramlim, \gfnoisy{\globParam{t}}{\funnoise{t+1}} - \gF{\Paramlim} \right\rangle}{\calF_t}
        \eqsp.
    \end{align*}
    Then, using \Cref{assum:noise-filtration} and \Cref{assum:noise-cocoercive}:
    \begin{align*}
        \CPE{\norm{ \gfnoisy{\globParam{t}}{\funnoise{t+1}} - \gF{\Paramlim} }^2}{\calF_t}
        & \leq
        2 \CPE{\norm{ \gfnoisy{\globParam{t}}{\funnoise{t+1}} - \gfnoisy{\Paramlim}{\funnoise{t+1}} }^2}{\calF_t}
        + 2 \CPE{\norm{ \globnoise{t+1}{\Paramlim} }^2}{\calF_t} \\
        & \leq
        2\lip \CPE{\left\langle \gfnoisy{\globParam{t}}{\funnoise{t+1}} - \gfnoisy{\Paramlim}{\funnoise{t+1}}, \globParam{t} - \Paramlim \right\rangle }{\calF_t}
        + 2 \tau_2(\Paramlim)^2 \\
        & =
        2\lip \left\langle \gF{\globParam{t}} - \gF{\Paramlim}, \globParam{t} - \Paramlim \right\rangle
        + 2 \tau_2(\Paramlim)^2
        \eqsp.
    \end{align*}
    Using \Cref{assum:functions}, we then obtain
    \begin{align*}
        \CPE{ \norm{\globParam{t+1} - \Paramlim}^2}{\calF_t}
        & \leq
        \norm{ \globParam{t} - \Paramlim }^2
        + 2\step^2\tau_2(\Paramlim)^2
        - 2\step(1 - \lip\step) \left\langle \gF{\globParam{t}} - \gF{\Paramlim}, \globParam{t} - \Paramlim \right\rangle \\
        & \leq
        (1 - 2\step\strcvx(1 - \lip\step)) \norm{ \globParam{t} - \Paramlim }^2
        + 2\tau_2(\Paramlim)^2\step^2
        \eqsp,
    \end{align*}
    hence the first inequality.
    The second statement is obtained by considering $ \PE[ \norm{\globParam{t} - \Paramlim}^2]  \wedge M$ for $M > 0$, taking the limit as $t\to+\infty$ in the previous inequalities, and letting $M\to+\infty$ using the monotone convergence theorem.


    To prove the third inequality, we now upper-bound $\tau_2(\Paramlim)^2$.
    Using \Cref{assum:functions} and \Cref{assum:noise-cocoercive}(2), the noise is almost surely $2\lip$-Lipschitz continuous, hence
    \[
        \norm{\globnoise{t}{\Paramlim} - \globnoise{t}{\Paramstar} }
        \le
        2\lip \norm{\Paramlim - \Paramstar}
        \eqsp,
    \]
    from which we deduce
    \begin{align*}
        \left| \tau_2(\Paramlim)^2 - \tau_2(\Paramstar)^2 \right|
        & =
        \left| \PE\Big[
        \norm{ \globnoise{t}{\Paramlim} }^2 - \norm{ \globnoise{t}{\Paramstar} }^2
        \Big] \right| \\
        & \le
        \PE\Big[
        \norm{\globnoise{t}{\Paramlim} - \globnoise{t}{\Paramstar} } \cdot (\norm{\globnoise{t}{\Paramlim} } + \norm{ \globnoise{t}{\Paramstar}})
        \Big] \\
        & \le
        \PE\Big[
        \norm{\globnoise{t}{\Paramlim} - \globnoise{t}{\Paramstar} } \cdot (\norm{\globnoise{t}{\Paramlim} - \globnoise{t}{\Paramstar}} + 2\norm{ \globnoise{t}{\Paramstar}})
        \Big] \\
        & \le
        4\lip \norm{\Paramlim - \Paramstar}
        \big(
        \lip \norm{\Paramlim - \Paramstar} + \tau_1
        \big),
    \end{align*}
    and using $\tau_1 \le \tau_2$ and \Cref{lemma:upper-bound-error},
    \begin{align*}
        \tau_2(\Paramlim)^2
        & \le
        \tau_2^2
        + 4\lip \norm{\Paramlim - \Paramstar}
        \big(
        \lip \norm{\Paramlim - \Paramstar} + \tau_2
        \big) \\
        & =
        \big( \tau_2 + 2 \lip \norm{\Paramlim - \Paramstar})^2 \\
        & \le
        \Big( \tau_2 + \step \frac{2 \lip^2}{\strcvx} \Lambda \hgty \Big)^2
        \eqsp.
    \end{align*}
\end{proof}

\begin{lemma}\label{lemma:control-R4}
    For $\step < 1/(10\lip)$, and under $\Cref{assum:functions}, \Cref{assum:noise-filtration}$, and $\Cref{assum:noise-cocoercive}(4)$,
    \[
        \PE\left[ \norm{\globParam{t+1} - \Paramlim}^4 \right]^{1/2}
        \le
        (1 - 2\step\strcvx(1 - 9\lip\step))^t \PE\left[ \norm{\globParam{0} - \Paramlim}^4 \right]^{1/2}
        + \frac{100\step\tau_4(\Paramlim)^2}{\strcvx}
        \eqsp,
    \]
    where
    \[
        \tau_4(\Paramlim)
        =
        \PE \left[ \norm{ \globnoise{t}{\Paramlim} }^4 \right]^{1/4}
        \eqsp.
    \]
    Moreover
    \[
        \int_{\rset^{\nagent d}} \norm{\Paramw - \Paramlim}^4 \, \meas{\step}(d\Paramw)
        \le
        \Big( \frac{100 \step\tau_4(\Paramlim)^2}{\strcvx} \Big)^2
        \eqsp.
    \]
    If $\step \le 1 / (\Lambda\lip)$ and $\step < 1/(10\lip)$, then
    \[
        \int_{\rset^{\nagent d}} \norm{\Paramw - \Paramlim}^4 \, \meas{\step}(d\Paramw)
        \le
        \frac{
        250 \times 100^2\step^2
        \big( \tau_4^4 + \step^4 \frac{\lip^8}{\strcvx^4} \Lambda^4 \hgty^4 \big)
        }{
        \strcvx^2
        }.
    \]
\end{lemma}
\begin{proof}
Using Cauchy-Schwarz inequality:
\begin{align*}
    \CPE{\norm{\globParam{t+1} - \Paramlim}^4}{\calF_t}
    & =
    \CPE{\norm{ \tenscommmat \big( (\globParam{t} - \Paramlim) - \step (\gfnoisy{\globParam{t}}{\funnoise{t+1}} - \gF{\Paramlim}) \big) }^4}{\calF_t} \\
    & \leq
    \CPE{\norm{ (\globParam{t} - \Paramlim) - \step (\gfnoisy{\globParam{t}}{\funnoise{t+1}} - \gF{\Paramlim}) }^4}{\calF_t} \\
    & =
    \norm{ \globParam{t} - \Paramlim }^4
    +
    \step^4 \CPE{
    \norm{ \gfnoisy{\globParam{t}}{\funnoise{t+1}} - \gF{\Paramlim} }^4}{\calF_t} \\
    &
    + 4\step^2 \CPE{\left\langle \globParam{t} - \Paramlim, \gfnoisy{\globParam{t}}{\funnoise{t+1}} - \gF{\Paramlim} \right\rangle^2 }{\calF_t} \\
    & + 2\step^2 \norm{ \globParam{t} - \Paramlim }^2 \CPE{ \norm{ \gfnoisy{\globParam{t}}{\funnoise{t+1}} - \gF{\Paramlim} }^2 }{\calF_t} \\
    & - 4\step \norm{ \globParam{t} - \Paramlim }^2 \CPE{ \left\langle \globParam{t} - \Paramlim, \gfnoisy{\globParam{t}}{\funnoise{t+1}} - \gF{\Paramlim} \right\rangle}{\calF_t} \\
    & - 4\step^3 \CPE{\norm{ \gfnoisy{\globParam{t}}{\funnoise{t+1}} - \gF{\Paramlim} }^2 \left\langle \globParam{t} - \Paramlim, \gfnoisy{\globParam{t}}{\funnoise{t+1}} - \gF{\Paramlim} \right\rangle}{\calF_t} \\ \\
    & \leq
    \norm{ \globParam{t} - \Paramlim }^4
    - 4\step \norm{ \globParam{t} - \Paramlim }^2 \left\langle \globParam{t} - \Paramlim, \gF{\globParam{t}} - \gF{\Paramlim} \right\rangle \\
    & +
    \step^4 \CPE{ \norm{ \gfnoisy{\globParam{t}}{\funnoise{t+1}} - \gF{\Paramlim} }^4}{\calF_t} \\
    & + 6\step^2 \norm{ \globParam{t} - \Paramlim }^2 \CPE{ \norm{\gfnoisy{\globParam{t}}{\funnoise{t+1}} - \gF{\Paramlim} }^2 }{\calF_t} \\
    & + 4\step^3 \norm{ \globParam{t} - \Paramlim } \CPE{\norm{ \gfnoisy{\globParam{t}}{\funnoise{t+1}} - \gF{\Paramlim} }^3 }{\calF_t}
    \eqsp.
\end{align*}
Then, we use $(x + y)^n \le 2^{n-1}(x^n + y^n)$ for $n \in \{ 2, 3, 4 \}$, $\tau_n \le \tau_4$, and \Cref{assum:noise-cocoercive}:
\begin{align*}
    & \CPE{ \norm{ \gfnoisy{\globParam{t}}{\funnoise{t+1}} - \gF{\Paramlim} }^n}{\calF_t} \\
    & \leq
    2^{n-1} \CPE{ \norm{ \gfnoisy{\globParam{t}}{\funnoise{t+1}} - \gfnoisy{\Paramlim}{\funnoise{t+1}} }^n}{\calF_t}
    + 2^{n-1} \CPE{ \norm{ \globnoise{t+1}{\Paramlim} }^n}{\calF_t} \\ 
    & \leq
    2^{n-1} \lip^{n-2} \norm{ \globParam{t} - \Paramlim }^{n-2} \CPE{ \norm{ \gfnoisy{\globParam{t}}{\funnoise{t+1}} - \gfnoisy{\Paramlim}{\funnoise{t+1}} }^2}{\calF_t}
    + 2^{n-1} \tau_4(\Paramlim)^n \\ 
    & \leq
    2^{n-1} \lip^{n-1} \norm{ \globParam{t} - \Paramlim }^{n-2} \left\langle \globParam{t} - \Paramlim, \gF{\globParam{t}} - \gF{\Paramlim} \right\rangle
    + 2^{n-1} \tau_4(\Paramlim)^n\eqsp.
\end{align*}
Plugging this in the previous inequalities:
\begin{align*}
    \CPE{\norm{\globParam{t+1} - \Paramlim}^4}{\calF_t}
    & \leq
    \norm{ \globParam{t} - \Paramlim }^4 \\
    &
    - 4\step \big( 1 - 3 \lip\step - 4\lip^2\step^2 - 2\lip^3\step^3 \big) \norm{ \globParam{t} - \Paramlim }^2 \left\langle \globParam{t} - \Paramlim, \gF{\globParam{t}} - \gF{\Paramlim} \right\rangle \\
    & + 8 \tau_4(\Paramlim)^4 \step^4
    + 12 \tau_4(\Paramlim)^2 \step^2 \norm{ \globParam{t} - \Paramlim }^2
    + \underbrace{16 \tau_4(\Paramlim)^3 \step^3 \norm{ \globParam{t} - \Paramlim }}_{\color{gray} \le 8 \tau_4(\Paramlim)^2 \step^2 (\tau_4(\Paramlim)^2 \step^2 + \norm{ \globParam{t} - \Paramlim }^2)} \\
    & \leq
    \big( 1 - 4\strcvx \step \big( 1 - 3 \lip\step - 4\lip^2\step^2 - 2\lip^3\step^3 \big) \big) \norm{ \globParam{t} - \Paramlim }^4 \\
    & + 16 \tau_4(\Paramlim)^4 \step^4
    + 20 \tau_4(\Paramlim)^2 \step^2 \norm{ \globParam{t} - \Paramlim }^2
    \eqsp.
\end{align*}
Then, we define $c_t = \PE\left[ \norm{ \globParam{t} - \Paramlim }^4 \right]^{1/2}$, and use
$\PE\left[ \norm{ \globParam{t} - \Paramlim }^2 \right]
\le c_t$,
$1 - 2x \leq (1 - x)^2$
and
$1 - 2\strcvx \step \big( 1 - 3 \lip\step - 4\lip^2\step^2 - 2\lip^3\step^3 \big)
\ge 1 - 2 L \step
\ge 1 - 2/10
\ge 1/2$:
\begin{align*}
    c_{t+1}^2
    & \leq
    \big( 1 - 4\strcvx \step \big( 1 - 3 \lip\step - 4\lip^2\step^2 - 2\lip^3\step^3 \big) \big) c_t^2
    + 16 \tau_4(\Paramlim)^4 \step^4
    + 20 \tau_4(\Paramlim)^2 \step^2 c_t \\
    & \leq
    \big( 1 - 2\strcvx \step \big( 1 - 3 \lip\step - 4\lip^2\step^2 - 2\lip^3\step^3 \big) \big)^2 c_t^2
    + 16 \tau_4(\Paramlim)^4 \step^4
   \\
   & \qquad\qquad + 40 \tau_4(\Paramlim)^2 \step^2 \big( 1 - 2\strcvx \step \big( 1 - 3 \lip\step - 4\lip^2\step^2 - 2\lip^3\step^3 \big) c_t \\
    & \leq
    \big( \big( 1 - 2\strcvx \step \big( 1 - 3 \lip\step - 4\lip^2\step^2 - 2\lip^3\step^3 \big) \big) c_t + 20\tau_4(\Paramlim)^2\step^2 \big)^2 \\
    & \leq
    \big( \big( \underbrace{1 - 2\strcvx \step \big( 1 - 9 \lip\step \big)}_{= r(\step)} \big) c_t + 20\tau_4(\Paramlim)^2\step^2 \big)^2
    \eqsp,
\end{align*}
\ie,
$c_{t + 1} \le r(\step) c_t + 20\tau_4(\Paramlim)^2\step^2$.
Therefore
$c_t
\le r(\step)^t c_0 + \dfrac{10\tau_4(\Paramlim)^2\step}{\strcvx ( 1 - 9 \lip\step )}
\le r(\step)^t c_0 + \dfrac{100\tau_4(\Paramlim)^2\step}{\strcvx}$, which proves the first inequality.

The second inequality is obtained similarly to \Cref{lemma:control-R2}.

For the third inequality, we upper-bound $\tau_4(\Paramlim)^4$.
Using Cauchy-Schwarz inequality,
\begin{align*}
    \left| \tau_4(\Paramlim)^4 - \tau_4(\Paramstar)^4 \right|
    & =
    \left| \PE \left[ \norm{ \globnoise{1}{\Paramlim} }^4
    - \norm{ \globnoise{1}{\Paramstar} }^4 \right]
    \right| \\
    & \le
    \PE \Big[
    \big| \norm{ \globnoise{1}{\Paramlim} }^2 - \norm{ \globnoise{1}{\Paramstar} }^2 \big|
    \cdot
    \big( \norm{ \globnoise{1}{\Paramlim} }^2 + \norm{ \globnoise{1}{\Paramstar} }^2 \big)
    \Big] \\
    & \le
    \sqrt{
    \PE \left[
    \big( \norm{ \globnoise{1}{\Paramlim} }^2 - \norm{ \globnoise{1}{\Paramstar} }^2 \big)^2
    \right]
    \PE \left[
    \big( \norm{ \globnoise{1}{\Paramlim} }^2 + \norm{ \globnoise{1}{\Paramstar} }^2 \big)^2
    \right]
    }
    \eqsp.
\end{align*}
For
$\PE \left[
\big( \norm{ \globnoise{1}{\Paramlim} }^2 - \norm{ \globnoise{1}{\Paramstar} }^2 \big)^2
\right]$,
using \Cref{assum:noise-cocoercive},
$\forall a, b \ge 0, (a + b)^2 \le 2(a^2 + b^2)$,
and the upper-bound on $\tau_2(\Paramlim)^2$ established in the proof of \Cref{lemma:control-R2}:
\begin{align*}
    \PE \left[
    \big( \norm{ \globnoise{1}{\Paramlim} }^2 - \norm{ \globnoise{1}{\Paramstar} }^2 \big)^2
    \right]
    & =
    \PE \left[
    \Big( \norm{ \globnoise{1}{\Paramlim} - \globnoise{1}{\Paramstar} } \cdot \norm{ \globnoise{1}{\Paramlim} + \globnoise{1}{\Paramstar} } \Big)^2
    \right] \\
    & \le
    4\lip^2 \norm{\Paramlim - \Paramstar}^2
    \PE \left[
    \norm{ \globnoise{1}{\Paramlim} + \globnoise{1}{\Paramstar} }^2
    \right] \\
    & \le
    8\lip^2 \norm{\Paramlim - \Paramstar}^2
    \Big(
    \PE \left[
    \norm{ \globnoise{1}{\Paramlim} }^2
    \right]
    +
    \PE \left[
    \norm{ \globnoise{1}{\Paramstar} }^2
    \right]
    \Big) \\
    & \le
    8\lip^2 \norm{\Paramlim - \Paramstar}^2
    \Big(
    \tau_2(\Paramlim)^2
    +
    \tau_2^2
    \Big) \\
    & \le
    8\lip^2 \norm{\Paramlim - \Paramstar}^2
    \Big(
    \big( \tau_2 + 2 \lip \norm{\Paramlim - \Paramstar} \big)^2
    +
    \tau_2^2
    \Big) \\
    & \le
    2^6 \times \lip^2 \norm{\Paramlim - \Paramstar}^2
    \big( \tau_2^2 + \lip^2 \norm{\Paramlim - \Paramstar}^2 \big)
    \eqsp.
\end{align*}

For
$\PE \left[
\big( \norm{ \globnoise{1}{\Paramlim} }^2 + \norm{ \globnoise{1}{\Paramstar} }^2 \big)^2
\right]$,
using
\Cref{assum:noise-cocoercive} and,
for $a, b \ge 0$,
$(a + b)^2 \le 2(a^2 + b^2)$
and
$(a + b)^4 \le 8(a^4 + b^4)$:
\begin{align*}
    \PE \left[
    \big( \norm{ \globnoise{1}{\Paramlim} }^2 + \norm{ \globnoise{1}{\Paramstar} }^2 \big)^2
    \right]
    & \le
    2
    \PE \left[
    \norm{ \globnoise{1}{\Paramlim} }^4
    +
    \norm{ \globnoise{1}{\Paramstar} }^4
    \right] \\
    & \le
    2
    \PE \left[
    8 \norm{ \globnoise{1}{\Paramlim} - \globnoise{1}{\Paramstar} }^4
    +
    9 \norm{ \globnoise{1}{\Paramstar} }^4
    \right] \\
    & \le
    2^8 \lip^4 \norm{ \Paramlim - \Paramstar }^4
    +
    18 \tau_4^4 \\
    & \le
    2^8 \big( \lip^4 \norm{ \Paramlim - \Paramstar }^4
    +
    \tau_4^4 \big)
    \eqsp.
\end{align*}
\smallskip

We thus obtain, using, for $a, b \ge 0$,
$(a^2 + b^2)(a^4 + b^4)
\le 2 (a^6 + b^6)$,
$\sqrt{a + b} \le \sqrt{a} + \sqrt{b}$
and
$\tau_2 \le \tau_4$,
\begin{align*}
    \left| \tau_4(\Paramlim)^4 - \tau_4(\Paramstar)^4 \right|
    & \le
    \sqrt{
    \PE \left[
    \big( \norm{ \globnoise{1}{\Paramlim} }^2 - \norm{ \globnoise{1}{\Paramstar} }^2 \big)^2
    \right]
    \PE \left[
    \big( \norm{ \globnoise{1}{\Paramlim} }^2 + \norm{ \globnoise{1}{\Paramstar} }^2 \big)^2
    \right]
    } \\
    & \le
    \sqrt{
    2^{14} \times \lip^2 \norm{\Paramlim - \Paramstar}^2
    \big( \tau_2^2 + \lip^2 \norm{\Paramlim - \Paramstar}^2 \big)
    \big(
    \lip^4 \norm{ \Paramlim - \Paramstar }^4
    +
    \tau_4^4
    \big)
    } \\
    & \le
    \sqrt{
    2^{15} \times \lip^2 \norm{\Paramlim - \Paramstar}^2
    \big( \tau_4^6 + \lip^6 \norm{\Paramlim - \Paramstar}^6 \big)
    } \\
    & \le
    200
    \lip \norm{\Paramlim - \Paramstar}
    \big( \tau_4^3 + \lip^3 \norm{\Paramlim - \Paramstar}^3 \big)
    \eqsp.
\end{align*}
%
%
Using $\forall a, b \ge 0$,
$a^4 + 200 b(a^3 + b^3) \le 151 a^4 + 250 b^4 \le 250 (a^4 + b^4)$, and \Cref{lemma:upper-bound-error}:
\begin{align*}
    \tau_4(\Paramlim)^4
    & \le
    \tau_4^4
    +
    200
    \lip \norm{\Paramlim - \Paramstar}
    \big( \tau_4^3 + \lip^3 \norm{\Paramlim - \Paramstar}^3 \big) \\
    & \le
    250 (\tau_4^4
    +
    \lip^4 \norm{\Paramlim - \Paramstar}^4) \\
    & \le
    250 \left( \tau_4^4
    +
    \step^4 \frac{\lip^8}{\strcvx^4} \Lambda^4 \hgty^4 \right).
\end{align*}
\end{proof}

\begin{remark}\label{remark:control-moment-3}
    Using Jensen's inequality,
    \begin{align*}
        \int_{\rset^{\nagent d}} \norm{\Paramw - \Paramlim}^3 \, \meas{\step}(d\Paramw)
        & \le
        \left( \int_{\rset^{\nagent d}} \norm{\Paramw - \Paramlim}^4 \, \meas{\step}(d\Paramw) \right)^{3/4} \\
        & \le
        \frac{
        2^6 100^{3/2}\step^{3/2}
        \big( \tau_4^3 + \step^3 \frac{\lip^6}{\strcvx^3} \Lambda^3 \hgty^3 \big)
        }{
        \strcvx^{3/2}
        }
        \eqsp.
    \end{align*}
\end{remark}


We now establish a first expansion of the stationary distribution's variance, which depends on the topology through $\Paramlim$. This dependence will be removed in later results.
\medskip

\begin{lemma}\label{lemma:first-general-variance}
    Suppose
    \Cref{assum:functions},
    \Cref{assum:commmat},
    \Cref{assum:noise-filtration},
    \Cref{assum:noise-cocoercive}(4) and \Cref{assum:C} hold.
    Then the following expansion holds as $\step\to0$:
    \begin{align*}
        & \int_{\rset^{\nagent d}} (\Paramw - \Paramlim)^{\otimes 2} \, \pi_\step(d\Paramw) \\
        & =
        \step \Big( (\projconsensus \otimes \projconsensus) \big( \tensId\otimes\hF{\Paramlim} + \hF{\Paramlim}\otimes\tensId \big) (\projconsensus \otimes \projconsensus) \Big)^\dagger
        \int_{\rset^{\nagent d}} \PE[\globnoise{1}{\Paramw}^{\otimes 2}] \, \pi_\step(d\Paramw)
        + \mathcal{O}(\step^{3/2})
        \eqsp.
    \end{align*}
    Moreover, if $\step \le \min\left( \dfrac{1}{\Lambda\lip}, \dfrac{1}{10\lip} \right)$, then
    \begin{align*}
        & \bnorm{
        \int_{\rset^{\nagent d}} (\Paramw - \Paramlim)^{\otimes 2} \, \pi_\step(d\Paramw)
        -
        \step \Big( (\projconsensus \otimes \projconsensus) \big( \tensId\otimes\hF{\Paramlim} + \hF{\Paramlim}\otimes\tensId \big) (\projconsensus \otimes \projconsensus) \Big)^\dagger
        \int_{\rset^{\nagent d}} \PE[\globnoise{1}{\Paramw}^{\otimes 2}] \, \pi_\step(d\Paramw)
        } \\
        & \lesssim
        \step^{3/2}
        \left(
            \frac{
            100 \boundThird C^2
            \tau_4^4
            }{
            \strcvx^3
            }
            + \frac{\boundThird}{8\strcvx}
        \right)
        +
        \step^2
        \tau_2^2
        \left(
            1
            + \frac{3\lip^2}{\strcvx^2}
            + \dfrac{1}{1 - \lambda_2(\commmat)}
            \Big( 1 + \dfrac{\lip}{\strcvx} \Big)^2
        \right)
        \eqsp,
    \end{align*}
    where we only kept the lowest orders in $\step$.
\end{lemma}
\begin{proof}
Starting from the update equation with $t = 0$, we use a Taylor expansion:
\begin{align}
\nonumber
    \globParam{1} - \Paramlim
    & =
    \tenscommmat (\globParam{0} - \Paramlim)
    + (\tenscommmat - \tensId) \Paramlim
    - \step \tenscommmat \big( \gF{\Paramlim} + \hF{\Paramlim}(\globParam{0} - \Paramlim) + \globnoise{1}{\globParam{0}} + R(\globParam{0}) \big) 
    \\
\nonumber
    & =
    \tenscommmat (\globParam{0} - \Paramlim)
    - \step \tenscommmat \big( \hF{\Paramlim}(\globParam{0} - \Paramlim) + \globnoise{1}{\globParam{0}} + R(\globParam{0}) \big)
    \\
    \label{eq:decomp-update-1}
    & =
    \tenscommmat (\tensId - \step \hF{\Paramlim}) (\globParam{0} - \Paramlim)
    - \step \tenscommmat \big( \globnoise{1}{\globParam{0}} + R(\globParam{0}) \big)
    \eqsp,
\end{align}
where for each $\Paramw\in\rset^{\nagent d}$,
\begin{align}\label{eq:upper-bound-R-mat}
    \norm{R(\Paramw)}
    & \le \frac{\boundThird}{2} \norm{ \Paramw - \Paramlim }^2
    \eqsp,
\end{align}
with $\boundThird$ introduced in \Cref{assum:functions}.

We integrate over $\globParam{0} \sim \meas{\step}$, take the expected value on $\paramnoise_1$ and use \Cref{lemma:control-R2} and \Cref{lemma:control-R4}:
\begin{align}\label{eq:start-variance}
    A_\step
    \int_{\rset^{\nagent d}} (\Paramw - \Paramlim)^{\otimes2} \, \meas{\step}(d\Paramw)
    & =
    \step^2 (\tenscommmat\otimes\tenscommmat) \int_{\rset^{\nagent d}} \PE[\globnoise{1}{\Paramw}^{\otimes2}] \, \meas{\step}(d\Paramw)
    + R_0(\step)
    \eqsp,
\end{align}
where:
\begin{itemize}
    \item
    $A_\step
    =
    \tensId \otimes \tensId - \tenscommmat \big( \tensId - \step \hF{\Paramlim} \big) \otimes \tenscommmat \big( \tensId - \step \hF{\Paramlim} \big)$

    \item
    $R_0(\step) = \mathcal{O}(\step^{5/2})$. More precisely :
    \begin{align*}
        R_0(\step)
        & =
        -\step
        \int_{\rset^{\nagent d}}
        \Big( \tenscommmat (\tensId - \step \hF{\Paramlim}) (\Paramw - \Paramlim) \otimes \tenscommmat R(\Paramw) \Big)
        \, \meas{\step}(d\Paramw) \\
        & \quad -
        \step
        \int_{\rset^{\nagent d}}
        \Big( \tenscommmat R(\Paramw) \otimes \tenscommmat (\tensId - \step \hF{\Paramlim}) (\Paramw - \Paramlim) \Big)
        \, \meas{\step}(d\Paramw) \\
        & \quad +
        \step^2
        (\tenscommmat\otimes\tenscommmat)
        \int_{\rset^{\nagent d}}
        \big( R(\Paramw) \otimes R(\Paramw) \big)
        \, \meas{\step}(d\Paramw)
        \eqsp.
    \end{align*}
    Hence, using \eqref{eq:upper-bound-R-mat} and
    $\forall a, b \ge 0, ab \le \dfrac{1}{4} a^4 +\dfrac{3}{4} b^{4/3}$:
    \begin{align*}
        \norm{R_0(\step)}
        & \le
        \int_{\rset^{\nagent d}}
        \big(
        \step^2 \norm{R(\Paramw)}^2
        + 2(1 - \step\strcvx) \step \norm{\Paramw - \Paramlim} \cdot \norm{R(\Paramw)}
        \big)
        \, \meas{\step}(d\Paramw) \\
        & \le
        \int_{\rset^{\nagent d}}
        \left( \frac{\boundThird^2}{4} \step^2 \norm{ \Paramw - \Paramlim }^4
        + \boundThird \step \norm{\Paramw - \Paramlim}^3
        \right)
        \, \meas{\step}(d\Paramw) \\
        & \le
        \int_{\rset^{\nagent d}}
        \left( \frac{\boundThird^2}{4} \step^2 \norm{ \Paramw - \Paramlim }^4
        + \big(\boundThird^{1/4} \step^{5/8}\big) \big( \boundThird^{3/4} \step^{3/8} \norm{\Paramw - \Paramlim}^3 \big)
        \right)
        \, \meas{\step}(d\Paramw) \\
        & \le
        \int_{\rset^{\nagent d}}
        \left( \frac{\boundThird^2}{4} \step^2 \norm{ \Paramw - \Paramlim }^4
        + \Big(\frac14 \boundThird \step^{5/2} + \frac34 \boundThird \step^{1/2} \norm{\Paramw - \Paramlim}^4 \Big)
        \right)
        \, \meas{\step}(d\Paramw) \\
        & \le
        \frac14 \left( \boundThird^2 \step^2 + 3 \boundThird \step^{1/2} \right)
        \int_{\rset^{\nagent d}}
        \norm{ \Paramw - \Paramlim }^4
        \, \meas{\step}(d\Paramw)
        + \frac14 \boundThird \step^{5/2}
        \eqsp.
    \end{align*}
    Using \Cref{lemma:control-R4}:
    \begin{align}\label{eq:upper-bound-R0-mat}
        \norm{R_0(\step)}
        & \le
        \step^{5/2}
        \left(
        \frac{
        65 \boundThird C^2
        \big( \tau_4^4 + \step^4 \frac{\lip^8}{\strcvx^4} \Lambda^4 \hgty^4 \big)(3 + \boundThird \step^{3/2})
        }{
        \strcvx^2
        }
        + \frac14 \boundThird
        \right)
        \eqsp.
    \end{align}
\end{itemize}
Then, we use Woodbury's formula
\begin{align*}
    (I - UV)^{-1}
    & =
    I
    + U(V^{-1} - U)^{-1}
    \eqsp,
\end{align*}
from which we deduce:
\begin{align*}
    A_\step^{-1}
    & =
    \Big(
    \tensId\otimes\tensId - (\tenscommmat \otimes \tenscommmat) \big( (\tensId - \step \hF{\Paramlim}) \otimes (\tensId - \step \hF{\Paramlim}) \big)
    \Big)^{-1} \\
    & =
    \tensId\otimes\tensId
    + (\tenscommmat \otimes \tenscommmat)
    \Big(
        \big( (\tensId - \step \hF{\Paramlim}) \otimes (\tensId - \step \hF{\Paramlim}) \big)^{-1}
        -
        \big( \tenscommmat \otimes \tenscommmat)
    \Big)^{-1} \\
    & =
    \tensId\otimes\tensId
    + (\tenscommmat \otimes \tenscommmat)
    \Big(
        \tensId\otimes\tensId - \tenscommmat\otimes\tenscommmat
        + \step \big( \tensId\otimes\hF{\Paramlim} + \hF{\Paramlim}\otimes\tensId \big)
        + R_1(\step)
    \Big)^{-1} \\
    & =
    \tensId\otimes\tensId
    + (\tenscommmat \otimes \tenscommmat)
    (M_\step + R_1(\step))^{-1}
    \eqsp,
\end{align*}
where:
\begin{itemize}
    \item
    $M_\step
    =
    \tensId\otimes\tensId - \tenscommmat\otimes\tenscommmat
    + \step \big( \tensId\otimes\hF{\Paramlim} + \hF{\Paramlim}\otimes\tensId \big)$,

    \item
    $R_1(\step)
    = \mathcal{O}(\step^2)$,
    and more precisely
    $\displaystyle
    R_1(\step)
    = \sum_{n = 2}^\infty \step^n \left( \sum_{k = 0}^n \hF{\Paramlim}^k \otimes \hF{\Paramlim}^{n - k} \right)$,
    so
    \begin{align*}
        \norm{R_1(\step)}
        & \le \sum_{n = 2}^\infty (n + 1) (\step \norm{\hF{\Paramlim}})^n \le \sum_{n = 2}^\infty (n + 1) (\step \lip)^n  = \frac{(\step\lip)^2(3 - 2 \step\lip)}{(1 - \step\lip)^2}
        \eqsp.
    \end{align*}
    If we take $\step < \dfrac{1}{10\lip}$, then
    \begin{align}\label{eq:upper-bound-R1-mat}
        \norm{R_1(\step)}
        & \le 6 (\step\lip)^2
        \eqsp.
    \end{align}
\end{itemize}
According to \Cref{lemma:AtB},
\begin{align}\label{eq:inverse-M}
    M_\step^{-1}
    & =
    \frac{1}{\step} \Big( (\projconsensus \otimes \projconsensus) \big( \tensId\otimes\hF{\Paramlim} + \hF{\Paramlim}\otimes\tensId \big) (\projconsensus \otimes \projconsensus) \Big)^\dagger
    + R_2(\step)
    \eqsp,
\end{align}
where:
\begin{itemize}
    \item
    $R_2(\step) = \mathcal{O}(1)$, and more precisely
    \begin{align}\label{eq:upper-bound-R2-mat}
        \norm{R_2(\step)}
        & \le
        \dfrac{1}{1 - \lambda_2(\commmat)}
        \Big( 1 + \dfrac{\lip}{\strcvx} \Big)^2
        \eqsp.
    \end{align}

    \item
    $\projconsensus =
    \dfrac{1}{\nagent} \underbrace{\oneVec\oneVec^\top}_{\in\rset^{\nagent \times \nagent}}
    \otimes
    \underbrace{\Id}_{\in\rset^{d\times d}}$
    is the matrix of the orthogonal projection onto the kernel of $\tensId - \tenscommmat$.
    
    \item
    $\projconsensus \otimes \projconsensus$ is the matrix of the orthogonal projection onto the kernel of $\tensId\otimes\tensId - \tenscommmat\otimes\tenscommmat$.
\end{itemize}
Therefore:
\begin{align*}
    A_\step^{-1}
    & =
    \tensId\otimes\tensId
    + (\tenscommmat \otimes \tenscommmat)
    (M_\step + R_1(\step))^{-1} \\
    & =
    (\tenscommmat \otimes \tenscommmat) M_\step^{-1} + \left( \tensId\otimes\tensId + R_3(\step) \right) \\
    & =
    \frac{1}{\step} (\tenscommmat\otimes\tenscommmat)\Big( (\projconsensus \otimes \projconsensus) \big( \tensId\otimes\hF{\Paramlim} + \hF{\Paramlim}\otimes\tensId \big) (\projconsensus \otimes \projconsensus) \Big)^\dagger \\
    & + \big( \tensId\otimes\tensId + R_3(\step) + (\tenscommmat\otimes\tenscommmat) R_2(\step) \big) \\
    & =
    \frac{1}{\step} \Big( (\projconsensus \otimes \projconsensus) \big( \tensId\otimes\hF{\Paramlim} + \hF{\Paramlim}\otimes\tensId \big) (\projconsensus \otimes \projconsensus) \Big)^\dagger \\
    & + \big( \tensId\otimes\tensId + R_3(\step) + (\tenscommmat\otimes\tenscommmat) R_2(\step) \big)
    \eqsp,
\end{align*}
where, according to \Cref{lemma:norm-difference-inv}, 
$\norm{R_3(\step)}
\le
\norm{\tenscommmat\otimes\tenscommmat}_2
\times
\norm{M_\step^{-1}}_2^2 \cdot \norm{R_1(\step)}
= \norm{M_\step^{-1}}_2^2 \cdot \norm{R_1(\step)}$.

Using \eqref{eq:inverse-M} and \eqref{eq:upper-bound-R2-mat},
$\norm{M_\step^{-1}}_2
\le
\dfrac{1}{2\step\strcvx} + \dfrac{1}{1 - \lambda_2(\commmat)}
\Big( 1 + \dfrac{\lip}{\strcvx} \Big)^2$,
so
$\norm{M_\step^{-1}}_2^2
\le
\dfrac{1}{2\step^2\strcvx^2} + \dfrac{2}{(1 - \lambda_2(\commmat))^2}
\Big( 1 + \dfrac{\lip}{\strcvx} \Big)^4$,
from which we deduce, with \eqref{eq:upper-bound-R1-mat}:
\begin{align}\label{eq:upper-bound-R3-mat}
    \norm{R_3(\step)}
    & \le
    \frac{3\lip^2}{\strcvx^2} + \frac{12 \step^2 \lip^2}{(1 - \lambda_2(\commmat))^2} \Big( 1 + \frac{\lip}{\strcvx} \Big)^4
    \eqsp.
\end{align}
We then have
\begin{align}\label{eq:expressionAgamma}
    A_\step^{-1}
    & =
    \frac{1}{\step} \Big( (\projconsensus \otimes \projconsensus) \big( \tensId\otimes\hF{\Paramlim} + \hF{\Paramlim}\otimes\tensId \big) (\projconsensus \otimes \projconsensus) \Big)^\dagger
    + R_4(\step)
    \eqsp,
\end{align}
with
\begin{align}\label{eq:upper-bound-R4-mat}
    \norm{ R_4(\step) }_2
    & \le
    1
    + \frac{3\lip^2}{\strcvx^2} + \frac{12 \step^2 \lip^2}{(1 - \lambda_2(\commmat))^2} \Big( 1 + \frac{\lip}{\strcvx} \Big)^4
    + \dfrac{1}{1 - \lambda_2(\commmat)}
    \Big( 1 + \dfrac{\lip}{\strcvx} \Big)^2
    \eqsp.
\end{align}

Plugging this back in \eqref{eq:start-variance}, we get
\begin{align*}
    \int_{\rset^{\nagent d}} (\Paramw - \Paramlim)^{\otimes 2} \, \pi_\step(d\Paramw)
    & =
    \step \Big( (\projconsensus \otimes \projconsensus) \big( \tensId\otimes\hF{\Paramlim} + \hF{\Paramlim}\otimes\tensId \big) (\projconsensus \otimes \projconsensus) \Big)^\dagger
    \int_{\rset^{\nagent d}} \PE[\globnoise{1}{\Paramw}^{\otimes 2}] \, \pi_\step(d\Paramw) \\
    & \quad
    + \step^2 R_4 (\tenscommmat\otimes\tenscommmat) \int_{\rset^{\nagent d}} \PE[\globnoise{1}{\Paramw}^{\otimes 2}] \, \pi_\step(d\Paramw)
    + A_\step^{-1} R_0
    \eqsp.
\end{align*}
Using \eqref{eq:upper-bound-R4-mat}, \Cref{assum:noise-filtration}, \eqref{eq:expressionAgamma} and \eqref{eq:upper-bound-R0-mat}:
\begin{align*}
    & \bnorm{
    \step^2 R_4 (\tenscommmat\otimes\tenscommmat) \int_{\rset^{\nagent d}} \PE[\globnoise{1}{\Paramw}^{\otimes 2}] \, \pi_\step(d\Paramw)
    + A_\step^{-1} R_0
    } \\
    & \le
    \step^2
    \tau_2^2
    \left(
    1
    + \frac{3\lip^2}{\strcvx^2} + \frac{12 \step^2 \lip^2}{(1 - \lambda_2(\commmat))^2} \Big( 1 + \frac{\lip}{\strcvx} \Big)^4
    + \dfrac{1}{1 - \lambda_2(\commmat)}
    \Big( 1 + \dfrac{\lip}{\strcvx} \Big)^2
    \right) \\
    & +
    \step^{5/2}
    \left(
    \frac{1}{2\step\strcvx}
    +
    1
    + \frac{3\lip^2}{\strcvx^2} + \frac{12 \step^2 \lip^2}{(1 - \lambda_2(\commmat))^2} \Big( 1 + \frac{\lip}{\strcvx} \Big)^4
    + \dfrac{1}{1 - \lambda_2(\commmat)}
    \Big( 1 + \dfrac{\lip}{\strcvx} \Big)^2
    \right) \\
    & \times
    \left(
    \frac{
    65 \boundThird C^2
    \big( \tau_4^4 + \step^4 \frac{\lip^8}{\strcvx^4} \Lambda^4 \hgty^4 \big)(3 + \boundThird \step^{3/2})
    }{
    \strcvx^2
    }
    + \frac14 \boundThird
    \right)
    \eqsp.
\end{align*}
\end{proof}

Using the previous lemma, we can now give a first expansion of the stationary distribution's bias.

\begin{lemma}\label{lemma:first-general-bias}
    Suppose
    \Cref{assum:functions},
    \Cref{assum:commmat},
    \Cref{assum:noise-filtration},
    \Cref{assum:noise-cocoercive}(4) and \Cref{assum:C} hold.
    Also assume that for each $k$, $\nfw{k}$\!'s fourth derivative is bounded: there exists $\boundFourth > 0$ such that for $\paramw, u \in\rset^d$,
    $\norm{\quaterfk{k}(\paramw) u^{\otimes 3}} \le \boundFourth \norm{u}^3$.
    Then the following expansion holds as $\step\to0$:
    \begin{align*}
        \Paramsto - \Paramlim
        & =
        - \frac{\step}{2} \big( \projconsensus \hF{\Paramlim} \projconsensus \big)^\dagger \terF{\Paramlim}
        \int_{\rset^{\nagent d}} (\Paramw - \Paramlim)^{\otimes2} \, \meas{\step}(d\Paramw)
        + \mathcal{O}(\step^{3/2})
        \eqsp.
    \end{align*}
    Moreover, if $\step \le \min\left( \dfrac{1}{\Lambda\lip}, \dfrac{1}{10\lip} \right)$, then
    \begin{align*}
        & \bnorm{
        \Paramsto - \Paramlim
        + \frac{\step}{2} \big( \projconsensus \hF{\Paramlim} \projconsensus \big)^\dagger \terF{\Paramlim}
        \int_{\rset^{\nagent d}} (\Paramw - \Paramlim)^{\otimes2} \, \meas{\step}(d\Paramw)
        } \\
        & \lesssim
        \step^{3/2}
        \frac{
        20 C^{3/2}\tau_4^3 \boundFourth
        }{
        \strcvx^{5/2}
        }
        +
        \step^2
        \left(
        1 +
        \dfrac{4 \lip^2}{\strcvx^2}
        \right)
        \frac{2\boundThird\tau_2^2}{\strcvx}
        \eqsp,
    \end{align*}
    where we only kept the lowest orders in $\step$.
\end{lemma}
\begin{proof}[Proof of \Cref{lemma:first-general-bias}]
Using the update equation with $t = 0$, integrating over $\globParam{0} \sim \meas{\step}$ and taking the expected value on $\funnoise{1}$, it holds that
\begin{align*}
    (\tensId - \tenscommmat) \Paramsto
    & =
    -\step \tenscommmat \int_{\rset^{\nagent d}} \gF{\Paramw} \, \meas{\step}(d\Paramw)
    \eqsp.
\end{align*}
Then, with a Taylor expansion of $\gF{\Paramw}$ around $\Paramlim$:
\begin{align*}
    \gF{\Paramw}
    & =
    \gF{\Paramlim}
    + \hF{\Paramlim} (\Paramw - \Paramlim)
    + \frac12 \terF{\Paramlim} (\Paramw - \Paramlim)^{\otimes 2}
    + R_1(\Paramw)
    \eqsp,
\end{align*}
where $\forall\Paramw\in\rset^{\nagent d}, \norm{R_1(\Paramw)} \le \frac16 \boundFourth \norm{ \Paramw - \Paramlim }^3$. Thus:
\begin{align*}
    & \quad (\tensId - \tenscommmat) \Paramsto \\
    & =
    - \step \tenscommmat \gF{\Paramlim}
    - \step \tenscommmat \hF{\Paramlim} (\Paramsto - \Paramlim)
    - \frac{\step}{2} \tenscommmat \terF{\Paramlim} \int_{\rset^{\nagent d}} (\Paramw - \Paramlim)^{\otimes2} \, \meas{\step}(d\Paramw) \\
    & \ - \step \tenscommmat \int_{\rset^{\nagent d}} R_1(\Paramw) \, \meas{\step}(d\Paramw)
    \eqsp.
\end{align*}
Since
$(\tensId - \tenscommmat) \Paramlim
= -\step \tenscommmat \gF{\Paramlim}$,
we obtain
\begin{align*}
    & \quad (\tensId - \tenscommmat + \step \tenscommmat \hF{\Paramlim}) (\Paramsto - \Paramlim) \\
    & =
    - \frac{\step}{2} \tenscommmat \terF{\Paramlim} \int_{\rset^{\nagent d}} (\Paramw - \Paramlim)^{\otimes2} \, \meas{\step}(d\Paramw)
    - \step \tenscommmat \int_{\rset^{\nagent d}} R_1(\Paramw) \, \meas{\step}(d\Paramw)
    \eqsp.
\end{align*}
Using Woodbury's formula:
\begin{align*}
    (I - UV)^{-1}
    & =
    I
    + U(V^{-1} - U)^{-1}
    \eqsp,
\end{align*}
and \Cref{lemma:AtB}:
\begin{align*}
    (\tensId - \tenscommmat + \step \tenscommmat \hF{\Paramlim})^{-1}
    & =
    (\tensId - \tenscommmat(\tensId - \step \hF{\Paramlim}))^{-1} \\
    & =
    \tensId
    + \tenscommmat \big( (\tensId - \step \hF{\Paramlim})^{-1} - \tenscommmat \big)^{-1} \\
    & =
    \tensId
    + \tenscommmat \big( \tensId - \tenscommmat + \step \hF{\Paramlim} + R_2(\step) \big)^{-1} \\
    & =
    \frac{1}{\step} \tenscommmat \big( \projconsensus \hF{\Paramlim} \projconsensus \big)^\dagger
    + R_3(\step) \\
    & =
    \frac{1}{\step} \big( \projconsensus \hF{\Paramlim} \projconsensus \big)^\dagger
    + R_3(\step)
    \eqsp,
\end{align*}
where
\begin{itemize}
    \item
    $\displaystyle
    R_2(\step)
    = \sum_{k = 2}^{+\infty} (\step \hF{\Paramlim})^k$
    satisfies
    $\norm{R_2(\step)}_2
    \le \dfrac{\step^2\lip^2}{1 - \step\lip}
    \le 2 \step^2\lip^2$
    if $\step < \dfrac{1}{10\lip}$.

    \item
    According to \Cref{lemma:norm-difference-inv},
    $\norm{R_3(\step)}_2
    \le 1 + \norm{(\tensId - \tenscommmat + \step \hF{\Paramlim})^{-1}}_2^2 \cdot \norm{R_2(\step)}_2$.

    \item
    According to \Cref{lemma:AtB},
    $\norm{(\tensId - \tenscommmat + \step \hF{\Paramlim})^{-1}}_2
    \le
    \dfrac{1}{\step\strcvx}
    +
    \dfrac{1}{1 - \lambda_2(\commmat)}
    \Big( 1 + \dfrac{\lip}{\strcvx} \Big)^2$.
\end{itemize}
Combining these inequalities yields
\begin{align*}
    \norm{R_3(\step)}_2
    & \le
    1 +
    \dfrac{4 \lip^2}{\strcvx^2}
    +
    \step^2
    \dfrac{4 \lip^2}{(1 - \lambda_2(\commmat))^2}
    \Big( 1 + \dfrac{\lip}{\strcvx} \Big)^4
    \eqsp.
\end{align*}
so
\begin{align*}
    \Paramsto - \Paramlim
    & =
    - \frac{\step}{2} (\tensId - \tenscommmat + \step \tenscommmat \hF{\Paramlim})^{-1} \tenscommmat \terF{\Paramlim}
    \int_{\rset^{\nagent d}} (\Paramw - \Paramlim)^{\otimes2} \, \meas{\step}(d\Paramw) \\
    & - \step (\tensId - \tenscommmat + \step \tenscommmat \hF{\Paramlim})^{-1} \tenscommmat \int_{\rset^{\nagent d}} R_1(\Paramw) \, \meas{\step}(d\Paramw) \\
    & =
    - \frac{1}{2} \big( \projconsensus \hF{\Paramlim} \projconsensus \big)^\dagger \terF{\Paramlim}
    \int_{\rset^{\nagent d}} (\Paramw - \Paramlim)^{\otimes2} \, \meas{\step}(d\Paramw)
    + R_4(\step)
    \eqsp,
\end{align*}
where
\begin{align*}
    R_4(\step)
    & =
    - \frac{\step}{2} R_3(\step) \tenscommmat \terF{\Paramlim}
    \int_{\rset^{\nagent d}} (\Paramw - \Paramlim)^{\otimes2} \, \meas{\step}(d\Paramw) \\
    & \quad - \step (\tensId - \tenscommmat + \step \tenscommmat \hF{\Paramlim})^{-1} \tenscommmat \int_{\rset^{\nagent d}} R_1(\Paramw) \, \meas{\step}(d\Paramw)
    \eqsp.
\end{align*}
Using the upper-bounds we established previously, as well as \Cref{assum:functions}, \Cref{lemma:control-R2} and \Cref{remark:control-moment-3}:
\begin{align*}
    \norm{R_4(\step)}_2
    & \le
    \step^2
    \boundThird
    \left(
    1 +
    \dfrac{4 \lip^2}{\strcvx^2}
    +
    \step^2
    \dfrac{4 \lip^2}{(1 - \lambda_2(\commmat))^2}
    \Big( 1 + \dfrac{\lip}{\strcvx} \Big)^4
    \right)
    \frac{\Big( \tau_2 + \step \frac{2\lip^2}{\strcvx} \Lambda \hgty \Big)^2}{\strcvx} \\
    & +
    \step^{5/2}
    \frac{\boundFourth}{6}
    \left( \frac{1}{\step\strcvx}
    + 1 +
    \dfrac{4 \lip^2}{\strcvx^2}
    +
    \step^2
    \dfrac{4 \lip^2}{(1 - \lambda_2(\commmat))^2}
    \Big( 1 + \dfrac{\lip}{\strcvx} \Big)^4 \right)
    \frac{
    2^6 C^{3/2}
    \big( \tau_4^3 + \step^3 \frac{\lip^6}{\strcvx^3} \Lambda^3 \hgty^3 \big)
    }{
    \strcvx^{3/2}
    }
    \eqsp.
\end{align*}
\end{proof}


The two previous lemmas use the covariance of the noise under the stationary distribution, for which we give an expansion in terms of the optimum $\Paramstar$ in the next two lemmas.

\begin{lemma}\label{lemma:correlation}
    Under the assumptions of \Cref{prop:DSGD-variance-general}, and if all local functions have a bounded fourth derivative as in \Cref{lemma:first-general-bias}, then it holds that
    \begin{align*}
        \int_{\rset^{\nagent d}} \CovarianceTens(\Paramw) \, \meas{\step}(d\Paramw)
        & =
        \CovarianceTens(\Paramlim)
        + \mathcal{O}(\step)
        \eqsp.
    \end{align*}
    Moreover, if we assume that there exist $\CovBoundA, \CovBoundB \ge 0$ such that for any $\Paramw$,
    \begin{equation*}
        \norm{\nabla \CovarianceTens(\Paramw)}
        \le \CovBoundA + \CovBoundB \norm{\Paramw - \Paramstar}^2,
        \qquad
        \norm{\nabla^2 \CovarianceTens(\Paramw)}
        \le \CovBoundA + \CovBoundB \norm{\Paramw - \Paramstar}^2
        \eqsp,
    \end{equation*}
    and if
    $\step \le \min\left( \dfrac{1}{\Lambda\lip}, \dfrac{1}{10\lip} \right)$, then
    \begin{align*}
        \bnorm{
        \int_{\rset^{\nagent d}} \CovarianceTens(\Paramw) \, \meas{\step}(d\Paramw)
        -
        \CovarianceTens(\Paramlim)
        }
        & \lesssim
        \step
        \frac{4 \CovBoundA \tau_2}{\strcvx}
        \eqsp.
    \end{align*}
\end{lemma}
\begin{proof}
    We use a Taylor expansion
    \begin{align*}
        \CovarianceTens(\Paramw)
        & =
        \CovarianceTens(\Paramlim)
        + \nabla\CovarianceTens(\Paramlim)(\Paramw - \Paramlim)
        + R(\Paramw)
        \eqsp,
    \end{align*}
    where
    \begin{align*}
        \norm{R(\Paramw)}
        & \le
        \left(\CovBoundA + \CovBoundB \left( \norm{\Paramlim - \Paramstar}^2 +  \norm{\Paramw - \Paramstar}^2 \right) \right) \norm{\Paramw - \Paramlim}^2 \\
        & =
        \left(\CovBoundA + \CovBoundB \norm{\Paramlim - \Paramstar}^2 \right) \norm{\Paramw - \Paramlim}^2
        +
        \CovBoundB \norm{\Paramw - \Paramstar}^2 \norm{\Paramw - \Paramlim}^2 \\
        & \le
        \left(\CovBoundA + \CovBoundB \norm{\Paramlim - \Paramstar}^2 \right) \norm{\Paramw - \Paramlim}^2
        +
        2 \CovBoundB \norm{\Paramw - \Paramlim}^4
        +
        2 \CovBoundB \norm{\Paramlim - \Paramstar}^2 \norm{\Paramw - \Paramlim}^2 \\
        & \le
        \left(\CovBoundA + 3\CovBoundB \norm{\Paramlim - \Paramstar}^2 \right) \norm{\Paramw - \Paramlim}^2
        +
        2 \CovBoundB \norm{\Paramw - \Paramlim}^4
        \eqsp.
    \end{align*}
    We then have
    \begin{align*}
        \int_{\rset^{\nagent d}} \CovarianceTens(\Paramw) \, \meas{\step}(d\Paramw)
        & =
        \CovarianceTens(\Paramlim)
        +
        \nabla\CovarianceTens(\Paramlim)(\Paramsto - \Paramlim)
        +
        \int_{\rset^{\nagent d}} R(\Paramw) \, \meas{\step}(d\Paramw)
        \eqsp.
    \end{align*}
    Using \Cref{assum:C}, \Cref{lemma:first-general-bias}, \Cref{lemma:control-R2} and \Cref{lemma:control-R4}:
    \begin{align*}
        \int_{\rset^{\nagent d}} \CovarianceTens(\Paramw) \, \meas{\step}(d\Paramw)
        & =
        \CovarianceTens(\Paramlim)
        + \mathcal{O}(\step)
        \eqsp.
    \end{align*}
    More precisely,
    \begin{align*}
        \bnorm{
        \nabla\CovarianceTens(\Paramlim)(\Paramsto - \Paramlim)
        +
        \int_{\rset^{\nagent d}} R(\Paramw) \, \meas{\step}(d\Paramw)
        }
        & \le
        \norm{\nabla\CovarianceTens(\Paramlim)} \cdot \norm{\Paramsto - \Paramlim}
        +
        \int_{\rset^{\nagent d}} \norm{R(\Paramw)} \, \meas{\step}(d\Paramw) \\
        & \le
        \big( \CovBoundA + \CovBoundB \norm{\Paramlim - \Paramstar}^2 \big) \cdot \norm{\Paramsto - \Paramlim} \\
        & \quad
        +
        \left(\CovBoundA + 3\CovBoundB \norm{\Paramlim - \Paramstar}^2 \right)
        \int_{\rset^{\nagent d}} \norm{\Paramw - \Paramlim}^2 \, \meas{\step}(d\Paramw) \\
        & \quad
        +
        2 \CovBoundB
        \int_{\rset^{\nagent d}} \norm{\Paramw - \Paramlim}^4 \, \meas{\step}(d\Paramw)
        \eqsp.
    \end{align*}
    Using \Cref{lemma:first-general-bias}, \Cref{lemma:control-R2} and \Cref{lemma:control-R4}, if $\step \le \min\left( \dfrac{1}{\Lambda\lip}, \dfrac{1}{10\lip} \right)$ :
    \begin{align*}
        & \bnorm{
        \nabla\CovarianceTens(\Paramlim)(\Paramsto - \Paramlim)
        +
        \int_{\rset^{\nagent d}} R(\Paramw) \, \meas{\step}(d\Paramw)
        } \\
        & \le
        \big( \CovBoundA + \CovBoundB \norm{\Paramlim - \Paramstar}^2 \big) \cdot \norm{\Paramsto - \Paramlim}
        +
        \left(\CovBoundA + 3\CovBoundB \norm{\Paramlim - \Paramstar}^2 \right)
        \int_{\rset^{\nagent d}} \norm{\Paramw - \Paramlim}^2 \, \meas{\step}(d\Paramw) \\
        & \quad
        +
        2 \CovBoundB
        \int_{\rset^{\nagent d}} \norm{\Paramw - \Paramlim}^4 \, \meas{\step}(d\Paramw) \\
        & \le
        \big( \CovBoundA + 3 \CovBoundB \norm{\Paramlim - \Paramstar}^2 \big)
        \cdot
        \left(
        \norm{\Paramsto - \Paramlim}
        +
        \int_{\rset^{\nagent d}} \norm{\Paramw - \Paramlim}^2 \, \meas{\step}(d\Paramw)
        \right)
        +
        2 \CovBoundB
        \int_{\rset^{\nagent d}} \norm{\Paramw - \Paramlim}^4 \, \meas{\step}(d\Paramw) \\
        & \le
        \step
        \left( \CovBoundA + 3 \CovBoundB \step^2 \frac{\lip^2}{\strcvx^2} \Lambda^2 \hgty^2 \right)
        \cdot
        \left(
        \left( 1 + \frac{\step\boundThird}{2\strcvx} \right)
        \frac{4 \left( \tau_2^2 + \step^2 \frac{4\lip^4}{\strcvx^2} \Lambda^2 \hgty^2 \right)}{\strcvx}
        \right. \\
        & + \qquad \left.
        \step^2
        \boundThird
        \left(
        1 +
        \dfrac{4 \lip^2}{\strcvx^2}
        +
        \step^2
        \dfrac{4 \lip^2}{(1 - \lambda_2(\commmat))^2}
        \Big( 1 + \dfrac{\lip}{\strcvx} \Big)^4
        \right)
        \frac{\Big( \tau_2 + \step \frac{2\lip^2}{\strcvx} \Lambda \hgty \Big)^2}{\strcvx}
        \right. \\
        & + \qquad \left.
        \step^{5/2}
        \frac{\boundFourth}{6}
        \left( \frac{1}{\step\strcvx}
        + 1 +
        \dfrac{4 \lip^2}{\strcvx^2}
        +
        \step^2
        \dfrac{4 \lip^2}{(1 - \lambda_2(\commmat))^2}
        \Big( 1 + \dfrac{\lip}{\strcvx} \Big)^4 \right)
        \frac{
        2^6 C^{3/2}
        \big( \tau_4^3 + \step^3 \frac{\lip^6}{\strcvx^3} \Lambda^3 \hgty^3 \big)
        }{
        \strcvx^{3/2}
        }
        \right) \\
        & + \qquad
        \step^2
        \frac{
        500 C^2 \CovBoundB
        \big( \tau_4^4 + \step^4 \frac{\lip^8}{\strcvx^4} \Lambda^4 \hgty^4 \big)
        }{
        \strcvx^2
        }
        \eqsp.
    \end{align*}
\end{proof}

\begin{lemma}\label{lemma:correlation-2}
    Under the assumptions of \Cref{lemma:correlation}, it holds that
    \begin{align*}
        \int_{\rset^{\nagent d}} \CovarianceTens(\Paramw) \, \meas{\step}(d\Paramw)
        & =
        \CovarianceTens(\Paramstar)
        + \mathcal{O}(\step)
        \eqsp.
    \end{align*}
    Moreover, if we assume, like in \Cref{lemma:correlation}, that the first two derivatives of $\CovarianceTens$ have polynomial growth, then for $\step \le \max\left( \dfrac{1}{\Lambda\lip}, \dfrac{1}{10\lip} \right)$,
    \begin{align*}
        & \bnorm{
        \int_{\rset^{\nagent d}} \CovarianceTens(\Paramw) \, \meas{\step}(d\Paramw)
        -
        \CovarianceTens(\Paramstar)
        } \\
        & \le
        \step \frac{\lip}{\strcvx} \Lambda \hgty
        \cdot
        \left(
        \CovBoundA + \step^2 \frac{\CovBoundB\lip^2}{2\strcvx^2} \Lambda^2 \hgty^2
        \right) \\
        & + \quad
        \step
        \left( \CovBoundA + 3 \CovBoundB \step^2 \frac{\lip^2}{\strcvx^2} \Lambda^2 \hgty^2 \right)
        \cdot
        \left(
        \left( 1 + \frac{\step\boundThird}{2\strcvx} \right)
        \frac{4 \left( \tau_2^2 + \step^2 \frac{4\lip^4}{\strcvx^2} \Lambda^2 \hgty^2 \right)}{\strcvx}
        \right. \\
        & + \qquad \left.
        \step^2
        \boundThird
        \left(
        1 +
        \dfrac{4 \lip^2}{\strcvx^2}
        +
        \step^2
        \dfrac{4 \lip^2}{(1 - \lambda_2(\commmat))^2}
        \Big( 1 + \dfrac{\lip}{\strcvx} \Big)^4
        \right)
        \frac{\Big( \tau_2 + \step \frac{2\lip^2}{\strcvx} \Lambda \hgty \Big)^2}{\strcvx}
        \right. \\
        & + \qquad \left.
        \step^{5/2}
        \left( \frac{1}{\step\strcvx}
        + 1 +
        \dfrac{4 \lip^2}{\strcvx^2}
        +
        \step^2
        \dfrac{4 \lip^2}{(1 - \lambda_2(\commmat))^2}
        \Big( 1 + \dfrac{\lip}{\strcvx} \Big)^4 \right)
        \frac{
        2^6 \cdot 100^{3/2} \boundFourth
        \big( \tau_4^3 + \step^3 \frac{\lip^6}{\strcvx^3} \Lambda^3 \hgty^3 \big)
        }{
        6\strcvx^{3/2}
        }
        \right) \\
        & + \quad
        \step^2
        \frac{
        500 \cdot 100^2 \CovBoundB
        \big( \tau_4^4 + \step^4 \frac{\lip^8}{\strcvx^4} \Lambda^4 \hgty^4 \big)
        }{
        \strcvx^2
        }
        \eqsp.
    \end{align*}
\end{lemma}
\begin{proof}
    Write
    \begin{align*}
        \bnorm{
        \int_{\rset^{\nagent d}} \CovarianceTens(\Paramw) \, \meas{\step}(d\Paramw)
        -
        \CovarianceTens(\Paramstar)
        }
        & \le
        \bnorm{
        \int_{\rset^{\nagent d}} \CovarianceTens(\Paramw) \, \meas{\step}(d\Paramw)
        -
        \CovarianceTens(\Paramlim)
        }
        +
        \bnorm{
        \CovarianceTens(\Paramlim)
        -
        \CovarianceTens(\Paramstar)
        }
        \eqsp,
    \end{align*}
    upper-bound the first term using \Cref{lemma:correlation}, and the second one using \Cref{lemma:upper-bound-error}:
    \begin{align*}
        \bnorm{
        \CovarianceTens(\Paramlim)
        -
        \CovarianceTens(\Paramstar)
        }
        & \le
        \norm{\Paramlim - \Paramstar}
        \cdot
        \left(
        \CovBoundA + \frac{\CovBoundB}{2} \norm{\Paramlim - \Paramstar}^2
        \right) \\
        & \le
        \step \frac{\lip}{\strcvx} \Lambda \hgty
        \cdot
        \left(
        \CovBoundA + \step^2 \frac{\CovBoundB\lip^2}{2\strcvx^2} \Lambda^2 \hgty^2
        \right)
        \eqsp.
    \end{align*}
\end{proof}

The following corollary gives an expansion of the stationary distribution's variance at the first order in $\step$, without any dependence in the network topology at the first order.

\begin{corollary}\label{corollary:variance-expansion}
    Suppose
    \Cref{assum:functions},
    \Cref{assum:commmat},
    \Cref{assum:noise-filtration},
    \Cref{assum:noise-cocoercive}(4) and \Cref{assum:C} hold.
    Then the following expansion holds as $\step\to0$:
    \begin{align*}
        \int_{\rset^{\nagent d}} (\Paramw - \Paramlim)^{\otimes 2} \, \pi_\step(d\Paramw)
        & =
        \step \Big( (\projconsensus \otimes \projconsensus) \big( \tensId\otimes\hF{\Paramstar} + \hF{\Paramstar}\otimes\tensId \big) (\projconsensus \otimes \projconsensus) \Big)^\dagger
        \CovarianceTens(\Paramstar)
        + \mathcal{O}(\step^{3/2})
        \eqsp.
    \end{align*}
\end{corollary}
\begin{proof}
    This is easily obtained from \Cref{lemma:first-general-variance} and \Cref{lemma:correlation-2} using \Cref{assum:functions}.
    
    Note that $A \mapsto A^\dagger$ is not continuous, but that $A \mapsto ((\projconsensus \otimes \projconsensus) \, A \, (\projconsensus \otimes \projconsensus))^\dagger$ is continuous at every symmetric positive definite $A$.
\end{proof}


The previous result is rewritten in a more readable in the following lemma.

\begin{lemma}\label{lemma:simplify-variance}
Under the same assumptions as \Cref{lemma:first-general-variance}, the first-order term of the quadratic error can be expressed, for any $k, \ell \in \{ 1, \dots, \nagent \}$, as
\begin{equation*}
    \int_{\rset^{\nagent d}} (\locparamw{k} - \paramlimloc{k}) \otimes (\locparamw{\ell} - \paramlimloc{\ell}) \, \pi_\step(d\Paramw)
    =
    \frac{\step}{\nagent}
    \big( \tensId \otimes \barA + \barA \otimes \tensId \big)^{-1}
    \Covariance(\paramstar)
    + \mathcal{O}(\step^{3/2})
    \eqsp.
\end{equation*}
\end{lemma}
\begin{proof}
We simply need to simplify $\Big( \projconsensus \otimes (\projconsensus \hF{\Paramstar} \projconsensus) + (\projconsensus \hF{\Paramstar} \projconsensus) \otimes \projconsensus \Big)^\dagger$.

First note that
$\projconsensus \hF{\Paramstar} \projconsensus
= \frac{1}{\nagent} (\oneVec\oneVec^\top) \otimes \barA$, with $\barA$ the average Hessian at $\paramstar$.

Let $\Pi$ be the operator such that $\Pi (x_1 \otimes x_2 \otimes x_3 \otimes x_4) = x_1 \otimes x_3 \otimes x_2 \otimes x_4$.

Then:
\begin{align*}
    & \Pi \Big( \projconsensus \otimes (\projconsensus \hF{\Paramstar} \projconsensus) + (\projconsensus \hF{\Paramstar} \projconsensus) \otimes \projconsensus \Big) \Pi^\top \\
    & = \Pi \Big( \frac{1}{\nagent} (\oneVec\oneVec^\top) \otimes \tensId \otimes \frac{1}{\nagent} (\oneVec\oneVec^\top) \otimes \barA + \frac{1}{\nagent} (\oneVec\oneVec^\top) \otimes \barA \otimes \frac{1}{\nagent} (\oneVec\oneVec^\top) \otimes \tensId \Big) \Pi^\top \\
    & = \frac{1}{\nagent} (\oneVec\oneVec^\top) \otimes \frac{1}{\nagent} (\oneVec\oneVec^\top) \otimes \big( \tensId \otimes \barA + \barA \otimes \tensId \big) \\
    & = \frac{1}{\nagent} (\oneVec \otimes \oneVec) \times \frac{1}{\nagent} (\oneVec \otimes \oneVec)^\top \otimes \big( \tensId \otimes \barA + \barA \otimes \tensId \big)
    \eqsp.
\end{align*}
This operator is zero on $\Span(\oneVec \otimes \oneVec)^\perp \otimes \rset^{\nagent d \times \nagent d}$, and acts as $\tensId \otimes \big( \tensId \otimes \barA + \barA \otimes \tensId \big)$ on $\Span(\oneVec \otimes \oneVec) \otimes \rset^{\nagent d \times \nagent d}$.

Its pseudo-inverse is thus
$\frac{1}{\nagent} (\oneVec \otimes \oneVec) \times \frac{1}{\nagent} (\oneVec \otimes \oneVec)^\top \otimes \big( \tensId \otimes \barA + \barA \otimes \tensId \big)^{-1}$,
and we deduce that
\begin{align*}
    \Big( \projconsensus \otimes (\projconsensus \hF{\Paramstar} \projconsensus) + (\projconsensus \hF{\Paramstar} \projconsensus) \otimes \projconsensus \Big)^\dagger
    & = \Pi^\top \Big( \frac{1}{\nagent} (\oneVec \oneVec^\top) \otimes \frac{1}{\nagent} (\oneVec \oneVec^\top) \otimes \big( \tensId \otimes \barA + \barA \otimes \tensId \big)^{-1} \Big) \Pi
    \eqsp.
\end{align*}
Therefore the result from \Cref{lemma:first-general-variance} can be rewritten as
\begin{equation}\label{eq:KroToVec}
    \Pi \int_{\rset^{\nagent d}} (\Paramw - \Paramlim)^{\otimes 2} \, \pi_\step(d\Paramw)
    =
    \step \Big( \frac{1}{\nagent} (\oneVec \oneVec^\top) \otimes \frac{1}{\nagent} (\oneVec \oneVec^\top) \otimes \big( \tensId \otimes \barA + \barA \otimes \tensId \big)^{-1} \Big)
    \times
    \Pi \CovarianceTens(\Paramstar)
    + \mathcal{O}(\step^{3/2})
    \eqsp.
\end{equation}
For
$\displaystyle
x =
\begin{pmatrix}
    x_1 \\
    \vdots \\
    x_m
\end{pmatrix}
=
\sum\limits_{k = 1}^\nagent e_k \otimes x_k$,
we have
$\displaystyle
\Pi (x \otimes x)
=
\sum\limits_{k, \ell = 1}^\nagent e_k \otimes e_\ell \otimes x_k \otimes x_\ell$,
so \eqref{eq:KroToVec} becomes
\begin{align*}
    & \sum\limits_{k, \ell = 1}^\nagent e_k \otimes e_\ell \otimes \left( \int_{\rset^{\nagent d}} (\locparamw{k} - \paramlimloc{k}) \otimes (\locparamw{\ell} - \paramlimloc{\ell})
    \, \pi_\step(d\Paramw) \right) \\
    & =
    \step \Big( \frac{1}{\nagent} (\oneVec \oneVec^\top) \otimes \frac{1}{\nagent} (\oneVec \oneVec^\top) \otimes \big( \tensId \otimes \barA + \barA \otimes \tensId \big)^{-1} \Big)
    \times
    \left( \sum\limits_{i = 1}^\nagent e_i \otimes e_i \otimes \PE\left[ \big(\paramnoise^{(i)}(\paramw)\big)^{\otimes 2} \right] \right)
    + \mathcal{O}(\step^{3/2}) \\
    & =
    \frac{\step}{\nagent} \oneVec \otimes \oneVec \otimes \left( \big( \tensId \otimes \barA + \barA \otimes \tensId \big)^{-1} \Covariance(\paramstar) \right)
    + \mathcal{O}(\step^{3/2}) \\
    & =
    \frac{\step}{\nagent} \sum_{k, \ell = 1}^\nagent e_k \otimes e_l \otimes \left( \big( \tensId \otimes \barA + \barA \otimes \tensId \big)^{-1} \Covariance(\paramstar) \right)
    + \mathcal{O}(\step^{3/2})
    \eqsp,
\end{align*}
from which we deduce, for any $k, \ell \in \{ 1, \dots, \nagent \}$,
\begin{equation*}
    \int_{\rset^{\nagent d}} (\locparamw{k} - \paramlimloc{k}) \otimes (\locparamw{\ell} - \paramlimloc{\ell}) \, \pi_\step(d\Paramw)
    =
    \frac{\step}{\nagent}
    \big( \tensId \otimes \barA + \barA \otimes \tensId \big)^{-1}
    \Covariance(\paramstar)
    + \mathcal{O}(\step^{3/2})
    \eqsp.
\end{equation*}
\end{proof}

\begin{proof}[Proof of \Cref{prop:DSGD-variance-general}]
    \Cref{prop:DSGD-variance-general} is obtained by combining \Cref{lemma:first-general-variance} and \Cref{lemma:simplify-variance}.
\end{proof}

\begin{proof}[Proof of \Cref{prop:DSGD-bias-general}]
    Similarly, \Cref{lemma:first-general-bias} and \Cref{lemma:simplify-variance} yield \Cref{prop:DSGD-bias-general}.
\end{proof}

\bigskip


\section{NON-ASYMPTOTIC ANALYSIS}
\label{sec:app-non-asymptotic-analysis}

We now give more explicit upper-bounds on the variance of the stationary distribution, and use them to establish non-asymptotic upper-bounds for the iterates of DSGD.

\subsection{Upper bound on Limiting Variance}
\label{sec:non-asymptotic-bounds-sup}

We first study the variance of the limiting distribution.
To this end, we decompose the error in the limit distribution in a consensus and a disagreement part, and study them separately.

Define
\begin{align*}
    B
    & =
    \frac{
        2000
        \big( \tau_4^2 + \step^2 \frac{\lip^4}{\strcvx^2} \Lambda^2 \hgty^2 \big)
    }{
        \strcvx
    }
    \eqsp.
\end{align*}
According to \Cref{lemma:control-R4} and Hölder's inequality,
\begin{align}
    \int \norm{ \Paramw - \Paramlim }^2 \pi_\step(\mathrm{d} \Paramw)
    \le 
    B \step
    \eqsp,
    \qquad
    \qquad
    \Big( \int \norm{ \Paramw - \Paramlim }^4 \pi_\step(\mathrm{d} \Paramw) \Big)^{1/2}
    \le 
    B \step
    \eqsp.
\end{align}

Starting from \eqref{eq:decomp-update-1}, we have
\begin{align}
    \globParam{1} - \Paramlim
    \label{eq:decomp-update-app-1}
    & =
    \tenscommmat (\tensId - \step \hF{\Paramlim}) (\globParam{0} - \Paramlim)
    - \step \tenscommmat \big( \globnoise{1}{\globParam{0}} + R(\globParam{0}) \big)
    \eqsp,
\end{align}
where, by \eqref{eq:upper-bound-R-mat}, the remainder satisfies $\norm{R(\Paramw)} \le \dfrac{\boundThird}{2} \norm{ \Paramw - \Paramlim }^2$.

\textbf{Bound on Consensus.}
Defining $\bglobParam{0}$, $\bglobParam{1}, \bparamlim$ the average of all agents parameters for $\globParam{0}$, $\globParam{1}$, and $\Paramlim$ respectively, and projecting \eqref{eq:decomp-update-app-1} on the consensus space, we obtain
\begin{align}
\label{eq:app-exp-consensus}
\bglobParam{1} - \bparamlim
=
(\tensId - \step \bhF{\Paramlim}) (\bglobParam{0} - \bparamlim)
    - \step \big( \bglobnoise{1}{\globParam{0}} + \bar{R}(\globParam{0}) \big)
    \eqsp,
\end{align}
where we also defined
$\bhF{\Paramlim} \in \rset^{d\times d},
\bglobnoise{1}{\globParam{0}} \in \rset^d,
\bar{R}(\globParam{0}) \in \rset^d$ as the averaged Hessian, noise and remainder.
Taking the tensor square of \eqref{eq:app-exp-consensus} and expectation over the noise at step $1$, we have
\begin{align}
\nonumber
(\bglobParam{1} - \bparamlim)^{\otimes 2}
& =
(\tensId - \step \bhF{\Paramlim}) 
(\bglobParam{0} - \bparamlim)^{\otimes 2}
(\tensId - \step \bhF{\Paramlim}) 
- \step (\tensId - \step \bhF{\Paramlim}) 
(\bglobParam{0} - \bparamlim) \bar{R}(\globParam{0})^\top
\\
& \quad 
- \step \bar{R}(\globParam{0}) (\bglobParam{0} - \bparamlim)^\top (\tensId - \step \bhF{\Paramlim}) 
+ \step^2 \PE[ \bglobnoise{1}{\globParam{0}}^{\otimes 2} ]
    \eqsp.
\end{align}
Denote by $\covconsensus$ (respectively $\covdisagreement$) the covariance matrix of the consensus part (respectively disagreement part) under the stationary distribution, and $\covavgnoise$ the covariance matrix of the average noise under the stationary distribution.

Assuming $\globParam{0} \sim \meas{\step}$ started from the stationary distribution, we can integrate this equality over this distribution to obtain
\begin{align*}
\nonumber
\covconsensus
& =
(\tensId - \step \bhF{\Paramlim}) 
\covconsensus
(\tensId - \step \bhF{\Paramlim}) 
+ \step^2 \covavgnoise
\\
& \quad 
- \step \int \Big\{ (\tensId - \step \bhF{\Paramlim}) 
(\bglobParam{0} - \bparamlim) \bar{R}(\globParam{0})^\top
+ \bar{R}(\globParam{0}) (\bglobParam{0} - \bparamlim)^\top (\tensId - \step \bhF{\Paramlim}) 
\Big\} \pi_\step(\mathrm{d} \globParam{0})
    \eqsp.
\end{align*}
Taking the operator norm and using Hölder's inequality gives
\begin{align}
    \norm{ \covconsensus }
    & \le
    (1 - \step \strcvx)^2 \norm{ \covconsensus }
    + \step^2 \norm{ \covavgnoise }
    + 2 \step \Big( \int \norm{ \bglobParam{0} - \bparamlim }^2 \pi_\step(\mathrm{d} \globParam{0}) \Big)^{1/2}
    \Big( \int \norm{ \bar{R}(\globParam{0}) }^2 \pi_\step(\mathrm{d} \globParam{0}) \Big)^{1/2}
    \eqsp.
\end{align}
Using Jensen's inequality:
\begin{align*}
    \Big( \int \norm{ \bglobParam{0} - \bparamlim }^2 \pi_\step(\mathrm{d} \globParam{0}) \Big)^{1/2}
    & \le
    \Big( \int \frac{1}{\nagent} \norm{ \globParam{0} - \Paramlim }^2 \pi_\step(\mathrm{d} \globParam{0}) \Big)^{1/2} \\
    & \le
    \frac{1}{\sqrt{\nagent}} B^{1/2} \step^{1/2}
    \eqsp.
\end{align*}
Similarly:
\begin{align*}
    \Big( \int \norm{ \bar{R}(\globParam{0}) }^2 \pi_\step(\mathrm{d} \globParam{0}) \Big)^{1/2}
    & \le
    \Big( \int \frac{1}{\nagent} \norm{ R(\globParam{0}) }^2 \pi_\step(\mathrm{d} \globParam{0}) \Big)^{1/2} \\
    & \le
    \frac{\boundThird}{2\sqrt{\nagent}}
    \Big( \int \norm{ \globParam{0} - \Paramlim }^4 \pi_\step(\mathrm{d} \globParam{0}) \Big)^{1/2} \\
    & \le
    \frac{\boundThird B}{2\sqrt{\nagent}} \step
    \eqsp.
\end{align*}
Hence
\begin{align*}
    \Big( \int \norm{ \bglobParam{0} - \bparamlim }^2 \pi_\step(\mathrm{d} \globParam{0}) \Big)^{1/2}
    \Big( \int \norm{ \bar{R}(\globParam{0}) }^2 \pi_\step(\mathrm{d} \globParam{0}) \Big)^{1/2}
    & \le
    \frac{\boundThird B^{3/2}}{2\nagent} \step^{3/2}
    \eqsp,
\end{align*}
and
\begin{align*}
    \norm{ \covconsensus }
    & \le
    (1 - \step \strcvx)^2 \norm{ \covconsensus }
    + \step^2 \norm{ \covavgnoise }
    + \step^{5/2} \frac{\boundThird B^{3/2}}{\nagent}
    \eqsp,
\end{align*}
which gives, after reorganizing and bounding the covariance matrix,
\begin{align*}
    \norm{ \covconsensus }
    & \le
    \frac{\step}{\mu\nagent} \norm{ \covnoise }
    + \step^{3/2} \frac{\boundThird B^{3/2}}{\strcvx\nagent}
    \eqsp.
\end{align*}

\textbf{Bound on Disagreement.}
First, we recall that for any vector $u \in \rset^{md}$, since $\tenscommmat$ is doubly stochastic,
\begin{align}
\projdisagreement \tenscommmat u 
& =
\tenscommmat u - \projconsensus \tenscommmat u
= \tenscommmat u -  \tenscommmat \projconsensus u
= \tenscommmat \projdisagreement u
\eqsp.
\end{align}
Starting from \eqref{eq:decomp-update-1} and applying $\projdisagreement$, we thus have
\begin{align*}
\projdisagreement (\globParam{1} - \Paramlim)
& =
\tenscommmat \projdisagreement (\globParam{0} - \Paramlim)
- \step \tenscommmat \projdisagreement(\hF{\Paramlim} (\globParam{0} - \Paramlim))
- \step \tenscommmat  \projdisagreement \big( \globnoise{1}{\globParam{0}} + R(\globParam{0}) \big)
    \eqsp.
\end{align*}
Computing the tensor square and taking the expectation on the noise, we obtain
\begin{align*}
\Big( \projdisagreement (\globParam{1} - \Paramlim) \Big)^{\otimes 2}
& =
\tenscommmat \Big( \projdisagreement (\globParam{0} - \Paramlim) \Big)^{\otimes 2} \tenscommmat
\\
& - \step 
\Big( \tenscommmat  \projdisagreement (\globParam{0} - \Paramlim) \Big) \Big(
\tenscommmat  \projdisagreement(\hF{\Paramlim} (\globParam{0} - \Paramlim))
+ \tenscommmat  \projdisagreement R(\globParam{0})
\Big)^\top
\\
& - \step 
 \Big(
\tenscommmat  \projdisagreement(\hF{\Paramlim} (\globParam{0} - \Paramlim))
+ \tenscommmat  \projdisagreement  R(\globParam{0})
\Big) \Big( \tenscommmat \projdisagreement (\globParam{0} - \Paramlim) \Big)^\top 
\\
& + \step^2 
\Big( \tenscommmat  \projdisagreement(\hF{\Paramlim} (\globParam{0} - \Paramlim))
+ \tenscommmat  \projdisagreement R(\globParam{0}) 
\Big)^{\otimes 2} 
+ \step^2 \PE[ \tenscommmat  \projdisagreement \globnoise{1}{\globParam{0}}^{\otimes 2}  \projdisagreement \tenscommmat ]
\eqsp.
\end{align*}
Integrating over the stationary distribution gives
\begin{align}
\label{eq:disagreement-cov-equation}
\covdisagreement
& =
\tenscommmat \covdisagreement \tenscommmat
- \step \boldsymbol{B_1}
+ \step^2 \boldsymbol{B_2}
\eqsp,
\end{align}
where we introduced
\begin{align*}
\boldsymbol{B_1}
& =
\int 
 \Big( \tenscommmat  \projdisagreement (\globParam{0} - \Paramlim) \Big) \Big(
\tenscommmat  \projdisagreement(\hF{\Paramlim} (\globParam{0} - \Paramlim))
+ \tenscommmat  \projdisagreement R(\globParam{0})
\Big)^\top \pi_\step(\mathrm{d}\globParam{0}) 
\\
& 
\quad + \int \Big(
\tenscommmat  \projdisagreement(\hF{\Paramlim} (\globParam{0} - \Paramlim))
+ \tenscommmat  \projdisagreement  R(\globParam{0})
\Big) \Big( \tenscommmat  \projdisagreement (\globParam{0} - \Paramlim) \Big)^\top 
\pi_\step(\mathrm{d}\globParam{0}) 
\eqsp,
\\
\boldsymbol{B_2}
& =
\int \Big( \tenscommmat \projdisagreement(\hF{\Paramlim} (\globParam{0} - \Paramlim))
+ \tenscommmat  \projdisagreement R(\globParam{0}) 
\Big)^{\otimes 2}  \pi_\step(\mathrm{d}\globParam{0}) 
+ \int \PE[ \tenscommmat \projdisagreement \globnoise{1}{\globParam{0}}^{\otimes 2} \projdisagreement \tenscommmat ]
\eqsp.
\end{align*}
Note that these matrices can be upper bounded, in operator norm, as
\begin{align*}
    \norm{\boldsymbol{B_1} }
    & \le
    2 \spectralgap^2 \int \Big(
    \lip\norm{ \globParam{0} - \Paramlim }
    + \norm{ R(\globParam{0}) }
    \Big) \norm{ \globParam{0} - \Paramlim }
    \pi_\step(\mathrm{d}\globParam{0}) 
    \\
    & =
    2 \spectralgap^2 \int \Big(
    \lip\norm{ \globParam{0} - \Paramlim }^2
    + \norm{ R(\globParam{0}) } \norm{ \globParam{0} - \Paramlim } \Big)
    \pi_\step(\mathrm{d}\globParam{0}) 
    \\
    & \le
    2 \spectralgap^2 \int \Big(
    \lip\norm{ \globParam{0} - \Paramlim }^2
    + \frac{\boundThird}{2} \norm{ \globParam{0} - \Paramlim }^3 \Big)
    \pi_\step(\mathrm{d}\globParam{0}) 
    \\
    & \le
    2 \spectralgap^2
    \Big(
    \lip B \step
    +
    \frac{\boundThird}{2}
    B^{3/2} \step^{3/2}
    \Big)
    \\
    & =
    2 \spectralgap^2
    \lip B \step
    +
    \spectralgap^2
    \boundThird
    B^{3/2} \step^{3/2}
    \eqsp,
\end{align*}
and
\begin{align*}
    \norm{ \boldsymbol{B_2} }
    & \le
    \int \bnorm{ \tenscommmat \projdisagreement(\hF{\Paramlim} (\globParam{0} - \Paramlim))
    - \step \tenscommmat  \projdisagreement R(\globParam{0}) }^2
    \pi_\step(\mathrm{d}\globParam{0})
    + \norm{ \tenscommmat \projdisagreement \covnoise \projdisagreement \tenscommmat }
    \\
    & \le
    2 \spectralgap^2\int  L^2 \norm{ \globParam{0} - \Paramlim }^2
    \pi_\step(\mathrm{d}\globParam{0})
    + 2 \spectralgap^2 \int \norm{ R(\globParam{0}) }^2
    \pi_\step(\mathrm{d}\globParam{0})
    + \spectralgap^2 \norm{ \covnoise }
    \\
    & \le
    2 \spectralgap^2L^2 B \step
    + \frac{1}{2} \step^2\spectralgap^2 \boundThird^2 B^2
    + \spectralgap^2 \norm{ \covnoise }
    \eqsp,
\end{align*}
where we used $\norm{ \tenscommmat \projdisagreement } \le \spectralgap$.

Remark that \eqref{eq:disagreement-cov-equation} is a Sylvester equation, whose solution is
\begin{align}\label{eq:disgreement-cov-solution}
    \covdisagreement
    & =
    \sum_{k\ge0} \tenscommmat^k (- \step \boldsymbol{B_1}
    + \step^2 \boldsymbol{B_2}) \tenscommmat^k
    \eqsp.
\end{align}
Taking the norm of \eqref{eq:disgreement-cov-solution} gives
\begin{align}
    \nonumber
    \norm{ \covdisagreement }
    & \le
    \frac{\step}{1 - \spectralgap^2}
    \norm{\boldsymbol{B_1} + \step \boldsymbol{B_2} }
    \\
    & \le
    \frac{\step^2}{1 - \spectralgap^2}
    \big(
    2 \spectralgap^2
    \lip B
    +
    \spectralgap^2
    \boundThird
    B^{3/2} \step^{1/2}
    +
    2 \spectralgap^2 \lip^2 B \step
    + \frac{1}{2} \step^2 \spectralgap^2 \boundThird^2 B^2
    + \spectralgap^2 \norm{ \covnoise }
    \big)
    \nonumber
    \\
    & \le
    \frac{\step^2 \spectralgap^2}{1 - \spectralgap^2}
    \big(
    4 \lip B
    +
    \boundThird
    B^{3/2} \step^{1/2}
    + \frac{1}{2} \step^2 \boundThird^2 B^2
    + \norm{ \covnoise }
    \big)
    \nonumber
    \eqsp,
\end{align}
where we used $\step \lip \le 1$ in the last inequality.

\paragraph{Bound on Variance.}
It holds that
$\norm{\covstationary}
\le
\norm{\covdisagreement}
+ \norm{\covconsensus}
+ 2 \norm{\covdisagreementconsensus}
\le
2(\norm{\covdisagreement}
+ \norm{\covconsensus})$.
Based on the two bounds above, we then have the following bound on the variance:
\begin{align}
    \norm{ \covstationary }
    & \le
    2
    \Big(
    \frac{\step}{\mu\nagent} \norm{ \covnoise }
    + \step^{3/2} \frac{\boundThird B^{3/2}}{\strcvx\nagent}
    +
    \frac{\step^2 \spectralgap^2}{1 - \spectralgap^2}
    \big(
    4 \lip B
    +
    \boundThird
    B^{3/2} \step^{1/2}
    + \frac{1}{2} \step^2 \boundThird^2 B^2
    + \norm{ \covnoise }
    \big)
    \Big)
    \eqsp.
\end{align}

\subsection{Non-Asymptotic Bounds for the Iterates of DSGD.}
If $(\globParamp{t})_{t\ge0}$ is another sequence, such that $\globParamp{0} \sim \meas{\step}$, and defined using the same noise sequence as $(\globParam{t})_{t\ge0}$, then
\begin{align*}
    \PE[ \norm{ \globParam{t} - \Paramstar }^2 ]
    \le
    3 \PE[ \norm{ \globParam{t} - \globParamp{t} }^2 ]
    + 3 \norm{ \Paramlim - \Paramstar }^2
    + 3 \int \norm{ \Paramw - \Paramlim }^2 
    \pi_\step(\mathrm{d}\Paramw)
    \eqsp.
\end{align*}
Bounding each term using \Cref{prop:DSGD}, \Cref{lemma:upper-bound-error} and the bounds established previously, we obtain
\begin{align*}
    \PE[ \norm{ \globParam{t} - \Paramstar }^2 ]
    & \le
    3 (1 - \step \strcvx)^t \PE[ \norm{ \globParam{0} - \globParamp{0} }^2 ]
    + 3 \frac{ \step^2 \lip^2 \Lambda^2}{\strcvx^2} \hgty^2 \\
    & \quad
    + 6
    \Big(
    \frac{\step}{\mu\nagent} \norm{ \covnoise }
    + \step^{3/2} \frac{\boundThird B^{3/2}}{\strcvx\nagent}
    \Big) \\
    & \quad
    + 6
    \Big(
    \frac{\step^2 \spectralgap^2}{1 - \spectralgap^2}
    \big(
    4 \lip B
    +
    \boundThird
    B^{3/2} \step^{1/2}
    + \frac{1}{2} \step^2 \boundThird^2 B^2
    + \norm{ \covnoise }
    \big)
    \Big)
    \eqsp.
\end{align*}
\bigskip

\begin{proof}[Proof of \Cref{coro:sample-complexity-dsgd}]
    The proof is straightforward from \Cref{prop:non-asymptotic-bounds}: the conditions on $\step$ ensure that all terms in the upper-bound except the first one are of order $\varepsilon$, and the condition on $T$ ensures that the first term is of order $\varepsilon$.
\end{proof}
\bigskip

\section{RICHARDSON-ROMBERG FOR DECENTRALIZED LEARNING}
\label{sec:app-RR}

We now prove the results stated in \Cref{sec:RR} on Richardson-Romberg extrapolation.

\begin{proof}[Proof of \Cref{prop:RR-det}]
    This result is a direct consequence of \Cref{prop:expansion-general}.
    Indeed, first write
    \begin{align*}
        \Paramlim^{\step}
        &= \Paramstar - \step  (\tensId - \hetMat) \melMat \gF{\Paramstar} + R(\step) \\
        \Paramlim^{\step/2}
        &= \Paramstar - \frac{\step}{2}  (\tensId - \hetMat) \melMat \gF{\Paramstar} + R(\step/2)
        \eqsp,
    \end{align*}
    where
    \begin{align*}
        \norm{R(\step)}
        & \le
        \step^2 \frac{\lip^2}{2\strcvx^2} \Lambda^2
        \Big(
        \tfrac{\boundThird}{\strcvx\sqrt{\nagent}} \hgty^2
        + \lip \cdot \hgty
        \Big)
        \eqsp.
    \end{align*}
    Hence, using $2\Paramlim^{\step/2} - \Paramlim^\step - \Paramstar = 2R(\step/2) - R(\step)$, we have
    \begin{align*}
        \norm{\ParamlimRR - \Paramstar}
        & =
        \norm{2\Paramlim^{\step/2} - \Paramlim^\step - \Paramstar}
        \le
        2\norm{R(\step/2)} + \norm{R(\step)}
         \le
        \step^2 \frac{3\lip^2}{4\strcvx^2} \Lambda^2
        \Big(
        \tfrac{\boundThird}{\strcvx\sqrt{\nagent}} \hgty^2
        + \lip \cdot \hgty
        \Big)
        \eqsp.
    \end{align*}
\end{proof}

\begin{proof}[Proof of \Cref{coro:conv-rr-dgd}]
    Using \Cref{lemma:DGD-A} and \Cref{prop:RR-det},
    \begin{align*}
        \norm{\globParamRR{t} - \Paramstar}
        & \le
        \norm{\globParamRR{t} - \ParamlimRR}
        + \norm{\ParamlimRR - \Paramstar} \\
        & \le
        2 \norm{\globParam{t}^{\step/2} - \Paramlim^{\gamma/2}}
        + \norm{\globParam{t}^{\step} - \Paramlim^{\gamma/2}}
        + \norm{\ParamlimRR - \Paramstar} \\
        & \le
        2 (1 - \step\strcvx / 2)^t \norm{\globParam{0} - \Paramlim^{\gamma/2}}
        + (1 - \step\strcvx)^t \norm{\globParam{0} - \Paramlim^{\gamma}}
        + \step^2 \frac{5\lip^2}{8\strcvx^2} \Lambda^2
        \big(
        \frac{\boundThird}{\strcvx\sqrt{\nagent}} \hgty^2
        + \lip \hgty
        \big) \\
        & \le
        3 (1 - \step\strcvx / 2)^t \norm{\globParam{0} - \Paramstar}
        + 2 (1 - \step\strcvx / 2)^t \norm{\Paramstar - \Paramlim^{\gamma/2}}
        + (1 - \step\strcvx)^t \norm{\Paramstar - \Paramlim^{\gamma}} \\
        & \quad
        + \step^2 \frac{5\lip^2}{8\strcvx^2} \Lambda^2
        \big(
        \frac{\boundThird}{\strcvx\sqrt{\nagent}} \hgty^2
        + \lip \hgty
        \big)
        \eqsp.
    \end{align*}
    We then apply \Cref{lemma:upper-bound-error},
    \begin{align*}
        \norm{\globParamRR{t} - \Paramstar}
        & \le
        3 (1 - \step\strcvx / 2)^t \norm{\globParam{0} - \Paramstar}
        + 2(1 - \step\strcvx / 2)^t \frac{\step \lip \Lambda}{\strcvx} \hgty 
        + \step^2 \frac{5\lip^2}{8\strcvx^2} \Lambda^2
        \big(
        \frac{\boundThird}{\strcvx\sqrt{\nagent}} \hgty^2
        + \lip \hgty
        \big)
        \eqsp,
    \end{align*}
    
    hence the conclusion.

    If $\boundThird / m^{1/2}$ is small enough, and if $\step$ is as in the statement of the corollary, then the second and third terms of the above bound are of order $\varepsilon$.

    We thus look for $T$ such that
    $3 (1 - \step\strcvx / 2)^T \norm{\globParam{0} - \Paramstar}
    \lesssim \varepsilon$, which yields
    $T \gtrsim
    \dfrac{\log(\varepsilon / \norm{\globParam{0} - \Paramstar})}{\log(1 - \step\strcvx / 2)}
    \gtrsim
    \dfrac{\log(\norm{\globParam{0} - \Paramstar} / \varepsilon)}{\step\strcvx}$, which concludes.
\end{proof}

\begin{proof}[Proof of \Cref{coro:sample-complexity-dsgd-RR}]
    We first write
    \begin{align*}
        \norm{ \globParamRR{t} - \Paramstar}^2
        & \le
        4 \norm{ \globParam{t}^{\step/2} - \Paramlim^{\step/2}}^2
        + \norm{ \globParam{t}^\step - \Paramstar }^2
        \eqsp.
    \end{align*}
    Then, we simply need to make sure that the upper-bound given in \Cref{prop:non-asymptotic-bounds} is of order $\varepsilon$. This can be done similarly to the proof of \Cref{coro:conv-rr-dgd}, using the assumptions on both $\step$ and $T$.
\end{proof}
\bigskip

\section{USEFUL RESULTS ON MATRICES}\label{app:matrices}


\begin{lemma}\label{lemma:AtB}
    Let $A \succeq 0$ and $B \succ 0$.
    Then, as $t$ approaches $0$,
    \begin{align*}
        (A + tB)^{-1}
        & =
        \frac1t \big( \projconsensus B \projconsensus \big)^\dagger
        + \mathcal{O}(1)
        \eqsp,
    \end{align*}
    where
    $\projconsensus$ is the orthogonal projection onto the kernel of $A$.

    Moreover,
    \begin{align*}
        \bnorm{(A + tB)^{-1} -
        \frac1t \big( \projconsensus B \projconsensus \big)^\dagger}_2
        & \le
        \frac{1}{\lambda_{\mathrm{min}}^+(A)}
        \left( 1 + \frac{\lambda_{\mathrm{max}}(B)}{\lambda_{\mathrm{min}}(B)} \right)^2
        \eqsp,
    \end{align*}
    where $\lambda_{\mathrm{min}}^+(A)$ is $A$'s smallest non-zero eigenvalue.
\end{lemma}
\begin{proof}
Let $U$ be an orthogonal matrix such that
\begin{equation*}
    A
    =
    U
    \begin{pmatrix}
        0 & 0 \\
        0 & D
    \end{pmatrix}
    U^\top
    \eqsp, \qquad
    B
    =
    U
    \begin{pmatrix}
        B_{00} & B_{01} \\
        B_{10} & B_{11}
    \end{pmatrix}
    U^\top
    \eqsp,
\end{equation*}
where $D$ is diagonal and non-singular, $B_{00}, B_{11} \succ 0$, and $B_{10} = B_{01}^\top$.

Then
\begin{align*}
    A + tB
    & =
    U
    \underbrace{
    \begin{pmatrix}
        t B_{00} & t B_{01} \\
        t B_{10} & D + t B_{11}
    \end{pmatrix}
    }_{\eqdef M}
    U^\top
    \eqsp.
\end{align*}
The Schur complement is defined as
\begin{align*}
    S(t)
    & =
    D + t \big( B_{11} - B_{10}B_{00}^{-1}B_{01} \big)
    \eqsp.
\end{align*}
Then:
\begin{align*}
    M^{-1}
    & =
    \begin{pmatrix}
        \dfrac1t B_{00}^{-1}
        + B_{00}^{-1} B_{01} S(t)^{-1} B_{10} B_{00}^{-1}
        &
        - B_{00}^{-1} B_{01} S(t)^{-1} \\
        -S(t)^{-1} B_ {10}B_{00}^{-1}
        &
        S(t)^{-1}
    \end{pmatrix} \\
    & =
    \frac1t
    \begin{pmatrix}
        B_{00}^{-1} & 0 \\
        0 & 0
    \end{pmatrix}
    +
    \begin{pmatrix}
        B_{00}^{-1} B_{01} S(t)^{-1} B_{10} B_{00}^{-1}
        &
        - B_{00}^{-1} B_{01} S(t)^{-1} \\
        -S(t)^{-1} B_ {10}B_{00}^{-1}
        &
        S(t)^{-1}
    \end{pmatrix}
    \eqsp,
\end{align*}
and we thus have
\begin{align*}
    (A + tB)^{-1}
    & =
    \frac1t
    U
    \begin{pmatrix}
        B_{00}^{-1}
        &
        0 \\
        0
        &
        0
    \end{pmatrix}
    U^\top
    +
    U
    \begin{pmatrix}
        B_{00}^{-1} B_{01} S(t)^{-1} B_{10} B_{00}^{-1}
        &
        - B_{00}^{-1} B_{01} S(t)^{-1} \\
        -S(t)^{-1} B_ {10}B_{00}^{-1}
        &
        S(t)^{-1}
    \end{pmatrix}
    U^\top
    \eqsp.
\end{align*}
If we set
\begin{align*}
    \projconsensus
    & =
    U
    \begin{pmatrix}
        I & 0 \\
        0 & 0
    \end{pmatrix}
    U^\top
\end{align*}
the orthogonal projection onto the kernel of $A$, then
\begin{align*}
    \frac1t
    U
    \begin{pmatrix}
        B_{00}^{-1} & 0 \\
        0 & 0
    \end{pmatrix}
    U^\top
    & =
    \frac1t \big( \projconsensus B \projconsensus \big)^\dagger
    \eqsp.
\end{align*}
Moreover:
\begin{itemize}
    \item
    $\bnorm{
    U
    \begin{pmatrix}
        B_{00}^{-1} B_{01} S(t)^{-1} B_{10} B_{00}^{-1}
        &
        - B_{00}^{-1} B_{01} S(t)^{-1} \\
        -S(t)^{-1} B_ {10}B_{00}^{-1}
        &
        S(t)^{-1}
    \end{pmatrix}
    U^\top
    }_2
    \le
    \norm{S(t)^{-1}}_2
    \big( 1 + \norm{B_{00}^{-1}}_2 \cdot \norm{B_{01}}_2 \big)^2$

    \item
    $S(t) \succeq D$, so $\lambda_{\mathrm{min}}(S(t)) \ge \lambda_{\mathrm{min}}(D)$,
    and
    $\norm{S(t)^{-1}}_2
    \le \dfrac{1}{\lambda_{\mathrm{min}}(D)}
    = \dfrac{1}{\lambda_{\mathrm{min}}^+(A)}$

    \item
    $\norm{B_{01}}_2
    \le \norm{B}_2
    = \lambda_{\mathrm{max}}(B)$

    \item
    $\lambda_{\mathrm{min}}(B_{00})
    \ge \lambda_{\mathrm{min}}(B)$,
    so
    $\norm{B_{00}^{-1}}_2
    \le \dfrac{1}{\lambda_{\mathrm{min}}(B)}$
    \eqsp.
\end{itemize}
Putting everything together yields:
\[
    \bnorm{
    U
    \begin{pmatrix}
        B_{00}^{-1} B_{01} S(t)^{-1} B_{10} B_{00}^{-1}
        &
        - B_{00}^{-1} B_{01} S(t)^{-1} \\
        -S(t)^{-1} B_ {10}B_{00}^{-1}
        &
        S(t)^{-1}
    \end{pmatrix}
    U^\top
    }_2
    \le
    \frac{1}{\lambda_{\mathrm{min}}^+(A)}
    \left( 1 + \frac{\lambda_{\mathrm{max}}(B)}{\lambda_{\mathrm{min}}(B)} \right)^2
    \eqsp,
\]
hence the conclusion.
\end{proof}

\begin{lemma}\label{lemma:norm-difference-inv}
    Let $A \succ 0, B \succeq 0$. Then it holds that
    \[
        \norm{(A + B)^{-1} - A^{-1}}_2
        \le
        \norm{A^{-1}}_2^2 \cdot \norm{B}_2
        \eqsp.
    \]
\end{lemma}
\begin{proof}
    We use the identity
    \begin{align*}
        M^{-1} - N^{-1}
        & =
        - N^{-1} (M - N) M^{-1}
    \end{align*}
    to obtain
    \begin{align*}
        (A + B)^{-1} - A^{-1}
        & =
        -A^{-1} B (A + B)^{-1}
        \eqsp,
    \end{align*}
    hence
    \begin{align*}
        \norm{(A + B)^{-1} - A^{-1}}_2
        & \le
        \norm{A^{-1}}_2 \cdot \norm{B}_2 \cdot \norm{ (A + B)^{-1} }_2
        \eqsp.
    \end{align*}
    And $A + B \succeq A$, so $(A + B)^{-1} \preceq A^{-1}$, which implies $\norm{(A + B)^{-1}}_2 \le \norm{A^{-1}}_2$, hence the conclusion.
\end{proof}


\end{document}